\documentclass[sigconf]{acmart}
\AtBeginDocument{%
  }
\usepackage[textsize=tiny]{todonotes}
\usepackage{paralist}
\usepackage{amsmath}
\usepackage{mathtools}
\usepackage{amsthm}
\usepackage{multirow}
\usepackage{multicol}
\usepackage[utf8]{inputenc}
\usepackage{tcolorbox}
\usepackage{float}
\usepackage{hyperref}
\usepackage{microtype}
\usepackage{graphicx}
\usepackage{subfigure}
\usepackage{booktabs} 
\usepackage{bbm}
\usepackage{wrapfig}
\usepackage{marvosym}
\usepackage[capitalize,noabbrev]{cleveref}

\setcopyright{acmlicensed}
\copyrightyear{2026}
\acmYear{2026}
\setcopyright{cc}
\setcctype{by}
\acmConference[KDD '26]{Proceedings of the 32nd ACM SIGKDD Conference on Knowledge Discovery and Data Mining V.1}{August 09--13, 2026}{Jeju Island, Republic of Korea}
\acmBooktitle{Proceedings of the 32nd ACM SIGKDD Conference on Knowledge Discovery and Data Mining V.1 (KDD '26), August 09--13, 2026, Jeju Island, Republic of Korea}
\acmPrice{}
\acmDOI{10.1145/3770854.3780218}
\acmISBN{979-8-4007-2258-5/2026/08}

\begin{document}
\newcommand{\method}{\textsc{TimeDistill}}

\clearpage
\newcommand{\titlelogo}{%
  \raisebox{-0.25\height}{\includegraphics[height=1.4em]{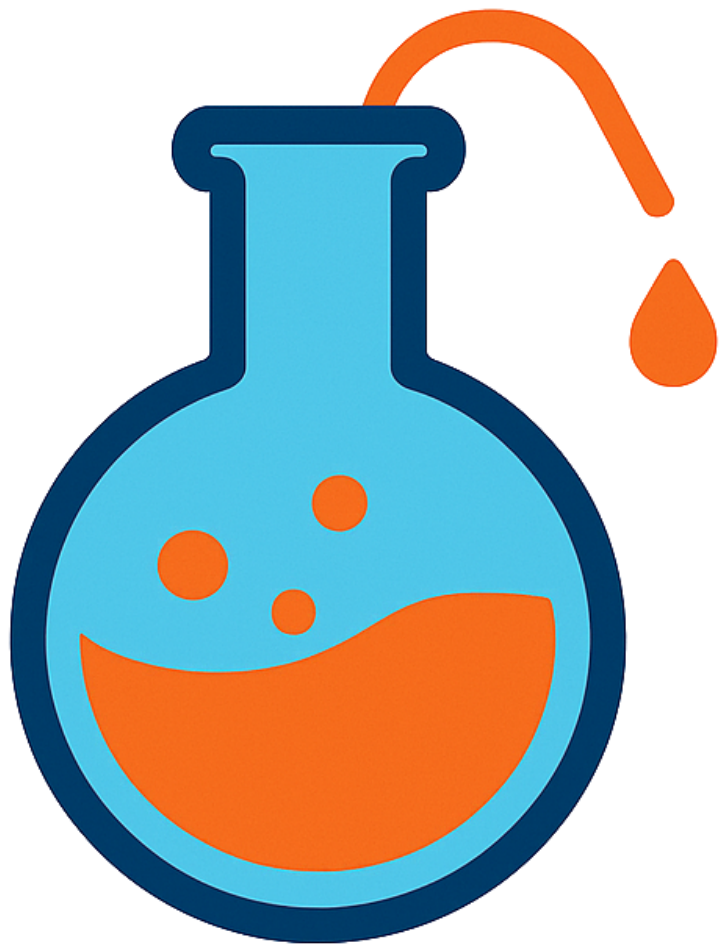}}%
}

\title{%
  \titlelogo
  \method{}: Efficient Long-Term Time Series Forecasting with MLP via Cross-Architecture Distillation%
}

\author{Juntong Ni}
\affiliation{%
  \institution{Emory University}
  \city{Atlanta}
  \country{United States}}
\email{juntong.ni@emory.edu}

\author{Zewen Liu}
\affiliation{%
  \institution{Emory University}
  \city{Atlanta}
  \country{United States}}
\email{zewen.liu@emory.edu}

\author{Shiyu Wang*}
\affiliation{%
  \institution{}
  \city{}
  \country{}
  }
\email{kwuking@gmail.com}

\author{Ming Jin}
\affiliation{%
  \institution{Griffith University}
  \city{Brisbane}
  \country{Australia}}
\email{mingjinedu@gmail.com}

\author{Wei Jin}
\affiliation{%
  \institution{Emory University}
  \city{Atlanta}
  \country{United States}}
\email{wei.jin@emory.edu}

\def\authors{Juntong Ni, Zewen Liu, Shiyu Wang, Ming Jin, and Wei Jin}

\renewcommand{\shortauthors}{Juntong Ni, Zewen Liu, Shiyu Wang, Ming Jin, and Wei Jin}


\thanks{* Corresponding author.}


\begin{abstract}
Transformer-based and CNN-based methods demonstrate strong performance in long-term time series forecasting. However, their high computational and storage requirements can hinder large-scale deployment. To address this limitation, we propose integrating lightweight MLP with advanced architectures using knowledge distillation (\textit{KD}). Our preliminary study reveals different models can capture complementary patterns, particularly multi-scale and multi-period patterns in the temporal and frequency domains.
Based on this observation, we introduce \method{}, a cross-architecture \textit{KD} framework that transfers these patterns from teacher models (e.g., Transformers, CNNs) to MLP. Additionally, we provide a theoretical analysis, demonstrating that our \textit{KD} approach can be interpreted as a specialized form of \textit{mixup} data augmentation.
\method{} improves MLP performance by up to 18.6\%, surpassing teacher models on eight datasets. It also achieves up to 7$\times$ faster inference and requires 130$\times$ fewer parameters. Furthermore, we conduct extensive evaluations to highlight the versatility and effectiveness of \method{}. The code is available at \href{https://github.com/LingFengGold/TimeDistill}{Github Code Repo}.

\end{abstract}





\maketitle

\section{Introduction}

Forecasting is a notably critical problem in the time series analysis community, which aims to predict future time series based on historical time series records~\cite{wang2024deep}. It has broad practical applications such as climate modeling~\cite{wu2023interpretable}, traffic flow management~\cite{yin2021deep}, healthcare monitoring~\cite{kaushik2020ai} and finance analytics~\cite{granger2014forecasting}.

Recently, there has been an ongoing debate over which deep learning architecture best suits time series forecasting. The 
\setlength{\columnsep}{0.8em} 
\begin{wrapfigure}[14]{r!}{0.28\textwidth} 
    \centering
    \vspace{-1.5em}
    \includegraphics[width=1\linewidth]{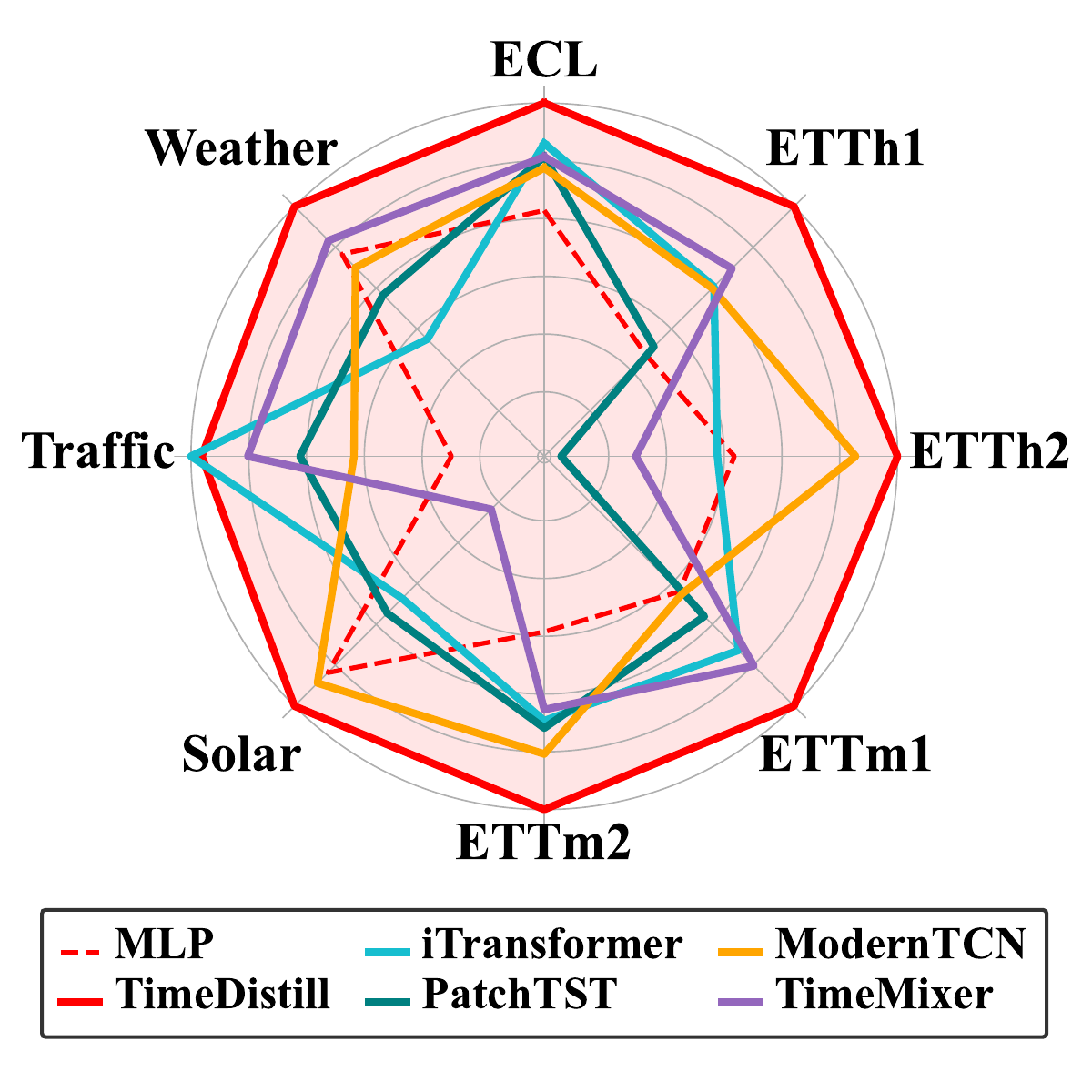}
    \vspace{-3em}
    \caption{Performance comparison.}
    \label{fig:result_rader}
\end{wrapfigure}
rise of Transformers in various domains~\cite{DBLP:journals/corr/abs-1810-04805,khan2022transformers} has led to their wide adoption in  time series forecasting~\cite{wen2022transformers, autoformer, fedformer, informer, patchtst, itransformer}, leveraging the strong capabilities of capturing pairwise dependencies and extracting multi-level representations within sequential data.
Similarly, CNN architectures have also proven effective by developing convolution blocks for time series~\cite{moderntcn, micn}.
However, despite the strong performance of Transformer-based and CNN-based models, they face significant challenges in large-scale industrial applications due to their relatively high computational demands, especially in latency-sensitive scenarios like financial prediction and healthcare monitoring~\cite{granger2014forecasting, kaushik2020ai}. In contrast, simpler linear or MLP models offer greater efficiency, although with lower performance~\cite{dlinear, sparsetsf}. These contrasting observations raises an intriguing question:


\begin{center}
\begin{tcolorbox}[colback=yellow!10!white, colframe=yellow!75!black, width=0.45\textwidth, boxrule=0.5mm, arc=5mm, auto outer arc]
\centering
\textit{\textbf{Can we combine MLP with other advanced architectures (e.g., Transformers and CNNs) to create a powerful yet efficient model?}}
\end{tcolorbox}
\end{center}

A promising approach to addressing this question is knowledge distillation (\textit{KD})~\cite{hinton2015distilling}, a technique that transfers knowledge from a larger and more complex model (\textit{teacher}) to a smaller and simpler one (\textit{student}) while maintaining comparable performance. In this work, we pioneer \textit{cross-architecture KD} in time series forecasting, with MLP as the \textit{student} and other advanced architectures (e.g., Transformers and CNNs) as the \textit{teacher}. However, designing such a framework is non-trivial, as it remains unclear what ``knowledge'' should be distilled into MLP.

To investigate this potential, we conduct a comparative analysis of prediction patterns between MLP and other time series models.  Our findings reveal that MLP still excels on some data subsets despite its overall lower performance (Sec.~\ref{sec:why_distill}), which highlights the value of harnessing the \textit{complementary capabilities} across different architectures. 
To further explore the specific properties to distill, we focus on two key time series patterns: multi-scale pattern in temporal domain and multi-period pattern in frequency domain, given that they are vital in capturing the complex structures typical of many time series.
\textbf{(1)~Multi-Scale Pattern:} Real-world 
 time series often show variations at multiple temporal scales. For example, hourly recorded traffic flow data capture changes within each day, while daily sampled data lose fine-grained details but reveal patterns related to holidays~\cite{timemixer}. We observe that models that perform well on the finest scale also perform accurately on coarser scales, while MLP fails on most scales (Sec.~\ref{sec:what2distill}). 
\textbf{(2) Multi-Period Pattern:} Time series often exhibit multiple periodicities. For instance, weather measurements may have both daily and yearly cycles, while electricity consumption data may show weekly and quarterly cycles~\cite{timesnet}. We find that models that perform well can capture periodicities similar to those in the ground truth, but MLP fails to capture these periodicities (Sec.~\ref{sec:what2distill}).
Therefore, enhancing MLP requires distilling and integrating these multi-scale and multi-period patterns from teacher models. 


Based on our observations, we propose a cross-architecture \textit{KD} framework named \method{} to bridge the performance and efficiency gap between complex teacher models and a simple MLP.  Instead of solely matching predictions in conventional \textit{KD}, \method{} focuses on aligning multi-scales and multi-period patterns between MLP and the teacher: we downsample the time series for temporal multi-scale alignment and apply Fast Fourier Transform (FFT) to align period distributions in the frequency domain. 
The \textit{KD} process can be conducted offline, shifting heavy computations from latency-critical \textit{inference phase}, where millisecond matter, to less time-sensitive \textit{training phase}, where longer processing time is acceptable. 
We validate effectiveness of \method{} both theoretically and empirically and summarize our contributions as follows:

\begin{figure}[t]
    \centering
    
    \includegraphics[width=0.37\textwidth]{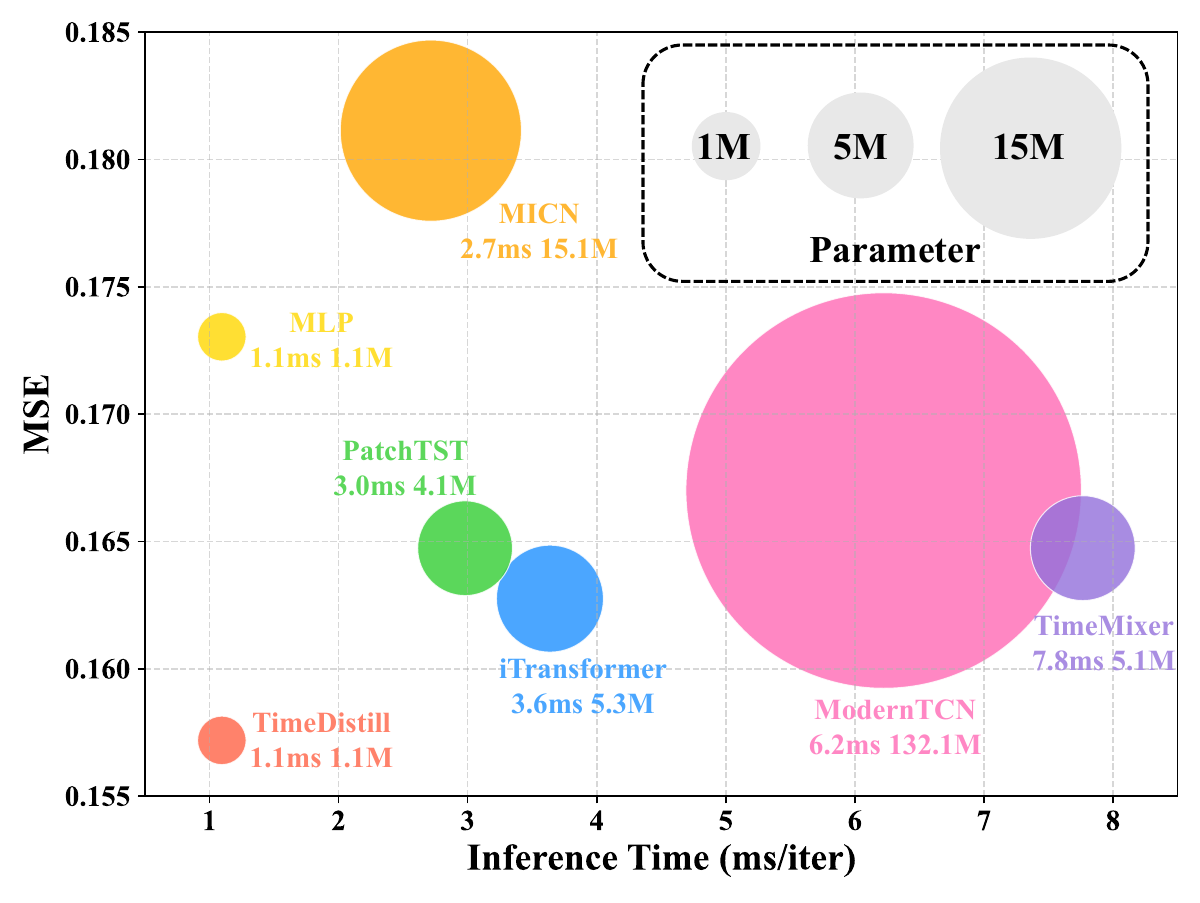}
    \vskip -1.5em
    \caption{Model efficiency comparison averaged across all prediction lengths (96, 192, 336, 720) for the ECL dataset. Full results on more datasets are listed in Appendix~\ref{app:efficiency}.}
    \label{fig:result_efficiency}
    \vskip -2em
\end{figure}


\begin{compactenum}[(a)] 

\item We present \textit{the first cross-architecture KD} framework \method{} tailored for efficient and effective time series forecasting via MLP, which is supported by our preliminary studies examining multi-scale and multi-period patterns in time series.


\item We provide theoretical insights into benefits of \method{}, illustrating that proposed distillation process can be viewed as a form of data augmentation through a special \textit{mixup} strategy.




\item We show that \method{} is both effective and efficient, consistently outperforming standalone MLP by up to \textbf{18.6\%} and surpassing teacher models in nearly all cases (see Figure~\ref{fig:result_rader}). Additionally, it achieves up to \textbf{7x} faster inference and requires up to \textbf{130×} fewer parameters compared to teachers (see Figure~\ref{fig:result_efficiency}).


\item We conduct deeper analyses of \method{}, exploring its adaptability across various teacher/student models and highlighting the distillation impacts it brings to the temporal and frequency domains.





\end{compactenum}

\begin{figure*}[ht!]
    \centering
    \begin{minipage}{0.32\textwidth}
        \centering
        \includegraphics[width=\textwidth]{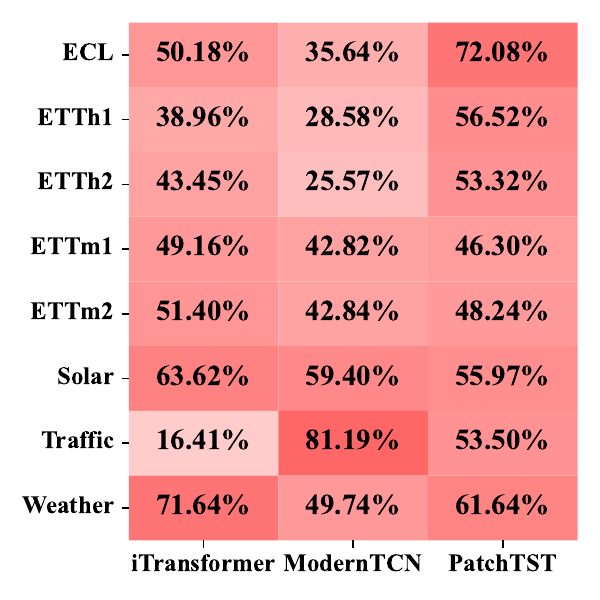}
        \vskip -1em
        \caption{Win ratio (\%) of MLP v.s. teacher models across datasets under input-720-predict-96 setting. The win ratio is generally large (average: 49.92\%, median: 49.96\%), indicating MLP and teacher models excel on different samples with minimal overlap.}
        \label{fig:result_cv}
    \end{minipage}
    \hfill
    \begin{minipage}{0.315\textwidth}
        \centering
        \includegraphics[width=\textwidth]{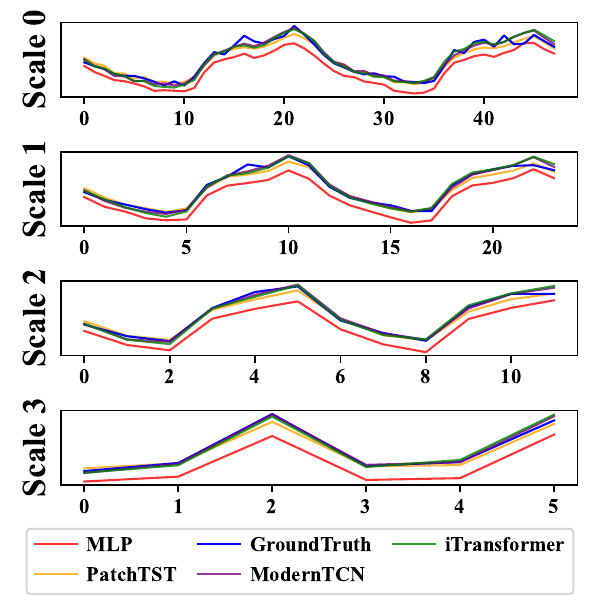}
        \vskip -1em
        \caption{Visualization of model predictions on different downsampled scales of ECL dataset. MLP consistently shows poor performance at multiple scales, while other models perform well, highlighting the importance of capturing multi-scale patterns.}
        \label{fig:result_multi_scale}
    \end{minipage}
    \hfill
    \begin{minipage}{0.32\textwidth}
        \centering
        \includegraphics[width=\textwidth]{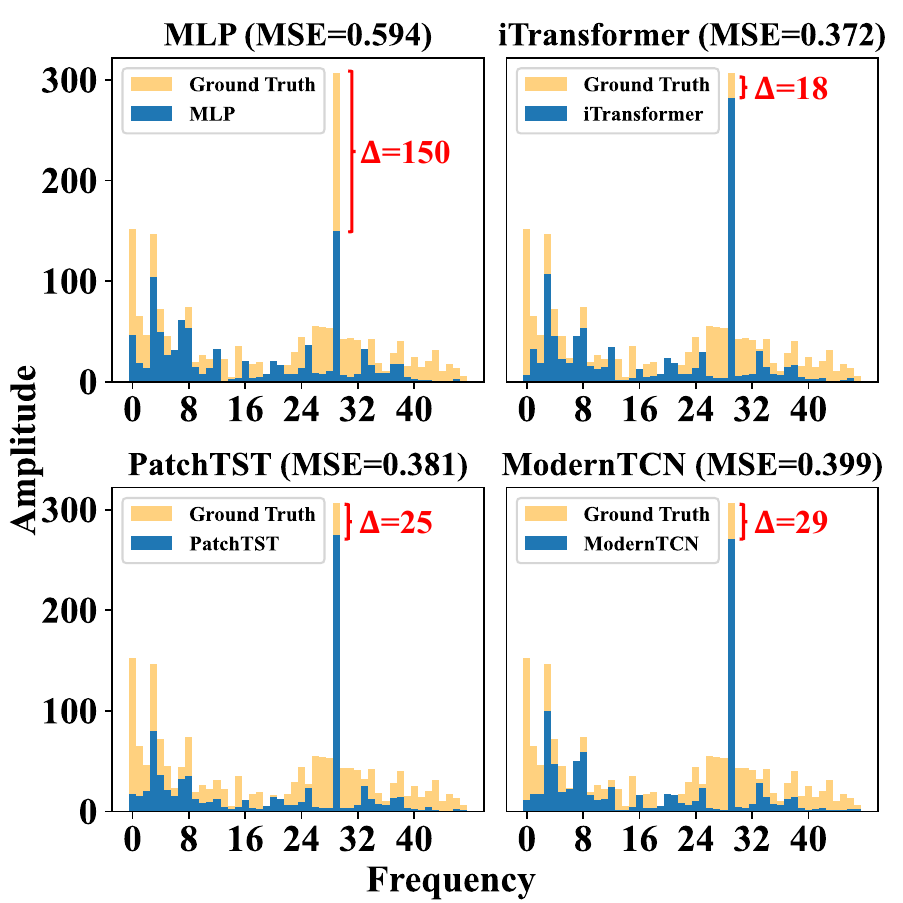}
        \vskip -1em
       \caption{Prediction spectrograms of various models on ECL dataset against the ground truth. MLP fails to match the amplitudes of several main frequencies in the ground truth, with red numbers indicating amplitude differences for the most significant frequency.}
        \label{fig:result_multi_period}
    \end{minipage}
    \vskip -1em
\end{figure*}

\section{Related Work}

\subsection{Debate in Long-Term Time Series Forecast}
The trade-off between performance and efficiency has prompted a long-standing debate between Transformer-based models and MLP in long-term time series forecasting~\cite{dlinear, lightts, sparsetsf}. Informer~\cite{informer}, Autoformer~\cite{autoformer}, and FEDformer~\cite{fedformer} were among the leading Transformer-based methods. However, recent findings show that a simple Linear or MLP model can achieve performance comparable to these complex Transformer models across various benchmarks while offering significantly better efficiency~\cite{dlinear}. This outcome has raised questions about the necessity of Transformers in time series forecasting. Following this, research has moved in two directions. One direction suggests that the issues with Transformer-based models arise from the way they are applied.  For example, PatchTST~\cite{patchtst} uses patching to preserve local information, and iTransformer~\cite{itransformer} focuses on capturing multivariate correlations. These approaches surpass the simple one-layer MLP and demonstrate that Transformers can still deliver strong results in time series forecasting if applied effectively. Meanwhile, CNN-based models have also shown strong performance similar to Transformer-based models. TimesNet~\cite{timesnet} transforms 1D time series into 2D variations and applies 2D CNN kernels, MICN~\cite{micn} adopt multi-scale convolution structures to capture local features and global correlations, and ModernTCN~\cite{moderntcn} proposes a framework with much larger receptive fields than prior CNN-based structures. Nevertheless, these powerful CNN-based models also face efficiency issues, which further broadens the scope of the debate between performance and efficiency. The other direction focuses on developing lightweight MLP, such as N-BEATS~\cite{nbeats}, N-hits~\cite{nhits}, LightTS~\cite{lightts}, FreTS~\cite{frets}, TSMixer~\cite{tsmixer}, TiDE~\cite{tide}, which offer improved efficiency. However, these models typically only match, rather than surpass, the performance of state-of-the-art Transformer-based methods and CNN-based methods. In summary, while Transformer-based and CNN-based models generally offer better performance, simple MLP is more efficient. Therefore, we managed to combine the performance of Transformer-based and CNN-based models with MLP to produce a powerful and efficient model by cross-architecture \textit{KD}.

\subsection{Knolwedge Distillation on Time Series}
Knowledge distillation (\textit{KD})~\cite{hinton2015distilling} transfers knowledge from a larger, more complex model (teacher) to a smaller, simpler model (student) while maintaining comparable performance. By aligning the output distributions of teacher and student models, \textit{KD} provides richer training signals than hard labels alone, enabling the student to capture subtle patterns that the teacher has learned. In the context of time series, CAKD~\cite{xu2022contrastive} introduces a two-stage distillation scheme that distills features using adversarial and contrastive learning and performs prediction-level distillation. LightTS~\cite{campos2023lightts} designs a \textit{KD} framework specifically for cases where the teacher is an ensemble classifier in time series classification tasks, which limits its applicability to teachers with other architectures. Both of these works do not incorporate time series-specific designs. In contrast, our framework emphasizes extracting essential time series patterns, including multi-scale and multi-period patterns, enabling more effective knowledge distillation. Additionally, we are the first to explore cross-architecture \textit{KD}. 
\section{Preliminary Studies}
\label{sec:preliminaries}

In this section, we explore the reasons behind adopting distillation (\textit{\textbf{What motivates distillation?}}) and investigate the specific time series information to distill into the MLP model (\textit{\textbf{What should we distill?}}). 
We first introduce key notations. 
For multivariate long-term time series forecasting, given an input time series \( \mathbf{X} \in \mathbb{R}^{T \times C} \), where \( T \) represents the length of the look-back window and \( C \) represents the number of variables, the goal is to predict the future \( S \) time steps \( \mathbf{Y} \in \mathbb{R}^{S \times C} \). 


\subsection{What Motivates Distillation?}
\label{sec:why_distill}

While MLP excel in efficiency, they often lag in performance compared to Transformer and CNN models or achieve only similar results (Figure~\ref{fig:result_rader}). However, even if MLP underperforms its teachers overall, it may excel on specific samples. To investigate this, we compare prediction errors of MLP's prediction $\mathbf{\hat{Y}}_s$ and teacher model's prediction $\mathbf{\hat{Y}}_t$ for $N$ samples, defined as 
$
e_s = \{\text{MSE}(\mathbf{\hat{Y}}_s^1, \mathbf{Y}^1), \cdots, \\
\text{MSE}(\mathbf{\hat{Y}}_s^N, \mathbf{Y}^N)\}
$
and 
$
e_t = \{\text{MSE}(\mathbf{\hat{Y}}_t^1, \mathbf{Y}^1), \cdots, \text{MSE}(\mathbf{\hat{Y}}_t^N, \mathbf{Y}^N)\}.
$
We calculate \textit{win ratio}, indicating where MLP outperforms teachers:
\begin{equation}
    \text{Win Ratio} = {\sum \mathbbm{1}(e_s < e_t)}/{N},
\end{equation}
where $\mathbbm{1}(\cdot)$ equals 1 if MLP outperforms the teacher, otherwise 0. As shown in Figure~\ref{fig:result_cv}, the win ratio is high (average: 49.92\%, median: 49.96\%), despite MLP underperforming teacher models overall. For example, on the Traffic dataset, MLP lags behind ModernTCN but wins on 81.19\% of the samples, indicating differing strengths on distinct subsets. \textbf{This highlights the potential of distilling complementary knowledge from the teacher model into MLP.}

An intuitive strategy of \textit{KD} for MLP is to align predictions with teachers, but this faces limitations. \textbf{First}, it risks overfitting noise in the teacher’s predictions, leading to less stable knowledge~\cite{chen2017learning, takamoto2020efficient, gajbhiye2021knowledge}. \textbf{Second}, MLP may struggle to replicate the complex patterns such as seasonality, trends, and periodicity in teacher predictions directly due to their limited capacity. \textbf{Third}, this approach overlooks valuable knowledge from intermediate features of teacher model. Therefore, the specific knowledge to distill into MLP requires further exploration.

\subsection{What Should We Distill?}
\label{sec:what2distill}


To investigate which complementary time series patterns to distill into MLP, we analyze prediction patterns in both temporal and frequency domains, considering real-world variations across temporal scales~\cite{timemixer} and periodicities~\cite{timesnet}. Thus, we conduct further analysis of these two patterns by presenting illustrative cases. The implementation is detailed in Appendix~\ref{app:implementation_preliminary}. 
\textbf{(1) Multi-Scale Pattern:}  
By downsampling predictions \(\mathbf{\hat{Y}}\) using convolutional operations, we obtain their multi-scale representations. Figure~\ref{fig:result_multi_scale} shows the analysis of multi-scale temporal patterns by downsampling the time series (\textit{Scale 0}) to coarser scales (\textit{Scales 1–3}). Models excelling at \textit{Scale 0} generally perform well across all scales. The teacher models align closely with the ground truth at \textit{Scale 3}, capturing underlying trends effectively. In contrast, MLP significantly deviates, showing its limitations in handling multi-scale patterns.
\textbf{(2) Multi-Period Pattern:} Periodicities in time series are visible in the frequency domain by transforming time series into spectrograms, with the x-axis representing frequency and periodicity calculated as the time series length $S$ divided by frequency. 
Figure~\ref{fig:result_multi_period} shows spectrograms of model predictions and the ground truth. Models with lower MSE capture periodicities more accurately, closely matching the ground truth at dominant frequencies. In contrast, MLP predictions show larger discrepancies, highlighting its inability to effectively capture periodic patterns.


From our observations, we conjecture that MLP underperforms on certain samples due to their inability to capture essential multi-scale and multi-period patterns, which are essential for effective time series forecasting. To improve MLP, it is crucial to incorporate these complementary patterns from teachers during distillation.

\begin{figure*}[ht]
    \centering
    \includegraphics[width=0.98\textwidth]{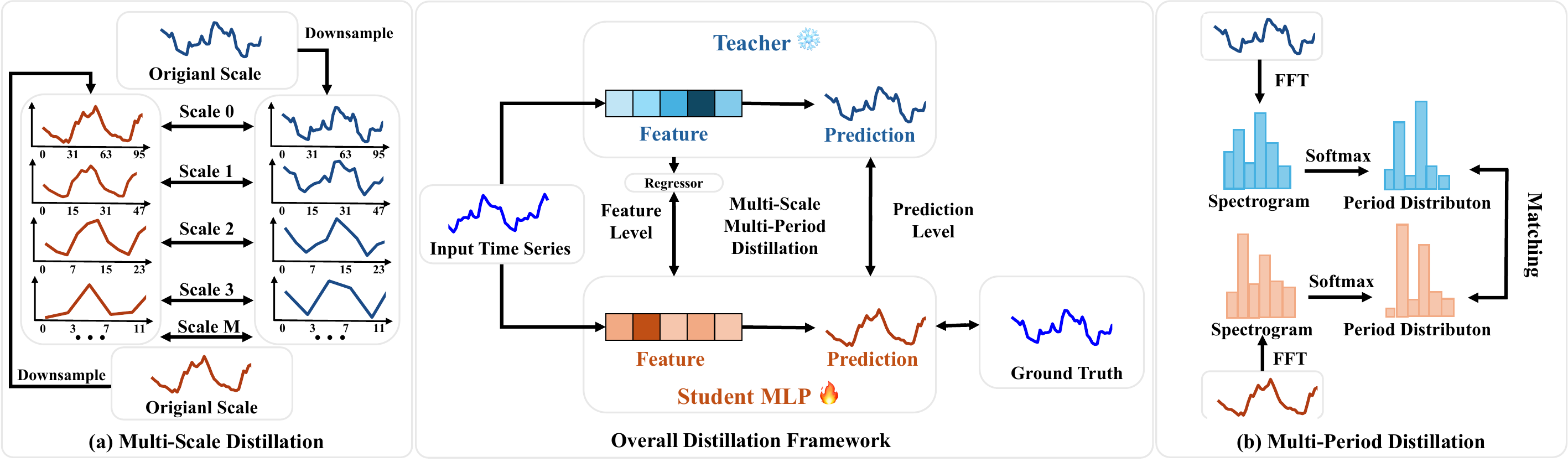}
    \vskip -1.5em
    \caption{Overall framework of \method{}, which distills knowledge from a teacher model to a student MLP using (a) Multi-Scale Distillation and (b) Multi-Period Distillation at both feature and prediction levels. (a) Multi-Scale Distillation involves downsampling the original time series into multiple coarser scales and aligning these scales between the student and teacher. (b) Multi-Period Distillation applies FFT to transform the time series into a spectrogram, followed by matching the period distributions after applying softmax.}
    \label{fig:method}
    \vskip -1em
\end{figure*}

\section{Methodology}




Motivated by our preliminary studies, we propose a novel \textit{KD} framework \method{} for time series to transfer the knowledge from a fixed, pretrained teacher model \(f_t\) to a student MLP model \(f_s\). The student produces predictions \(\mathbf{\hat{Y}}_s \in \mathbb{R}^{S \times C}\) and internal features \(\mathbf{H}_s\in \mathbb{R}^{D \times C}\). The teacher model produces predictions \(\mathbf{\hat{Y}}_t \in \mathbb{R}^{S \times C}\) and internal features \(\mathbf{H}_t\in \mathbb{R}^{D_t \times C}\). Our general objective is:
\begin{equation}\label{eq:kd_obj}
    \min\nolimits_{\theta_s} \mathcal{L}_{sup}(\mathbf{Y}, \mathbf{\hat{Y}}_s) + \mathcal{L}_{\mathrm{KD}}^\mathbf{Y}(\mathbf{\hat{Y}}_t, \mathbf{\hat{Y}}_s) + \mathcal{L}_{\mathrm{KD}}^\mathbf{H}(\mathbf{H}_t, \mathbf{H}_s),
\end{equation}
where \(\theta_s\) is the parameter of the student; \(\mathcal{L}_{sup}\) is the supervised loss (e.g., MSE) between predictions and ground truth; \(\mathcal{L}_{\mathrm{KD}}^\mathbf{Y}\) and \(\mathcal{L}_{\mathrm{KD}}^\mathbf{H}\) are the distillation loss terms that encourage student model to learn knowledge from teacher on both \textbf{prediction level}~\cite{hinton2015distilling} and \textbf{feature level}~\cite{romero2014fitnets}. Unlike conventional approaches that emphasize matching model predictions, \method{} integrates key time-series patterns including multi-scale and multi-period knowledge. The overall framework of \method{} is shown in Figure~\ref{fig:method}. 


\subsection{Multi-Scale Distillation}
One key component of \method{} is multi-scale distillation, where ``multi-scale'' refers to representing the same time series at different sampling rates. This enables MLP to effectively capture both coarse-grained and fine-grained patterns. By distilling at both the prediction level and the feature level, we ensure that MLP not only replicates the teacher's multi-scale predictions but also aligns with its internal representations from the intermediate layer.

\paragraph{Prediction Level.}
At the prediction level, we directly transfer multi-scale signals from the teacher’s outputs to guide the MLP’s predictions. We first produce multi-scale predictions by downsampling the original predictions from the teacher \(\mathbf{\hat{Y}}_t \in \mathbb{R}^{S \times C}\) and the MLP \(\mathbf{\hat{Y}}_s \in \mathbb{R}^{S \times C}\), where \(S\) is the prediction length and \(C\) is the number of variables. The predictions at \textit{Scale 0} are equal to the original predictions, that is, \(\mathbf{\hat{Y}}_t^0=\mathbf{\hat{Y}}_t\) and \(\mathbf{\hat{Y}}_s^0=\mathbf{\hat{Y}}_s\). We then downsample these predictions across \(M\) scales using convolutional operations with a stride of 2, generating multi-scale prediction sets \(\mathcal{Y}_t = \{\mathbf{\hat{Y}}_t^0, \mathbf{\hat{Y}}_t^1,\cdots,\mathbf{\hat{Y}}_t^M\}\) and \(\mathcal{Y}_s = \{\mathbf{\hat{Y}}_s^0, \mathbf{\hat{Y}}_s^1,\cdots,\mathbf{\hat{Y}}_s^M\}\), where \(\mathbf{\hat{Y}}_t^M, \mathbf{\hat{Y}}_s^M \in \mathbb{R}^{\lfloor S/2^M \rfloor \times C}\). The downsampling is defined as: 
\begin{equation}
    \mathbf{\hat{Y}}_x^m = \mathrm{Conv}(\mathbf{\hat{Y}}_x^{m-1}, \mathrm{stride}=2),
    \label{eq:multiscale_downsample}
\end{equation}
where \(x \in \{t, s\}\), \(m \in \{1, \cdots, M\}\), $\mathrm{Conv}$ denotes a 1D-convolutional layer with a temporal stride of 2. The predictions at the lowest level \(\mathbf{\hat{Y}}_x^0=\mathbf{\hat{Y}}_x\) maintain the original temporal resolution, while the highest-level predictions \(\mathbf{\hat{Y}}_x^M\) represent coarser patterns. We define the multi-scale distillation loss at the prediction level as:
\begin{equation}
    \mathcal{L}_{scale}^\mathbf{Y} = \textstyle\sum_{m=0}^M ||\mathbf{\hat{Y}}_t^m - \mathbf{\hat{Y}}_s^m||^2 /(M+1).
\end{equation}
Here we use MSE loss to match the MLP’s predictions to the teacher’s predictions at multiple scales.

\paragraph{Feature Level.} 
At the feature level, we align MLP’s intermediate features with teacher’s multi-scale representations, enabling MLP to form richer internal structures that support more accurate forecasts.
Let \(\mathbf{H}_s \in \mathbb{R}^{D \times C}\) and \(\mathbf{H}_t \in \mathbb{R}^{D_t \times C}\) denote MLP and teacher features with feature dimensions \(D\) and \(D_t\), respectively. As their dimensions can be different, we first use a parameterized regressor (e.g. MLP) to align their feature dimensions: 
\begin{equation}
    \mathbf{H}'_t = \text{Regressor}(\mathbf{H}_t),
\end{equation}
where \(\mathbf{H}'_t \in \mathbb{R}^{D \times C}\).  
Similar to the prediction level, we compute $\mathbf{H}_x^m$ by downsampling $\mathbf{H}_s$ and $\mathbf{H}'_t$ across multiple scales using the same approach as in Equation~\ref{eq:multiscale_downsample}. We define the multi-scale distillation loss at the feature level as:
\begin{equation}
    \mathcal{L}_{scale}^\mathbf{H} = \textstyle\sum_{m=0}^M ||\mathbf{H}_t^m - \mathbf{H}_s^m||^2 /(M+1).
\end{equation}

\subsection{Multi-Period Distillation}
{In addition to multi-scale distillation in the temporal domain, we further
propose multi-period distillation to help MLP learn complex periodic patterns in the frequency domain.} By aligning periodicity-related signals from the teacher model at both the prediction and feature levels, the MLP can learn richer frequency-domain representations and ultimately improve its forecasting performance.

\paragraph{Prediction Level.}
For the predictions from the teacher \(\mathbf{\hat{Y}}_t \in \mathbb{R}^{S \times C}\) and the MLP \(\mathbf{\hat{Y}}_s \in \mathbb{R}^{S \times C}\), we first identify their periodic patterns. We perform this in the frequency domain using the Fast Fourier Transform (FFT):
\begin{equation}
    \mathbf{A}_x = \text{Amp}(\text{FFT}(\mathbf{\hat{Y}}_x)),
    \label{eq:multiperiod_spectrograms}
\end{equation}
where \(x \in \{t, s\}\) and spectrograms \(\mathbf{A}_x \in \mathbb{R}^{\frac{S}{2} \times C}\). Here, \(\text{FFT}(\cdot)\) denotes the FFT operation and \(\text{Amp}(\cdot)\) calculates the amplitude. We remove the direct current (DC) component from \(\mathbf{A}_x\). For certain variable \(c\), the \(i\)-th value \(\mathbf{A}_x^{i,c}\) indicates the intensity of the frequency-\(i\) component, corresponding to a period length \(\lceil S/i\rceil\). Larger amplitude values indicate that the associated frequency (period) is more significant.

To reduce the influence of minor frequencies and avoid noise introduced by less meaningful frequencies~\cite{timesnet, fedformer}, we propose a distribution-based matching scheme. We use softmax function with a colder temperature to highlight the most significant frequencies:
\begin{equation}
    \mathbf{Q}_x^\mathbf{Y} = {\exp\bigl(\mathbf{A}_x^i / \tau\bigr)}/{\sum\nolimits_{j=1}^{S/2} \exp\bigl(\mathbf{A}_x^j /\tau\bigr)},
    \label{eq:multiperiod_distribution}
\end{equation}
where \(\mathbf{Q}_x^\mathbf{Y} \in \mathbb{R}^{\frac{S}{2} \times C}\) and \(\tau\) is a temperature parameter that controls the sharpness of the distribution. We set \(\tau=0.5\) by default. The period distribution \(\mathbf{Q}_x^\mathbf{Y}\) represents the multi-period pattern in the prediction time series, which we want the MLP to learn from the teacher. We use KL divergence to match these distributions~\cite{hinton2015distilling}. We define multi-period distillation loss at the prediction level as:
\begin{equation}
    \mathcal{L}_{period}^\mathbf{Y} = \text{KL}\bigl(\mathbf{Q}_t^\mathbf{Y}, \mathbf{Q}_s^\mathbf{Y}\bigr).
\end{equation}


\paragraph{Feature Level.}
Similar to the prediction level, we apply multi-period distillation at the feature level. For the features \(\mathbf{H}'_t \in \mathbb{R}^{D \times C}\) and \(\mathbf{H}_s \in \mathbb{R}^{D \times C}\), we compute the spectrograms and the corresponding period distributions \(\mathbf{Q}_x^\mathbf{H}\) using the same approach as in Equations~\ref{eq:multiperiod_spectrograms} and~\ref{eq:multiperiod_distribution}. Multi-period distillation loss at feature level is then defined as:
\begin{equation}
    \mathcal{L}_{period}^\mathbf{H} = \text{KL}\bigl(\mathbf{Q}_t^\mathbf{H}, \mathbf{Q}_s^\mathbf{H}\bigr).
\end{equation}

\subsection{Overall Optimization and Theoretical Analysis}
\label{sec:theoretical_analysis}
During the training of \method{}, we jointly optimize both the multi-scale and multi-period distillation losses at both the prediction and feature levels, together with the supervised ground-truth label loss:
\begin{equation}
    \mathcal{L}_{sup} = ||\mathbf{Y} - \mathbf{\hat{Y}}_s||^2,
\end{equation}
where \(\mathcal{L}_{sup}\) is the ground-truth loss (for example, MSE loss) used to train MLP directly. Thus, the overall training loss for the student MLP is defined as:
\begin{equation}
    \mathcal{L} = \mathcal{L}_{sup} + \alpha \cdot \bigl(\mathcal{L}_{scale}^\mathbf{Y} + \mathcal{L}_{period}^\mathbf{Y}\bigr) + \beta \cdot \bigl(\mathcal{L}_{scale}^\mathbf{H} + \mathcal{L}_{period}^\mathbf{H}\bigr),
    \label{eq:overall_optimization}
\end{equation}
where \(\alpha\) and \(\beta\) are hyper-parameters that control the contributions of the prediction-level and feature-level distillation loss terms, respectively. The teacher model is pretrained and remains frozen throughout the training process of MLP.

\underline{\textbf{Theoretical Interpretations.}} We provide a theoretical understanding of multi-scale and multi-period distillation loss from \textbf{a novel data augmentation perspective}. We further show that the proposed distillation loss can be interpreted as training with augmented samples derived from a special \textit{mixup}~\cite{mixup} strategy. The distillation process augments data by blending ground truth with teacher predictions, analogous to label smoothing in classification, and provides several benefits for time series forecasting:
\textit{\textbf{(1)} Enhanced Generalization:} It enhances generalization by exposing the student model to richer supervision signals from augmented samples, thus mitigating overfitting, especially with limited or noisy data.
{\textit{\textbf{(2)} Explicit Integration of Patterns:} The augmented supervision signals explicitly incorporate patterns across multiple scales and periods, offering insights that are not immediately evident in the raw ground truth.}
\textit{\textbf{(3}) Stabilized Training Dynamics:} The blending of targets softens the supervision signals, which diminishes the model’s sensitivity to noise and leads to more stable training phases. This will in turn support smoother optimization dynamics and fosters improved convergence. For clarity, our discussion is centered at the prediction level.  We present the following theorem:  
\begin{theorem} \label{thm:multiscale}
Let $(x, y)$ denote original input data pairs and $(x, y^t)$ represent corresponding teacher data pairs. Consider a data augmentation function $\mathcal{A}(\cdot)$ applied to $(x, y)$, generating augmented samples $(x', y')$. Define the training loss on these augmented samples as $\mathcal{L}_{aug} = \sum_{(x',y') \in \mathcal{A}(x,y)} |f_s(x') - y'|^2$. Then, the following inequality holds: 
$
   \mathcal{L}_{sup} + \eta \mathcal{L}_{scale} \geq \mathcal{L}_{aug},
$
when $\mathcal{A}(\cdot)$ is instantiated as a mixup function~\cite{mixup} that interpolates between the original input data $(x,y)$ and teacher data $(x,y^t)$ with a mixing coefficient $\lambda=\frac{1}{1+\eta}$, i.e. $y' = \lambda y + (1-\lambda) y^t$.
\end{theorem}
We provide proof of Theorem~\ref{thm:multiscale} in Appendix~\ref{app:theory}.  Theorem~\ref{thm:multiscale} suggests that optimizing multi-scale distillation loss \(\mathcal{L}_{\text{scale}}\) jointly with supervised loss \(\mathcal{L}_{\text{sup}}\) is equivalent to minimizing an upper bound on a special \textit{mixup} augmentation loss. In particular, we mix multi-scale teacher predictions \(\{\mathbf{\hat{Y}}_t^{(m)}\}_{m=0}^M\) with ground truth \(\mathbf{Y}\), allowing MLP to learn more informative time series temporal pattern. Similarly, we present a theorem for understanding $\mathcal{L}_{period}$.

\begin{theorem} \label{thm:multiperiod}
Define the training loss on the augmented samples using KL divergence as $\mathcal{L}_{aug} = \sum_{(x',y') \in \mathcal{A}(x,y)} \text{KL}\big(y', \mathcal{X}(f_s(x'))\big)$, where $\mathcal{X}(\cdot) = \text{Softmax}(\text{FFT}(\cdot))$. Then, the following inequality holds: 
$
   \mathcal{L}_{sup} + \eta\mathcal{L}_{period} \geq \mathcal{L}_{aug},
$
where $\mathcal{A}(\cdot)$ is instantiated as a mixup function that interpolates between the period distribution of original input data $(x,\mathcal{X}(y))$ and teacher data $(x,\mathcal{X}(y^t))$ with a mixing coefficient $\lambda=\eta$, i.e. $y' =  \mathcal{X}(y) + \lambda \mathcal{X}(y^t)$.
\end{theorem}
The proof can be found in Appendix~\ref{app:theory}. Theorem~\ref{thm:multiperiod} shows that optimizing the multi-period distillation loss \(\mathcal{L}_{\text{period}}\) jointly with the supervised loss \(\mathcal{L}_{\text{sup}}\) is equivalent to minimizing an upper bound on the KL divergence between the student period distribution \(\mathcal{X}(f_s(x'))\) (or \(\mathbf{Q}_s\)) and a \emph{mixed} period distribution \(y'\) (or \(\mathbf{Q}_y + \lambda\,\mathbf{Q}_t\)). 
\section{Experiment}
\begin{table*}[t]
\centering
\caption{Long-term time series forecasting results with prediction lengths $S~\in\{96, 192, 336, 720\}$. A lower MSE or MAE indicates a better prediction. For consistency, we maintain a fixed input length of 720 throughout all the experiments. Results are averaged from all prediction lengths. The best performance is highlighted in \textcolor{red}{\textbf{red}}, and the second-best is \textcolor{blue}{\underline{underlined}}. Full results are listed in Appendix~\ref{app:full_main_results}.}
\vskip -1.5em
\resizebox{\textwidth}{!}{%
\begin{tabular}{c lr|lr|lr|lr|lr|lr|lr|lr|lr}
\toprule
\multirow{2}{*}{Models} & \multicolumn{2}{c}{\textbf{\method{}}} & \multicolumn{2}{c}{iTranformer} & \multicolumn{2}{c}{ModernTCN} & \multicolumn{2}{c}{TimeMixer}& \multicolumn{2}{c}{PatchTST}  & \multicolumn{2}{c}{MICN} & \multicolumn{2}{c}{FEDformer} & \multicolumn{2}{c}{TimesNet} & \multicolumn{2}{c}{Autoformer}  \\ 
& \multicolumn{2}{c}{(\textbf{Ours})} & \multicolumn{2}{c}{(\citeyear{itransformer})} & \multicolumn{2}{c}{(\citeyear{moderntcn})} & \multicolumn{2}{c}{(\citeyear{timemixer})} &\multicolumn{2}{c}{(\citeyear{patchtst})} & \multicolumn{2}{c}{(\citeyear{micn})} & \multicolumn{2}{c}{(\citeyear{fedformer})} & \multicolumn{2}{c}{(\citeyear{timesnet})} &  \multicolumn{2}{c}{(\citeyear{autoformer})} \\
\cmidrule(lr){2-3} \cmidrule(lr){4-5} \cmidrule(lr){6-7} \cmidrule(lr){8-9} \cmidrule(lr){10-11} \cmidrule(lr){12-13} \cmidrule(lr){14-15} \cmidrule(lr){16-17} \cmidrule(lr){18-19}
Metric & \multicolumn{1}{|c}{MSE} & MAE & \multicolumn{1}{|c}{MSE} & MAE & \multicolumn{1}{|c}{MSE} & MAE & \multicolumn{1}{|c}{MSE} & MAE & \multicolumn{1}{|c}{MSE} & MAE & \multicolumn{1}{|c}{MSE} & MAE & \multicolumn{1}{|c}{MSE} & MAE & \multicolumn{1}{|c}{MSE} & MAE & \multicolumn{1}{|c}{MSE} & MAE \\
\midrule
ECL &\multicolumn{1}{|c}{\textcolor{red}{\textbf{0.157}}}	&\textcolor{red}{\textbf{0.254}}	&\textcolor{blue}{\underline{0.163}}	&0.259	&0.167	&0.262	&0.165	&\textcolor{blue}{\underline{0.259}} &0.165	&0.266	&0.181	&0.293	&0.274	&0.376	&0.250	&0.347	&0.238	&0.347 \\ \midrule
ETTh1 &\multicolumn{1}{|c}{\textcolor{red}{\textbf{0.429}}}	&\textcolor{red}{\textbf{0.441}}	&0.468	&0.476	&0.469	&0.465 &\textcolor{blue}{\underline{0.459}}	&\textcolor{blue}{\underline{0.465}}	&0.498	&0.490	&0.739	&0.631	&0.527	&0.524	&0.507	&0.500	&0.731	&0.659 \\ \midrule
ETTh2	&\multicolumn{1}{|c}{\textcolor{red}{\textbf{0.345}}}	&\textcolor{red}{\textbf{0.395}}	&0.398	&0.426	&\textcolor{blue}{\underline{0.357}}	&\textcolor{blue}{\underline{0.403}}	&0.422	&0.444 &0.444	&0.443	&1.078	&0.736	&0.456	&0.485	&0.419	&0.446	&1.594	&0.940 \\ \midrule
ETTm1	&\multicolumn{1}{|c}{\textcolor{red}{\textbf{0.348}}}	&\textcolor{red}{\textbf{0.379}}	&0.372	&0.402	&0.390	&0.410	&\textcolor{blue}{\underline{0.367}}	&\textcolor{blue}{\underline{0.388}} &0.383	&0.412	&0.439	&0.461	&0.423	&0.451	&0.398	&0.419	&0.570	&0.526 \\ \midrule
ETTm2	&\multicolumn{1}{|c}{\textcolor{red}{\textbf{0.244}}}	&\textcolor{red}{\textbf{0.311}}	&0.276	&0.337	&\textcolor{blue}{\underline{0.267}}	&\textcolor{blue}{\underline{0.330}}	&0.279	&0.339 &0.274	&0.335	&0.348	&0.404	&0.359	&0.401	&0.291	&0.349	&0.420	&0.448 \\ \midrule
Solar	&\multicolumn{1}{|c}{\textcolor{red}{\textbf{0.184}}}	&\textcolor{red}{\textbf{0.241}}	&0.214	&0.270	&\textcolor{blue}{\underline{0.191}}	&\textcolor{blue}{\underline{0.243}}	&0.238	&0.288 &0.210	&0.257	&0.213	&0.277	&0.300	&0.383	&0.196	&0.262	&1.037	&0.742 \\ \midrule
Traffic	&\multicolumn{1}{|c}{\textcolor{blue}{\underline{0.387}}}	&\textcolor{red}{\textbf{0.271}}	&\textcolor{red}{\textbf{0.379}}	&\textcolor{blue}{\underline{0.271}}	&0.413	&0.284	&0.391	&0.275 &0.402	&0.284	&0.500	&0.316	&0.629	&0.388	&0.693	&0.399	&0.696	&0.427 \\ \midrule
Weather	&\multicolumn{1}{|c}{\textcolor{red}{\textbf{0.220}}}	&\textcolor{red}{\textbf{0.269}}	&0.259	&0.290	&0.238	&0.277	&\textcolor{blue}{\underline{0.230}}	&\textcolor{blue}{\underline{0.271}} &0.246	&0.283	&0.240	&0.292	&0.355	&0.398	&0.257	&0.294	&0.471	&0.465 \\
\bottomrule

\end{tabular}%
}
\label{tab:main}
\end{table*}

\subsection{Experimental Setup}

\textbf{Datasets and Baselines.} We run experiments to evaluate the performance and efficiency of \method{} on 8 widely used benchmarks: Electricity (ECL), the ETT datasets (ETTh1, ETTh2, ETTm1, ETTm2), Solar, Traffic, and Weather, following~\cite{itransformer, timemixer, moderntcn}.  
To examine the effectiveness of our method across diverse tasks, we compare \method{} with 8 baseline models that cover a range of architectures. Specifically, we use \textit{Transformer-based models}: iTransformer~\cite{itransformer}, PatchTST~\cite{patchtst}, FEDformer~\cite{fedformer}, Autoformer~\cite{autoformer}; \textit{CNN-based models}: ModernTCN~\cite{moderntcn}, MICN~\cite{micn}, TimesNet~\cite{timesnet}; and an \textit{MLPs-based model}: TimeMixer~\cite{timemixer}.  More details can be found in Appendix~\ref{app:implementation_details}.


\subsection{Main Results}


Table~\ref{tab:main} presents the long-term time series forecasting performance of the proposed \method{} compared with previous state-of-the-art baselines. By default, \method{} uses ModernTCN as the teacher, though results with alternative teachers are provided in Sec.~\ref{sec:versatility}. Notably, \method{} outperforms the baselines on \textbf{7 out of 8 }datasets on MSE and \textbf{all} datasets on MAE. Furthermore, \method{} consistently exceeds the performance of its teacher (ModernTCN) by up to \textbf{5.37\%} and improves over vanilla MLP by up to \textbf{13.87\%}, as shown in Table~\ref{tab:different_teacher}. These results highlight the effectiveness of our multi-scale and multi-period distillation approach in transferring knowledge for enhanced forecasting performance.

\textbf{Efficiency Analysis.}
Beyond its strong predictive performance, another notable advantage of \method{} is its extremely lightweight architecture, as it is simply an MLP. Figure~\ref{fig:result_efficiency} in the Introduction section shows the trade-off between inference time, memory footprint, and performance. We can observe that \method{} can achieve up to \textbf{7×} speedup and up to \textbf{130×} fewer parameters compared with baselines. This property makes \method{} suitable for deployment on devices with limited computational resources and in latency-sensitive applications that require fast inference. Compared to previous Transformer-based method Autoformer, we achieve \textbf{196×} speedup as shown in Table~\ref{tab:app_efficiency}. We list full results of efficiency analysis in Appendix~\ref{app:efficiency}.

\subsection{Versatility of \method{}}
\label{sec:versatility}

\begin{table}[ht!]
\centering
\caption{Performance improvement by \textbf{\method{}} with \textbf{different teachers}. $\Delta_{MLP}$, $\Delta_{Teacher}$ indicate the improvement of \textbf{MLP+\method{}} over a trained MLP and Teacher, respectively. We report average MSE across all prediction lengths. Full results are in Appendix~\ref{app:full_different_teacher_results}.}
\vskip -1.5em
\label{tab:different_teacher}
\small
\resizebox{0.48\textwidth}{!}{%
\begin{tabular}{cc c|c|c|c}
\toprule
\multicolumn{2}{c}{\multirow{2}{*}{Teacher Models}} & \multicolumn{1}{c}{iTranformer} & \multicolumn{1}{c}{ModernTCN} & \multicolumn{1}{c}{TimeMixer} & \multicolumn{1}{c}{PatchTST} \\ 
& & (\citeyear{itransformer}) & (\citeyear{moderntcn}) & (\citeyear{timemixer}) & (\citeyear{patchtst}) \\
\midrule

\multirow{5}{*}{ECL} 
& Teacher         & 0.163 & 0.167 & 0.165 & 0.165 \\ 
& MLP             & 0.173 & 0.173 & 0.173 & 0.173 \\
& \textbf{+\method{} }   & \textbf{0.157} & \textbf{0.157} & \textbf{0.159} & \textbf{0.159} \\
\cmidrule(lr){2-6} 
& $\Delta_{Teacher}$  & 3.68\% & 5.61\% & 3.31\% & 3.64\% \\
& $\Delta_{MLP}$      & 9.27\% & 9.09\% & 7.94\% & 8.11\% \\
\midrule

\multirow{5}{*}{ETT(avg)} 
& Teacher         & 0.379 & 0.371 & 0.382 & 0.400 \\ 
& MLP             & 0.397 & 0.397 & 0.397 & 0.397 \\
& \textbf{+\method{}}    & \textbf{0.345} & \textbf{0.342} & \textbf{0.353} & \textbf{0.358} \\
\cmidrule(lr){2-6} 
& $\Delta_{Teacher}$  & 8.92\% & 7.94\% & 7.46\% & 10.38\% \\
& $\Delta_{MLP}$      & 13.06\% & 13.87\% & 10.91\% & 9.65\% \\
\midrule

\multirow{5}{*}{Solar} 
& Teacher         & 0.214 & 0.191 & 0.288 & 0.210 \\ 
& MLP             & 0.194 & 0.194 & 0.194 & \textbf{0.194} \\
& \textbf{+\method{}}    & \textbf{0.185} & \textbf{0.184} & \textbf{0.187} & 0.204 \\
\cmidrule(lr){2-6} 
& $\Delta_{Teacher}$  & 13.55\% & 3.60\% & 21.41\% & 2.86\% \\
& $\Delta_{MLP}$      & 4.80\% & 5.14\% & 3.58\% & -4.98\% \\
\midrule

\multirow{5}{*}{Traffic} 
& Teacher         & \textbf{0.379} & 0.413 & 0.391 & 0.402 \\ 
& MLP             & 0.434 & 0.434 & 0.434 & 0.434 \\
& \textbf{+\method{}}    & 0.389 & \textbf{0.387} & \textbf{0.391} & \textbf{0.390} \\
\cmidrule(lr){2-6} 
& $\Delta_{Teacher}$  & -2.64\% & 6.32\% & -0.04\% & 2.99\% \\
& $\Delta_{MLP}$      & 10.30\% & 10.70\% & 9.76\% & 10.07\% \\
\midrule

\multirow{5}{*}{Weather} 
& Teacher         & 0.259 & 0.238 & 0.230 & 0.246 \\ 
& MLP             & 0.234 & 0.234 & 0.234 & 0.234 \\
& \textbf{+\method{}}    & \textbf{0.220} & \textbf{0.220} & \textbf{0.219} & \textbf{0.220} \\
\cmidrule(lr){2-6} 
& $\Delta_{Teacher}$  & 15.06\% & 7.37\% & 4.82\% & 10.57\% \\
& $\Delta_{MLP}$      & 5.83\% & 5.66\% & 6.16\% & 5.83\% \\
\bottomrule

\end{tabular}%
}
\end{table}

\begin{figure}[t] \centering
\includegraphics[width=0.45\textwidth]{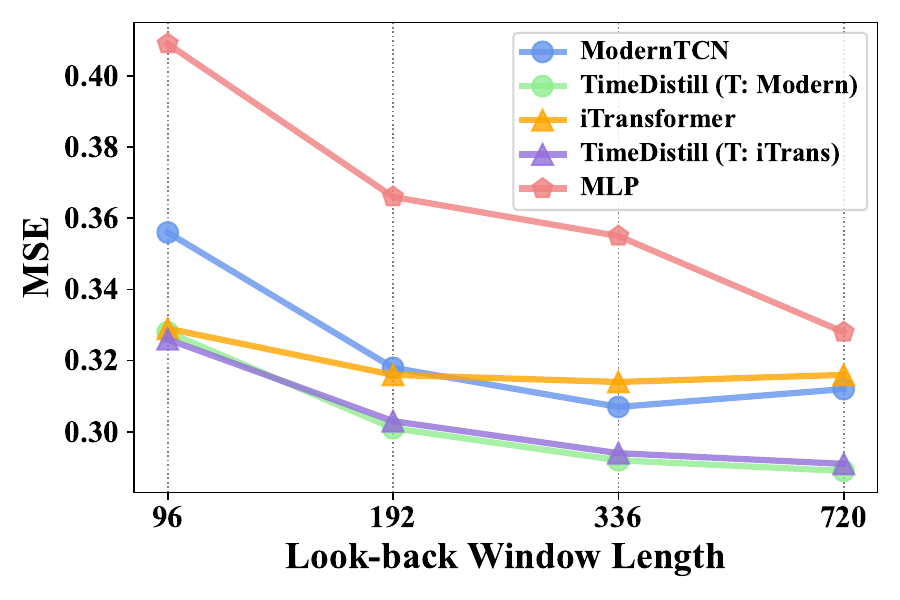}
\vskip -1em
    \caption{Forecasting performance is evaluated with different look-back lengths and the results are averaged across all prediction lengths $S \in \{96, 192, 336, 720\}$ on eight datasets. The complete results are provided in Appendix~\ref{app:full_diferent_lookback_window_results}.}
    \label{fig:result_look_back_window}
\end{figure}
In this subsection, we further demonstrate the versatility of the proposed \method{} by evaluating its performance under different configurations, including variations in teacher/student models and look-back window lengths.


\textbf{Different Teachers.}
We use ModernTCN as the teacher model for our main results. In Table~\ref{tab:different_teacher}, we present results using other teacher models, such as iTransformer, TimeMixer, and PatchTST. Results with additional teachers are provided in Appendix~\ref{app:full_different_teacher_results}. We observe that \method{} can effectively learn from various teachers, improving MLP by up to \textbf{13.87\%}. Furthermore, \method{} achieves significant performance improvements over the teachers themselves, with gains of up to \textbf{21.41\%}. We hypothesize two key reasons for these improvements. First, the student MLP model already demonstrates strong learning capabilities; for instance, on the Solar dataset, MLP outperforms most teachers. Second, the multi-scale and multi-period \textit{KD} approach delivers both temporal and frequency knowledge from teachers to MLP, offering additional valuable insights. Similar findings have been reported in recent \textit{KD} studies~\cite{allen2020towards, zhang2021graph, guo2023linkless}, which suggest that integrating diverse views from multiple models can enhance performance.

\textbf{Different Students.} 
To evaluate whether \method{} can enhance the performance of other lightweight models, we select two MLPs-based models, TSMixer~\cite{tsmixer}, LightTS~\cite{lightts}, as students. Unlike a simple MLP, these MLPs-based models consist of multiple MLPs and operate in a channel-dependent manner. Consequently, they are generally more complex than a standard MLP, which results in reduced efficiency. Besides, we select a Linear-based model FITS~\cite{fits}, which is extremely lightweight.
We use ModernTCN as the teacher model, consistent with other experiments. As shown in Table~\ref{tab:different_student}, \method{} consistently improves the performance of TSMixer, LightTS, and FITS, achieving remarkable MSE reductions of \textbf{6.26\%}, \textbf{8.02\%}, and \textbf{3.96\%}, respectively. These results demonstrate that \method{} is highly adaptable and can effectively enhance other lightweight methods. We also explore the influence of student MLP architecture (e.g. layer number and hidden dimension) on \method{} in Appendix~\ref{app:different_students}.

\begin{table}[ht!]
\centering
\vskip -0.5em
\caption{Performance improvement of \textbf{\method{}} with \textbf{different students} on ETTh1, averaged across all prediction lengths. $\Delta_{Student}$ represents the improvement of \textbf{Student+\method{}} over the original student. Additional results are in Appendix~\ref{app:different_students}.}

\label{tab:different_student}
\resizebox{0.49\textwidth}{!}{%
\begin{tabular}{c cc|cc|cc|cc}
\toprule
\multicolumn{1}{c}{\multirow{2}{*}{Student Models}} 
& \multicolumn{2}{c}{MLP}
& \multicolumn{2}{c}{LightTS} 
& \multicolumn{2}{c}{TSMixer} 
& \multicolumn{2}{c}{FITS}\\
& MSE & MAE & MSE & MAE & MSE & MAE & MSE & MAE \\
\midrule
Student
& 0.502 & 0.489
& 0.465 & 0.471
& 0.471 & 0.474 
& 0.428	& 0.443\\
\textbf{+\method{}} 
& \textbf{0.428} & \textbf{0.445}
& \textbf{0.436} & \textbf{ 0.445}
& \textbf{0.433} & \textbf{0.446} 
& \textbf{0.411} & \textbf{0.430}\\
\cmidrule{1-9}
$\Delta_{Student}$
& 14.74\% & 9.00\%
& 6.26\%  & 5.57\%
& 8.02\%  & 5.92\%
& 3.96\%  & 2.75\%\\
\bottomrule
\end{tabular}%
}
\end{table}

\textbf{Different Look-Back Window Lengths.}
The length of the look-back window significantly influences forecasting accuracy, as it determines how much historical information can be utilized for learning. Figure~\ref{fig:result_look_back_window} presents the average MSE results across all eight datasets. Overall, the performance of models in the figure, particularly MLP, generally improves as the look-back window size increases. Notably, \method{} consistently enhances the performance of MLP and outperforms the teacher models across all look-back window lengths.


\subsection{Deeper Analysis into Our Distillation Framework}
\label{sec:deeper_analysis}
\begin{table}[t]
\centering

\caption{Ablation study measured by MSE on different components of \method{} (Teacher: ModernTCN). The best performance is in \textcolor{red}{\textbf{red}}, and the second-best is \textcolor{blue}{\underline{underlined}}. More ablation study results are listed in Appendix~\ref{app:full_ablation_results}.}
\label{tab:ablation}
\vskip -1em
\resizebox{0.48\textwidth}{!}{
\begin{tabular}{l c c c c c}
\toprule
Method & ECL & ETT(avg) & Solar & Traffic & Weather \\
\midrule
Teacher            & 0.167 & 0.371 & 0.191 & 0.413 & 0.238 \\
MLP                 & 0.173 & 0.397 & 0.194 & 0.434 & 0.234 \\
\midrule
\method{}         & \textcolor{red}{\textbf{0.157}} & \textcolor{red}{\textbf{0.342}} & \textcolor{red}{\textbf{0.184}} & \textcolor{red}{\textbf{0.387}} & \textcolor{red}{\textbf{0.220}} \\
w/o prediction level & \textcolor{blue}{\underline{0.157}} & 0.373 & \textcolor{blue}{\underline{0.184}} & 0.392 & \textcolor{blue}{\underline{0.221}} \\
w/o feature level    & 0.161 & 0.349 & 0.188 & 0.393 & 0.224 \\
w/o multi-scale      & 0.162 & 0.377 & 0.187 & 0.393 & 0.224 \\
w/o multi-period     & 0.157 & \textcolor{blue}{\underline{0.342}} &0.184 & \textcolor{blue}{\underline{0.391}} & 0.221 \\
w/o sup              & 0.165 & 0.344 & 0.192 & 0.506 & 0.225 \\
\bottomrule
\end{tabular}
}
\vskip -1em
\end{table}

\textbf{Ablation study.} As our proposed \method{} incorporates two distillation strategies, \textit{multi-scale distillation} and \textit{multi-period distillation}, we assess their effectiveness by removing the corresponding losses, $\mathcal{L}_{\text{scale}}$ and $\mathcal{L}_{\text{period}}$, from \method{}. Additionally, we evaluate the impact of \textit{prediction-level} and \textit{feature-level distillation} by removing $\mathcal{L}^{\mathbf{H}}$ and $\mathcal{L}^{\mathbf{Y}}$, respectively. Furthermore, we test the model by removing the supervised loss $\mathcal{L}_{\text{sup}}$, using only the distillation losses as the overall loss for \method{}. 

Table~\ref{tab:ablation} presents the results of these ablations, compared with the full \method{}, a stand-alone MLP, and the teacher models. We draw the following observations: \textbf{First}, each loss term at both levels, when used individually, already outperforms the stand-alone MLP. \textbf{Second}, when combined, the losses complement each other, enabling \method{} to achieve the best performance, surpassing the teacher model. \textbf{Third}, notably, \method{} maintains superior performance over both the MLP and the teacher even without the supervised loss $\mathcal{L}_{\text{sup}}$. This can be attributed to two potential reasons: (1) ground truth may contain noise, making it more challenging to fit, whereas teacher provides simpler and more learnable knowledge; and (2) \textit{multi-scale} and \textit{multi-period} distillation processes effectively transfer complementary knowledge from teacher to MLP.

\begin{figure}[t]
    \centering    \includegraphics[width=0.45\textwidth]{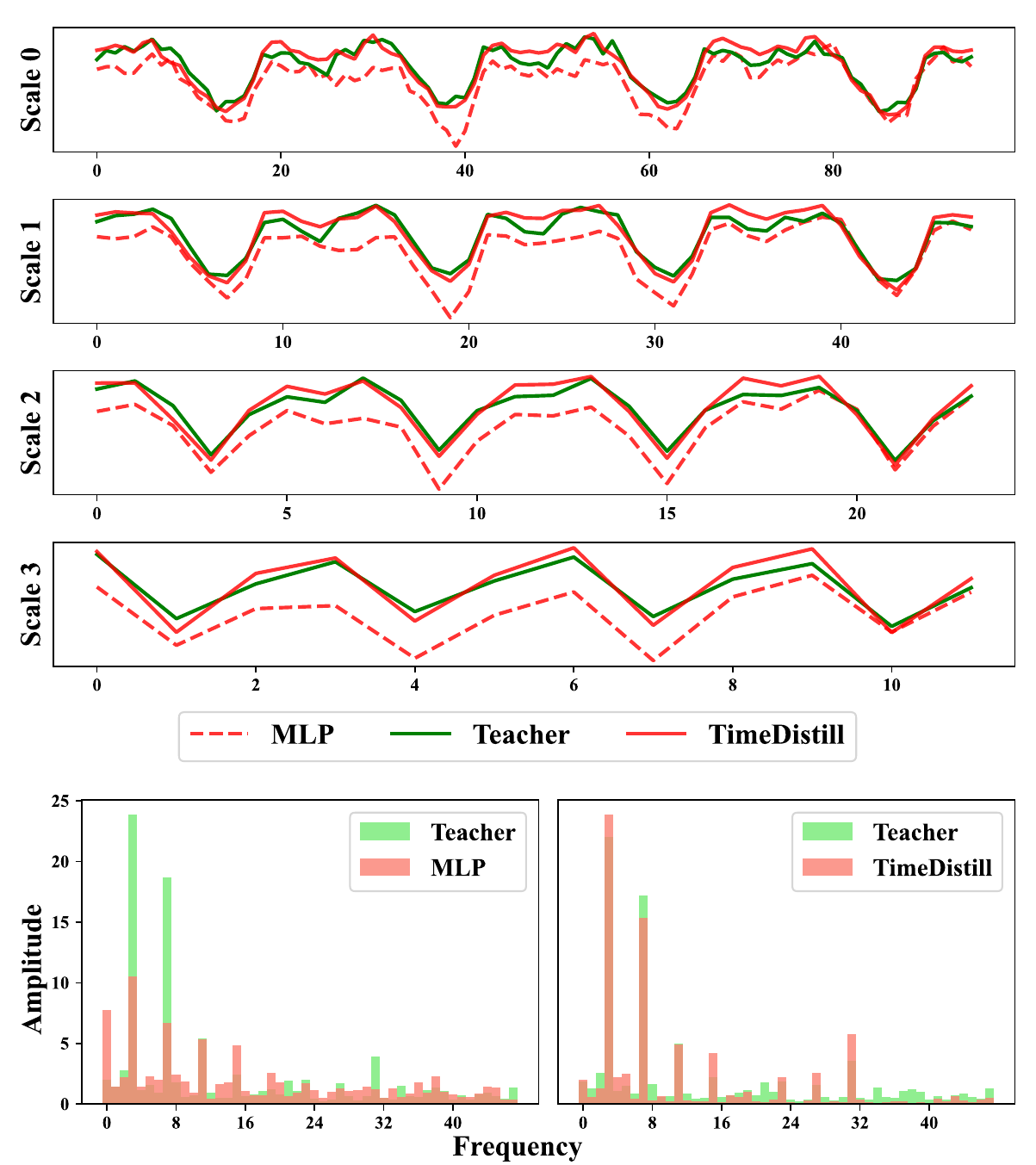}
    \vskip -1em
    \caption{Prediction comparison across temporal scales and spectrograms before and after distillation on ETTh1. MSE for MLP, Teacher (ModernTCN), and \method{} are 0.790, 0.365, and 0.366, showing \method{} bridges temporal and frequency domain gaps via multi-scale and multi-period distillation. More cases in Appendix~\ref{app:cases}.}
    \label{fig:vis_multiscale_multiperiod_ETTh1}
\end{figure}



\begin{table}[t]
\centering
\caption{Win ratio (\%) of MLP and \method{} vs. ModernTCN under input-720-predict-96 setting. \( U_{M} \) and \( U_{T} \) denote samples where MLP and \method{} outperform the teacher. \textit{Win Keep}, \( \frac{|U_{M} \cap U_{T}|}{|U_{M}|} \), shows \method{}’s retention of MLP’s wins.}
\label{tab:ratio_change}
\vskip -0.5em
\resizebox{0.4\textwidth}{!}{
\begin{tabular}{l c c c}
\toprule
Dataset & MLP & \method{} & \textit{Win Keep} \\
\midrule
ECL            & 35.64\% & 59.69\% & 95.44\% \\
ETTh1          & 28.58\% & 58.35\% & 79.27\% \\
ETTh2          & 25.57\% & 57.63\% & 69.10\% \\
ETTm1          & 42.82\% & 63.12\% & 84.61\% \\
ETTm2          & 42.84\% & 57.99\% & 80.14\% \\
Solar          & 59.40\% & 61.02\% & 91.77\% \\
Traffic        & 81.19\% & 91.24\% & 99.46\% \\
Weather        & 49.74\% & 56.43\% & 76.61\% \\
\bottomrule
\end{tabular}
}
\vskip -1em
\end{table}

\textbf{Does \method{} Truly Enhance MLP's Learning from the Teacher?} 
As shown in Table~\ref{tab:ratio_change}, \method{} improves the win ratio of MLP over the teacher by an average of 17.46\% across eight datasets. To determine whether this improvement is due to \method{} succeeding on samples MLP previously failed, we present the \textit{Win Keep} results. A higher \textit{Win Keep} indicates that \method{}'s improvements come from previously failed samples. \textit{Win Keep} remains consistently high across datasets (above 76.61\%, with an average of 84.55\%), indicating that \method{} not only retains MLP's success on samples where it already outperformed the teacher but also allows MLP to outperform the teacher on many samples it previously struggled with. This demonstrates that \method{} effectively transfers knowledge from the teacher to student MLP, enabling MLP to perform better on challenging samples while maintaining its existing strengths.

\textbf{Does \method{} Effectively Bridge the Gap Between Student and Teacher?} 
We present a case visualization from the ETTh1 dataset in Figure~\ref{fig:vis_multiscale_multiperiod_ETTh1}. The figure shows that \method{} reduces the difference between the teacher model (ModernTCN) and the student model (MLP) at multiple scales, and makes the fine-level series prediction more precise by transferring trend patterns at coarser scales from the teacher. In addition, for the same case, we visualize the spectrogram of the prediction and find that \method{} also helps reduce the difference between the teacher and MLP in the frequency domain, enabling the MLP to learn more structured multi-period patterns. 

\begin{figure}[ht!]
    \centering
    \includegraphics[width=0.45\textwidth]{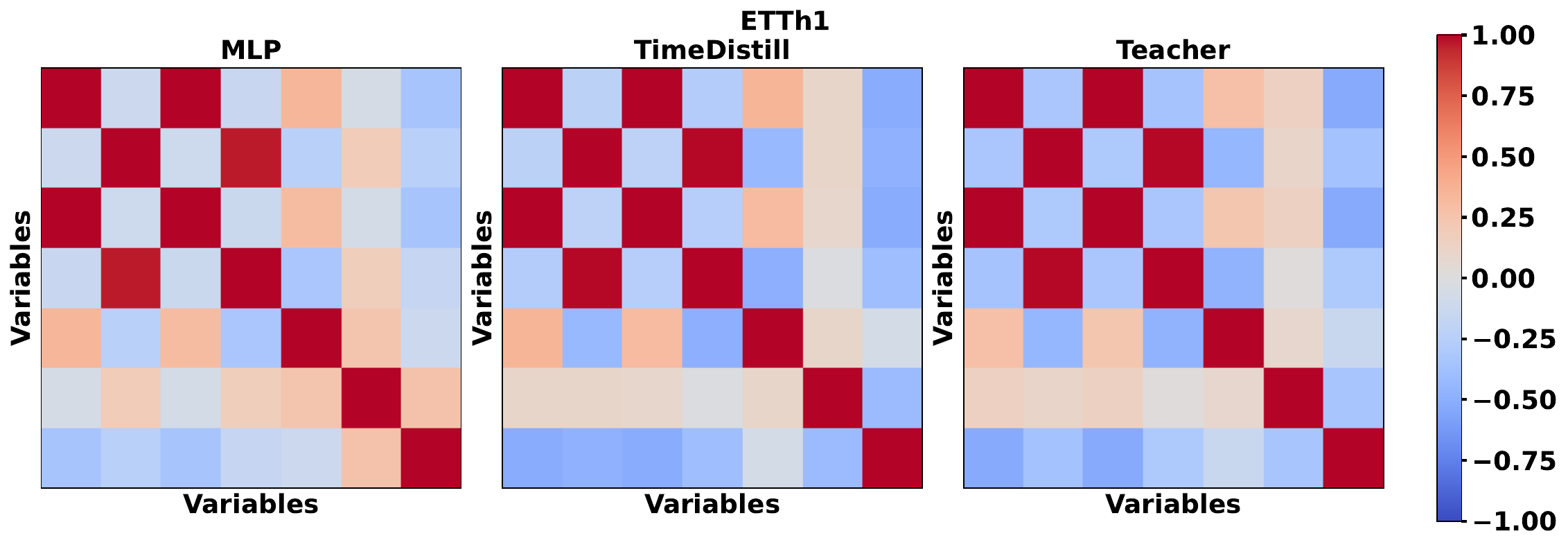} \\[1ex]
    \includegraphics[width=0.45\textwidth]{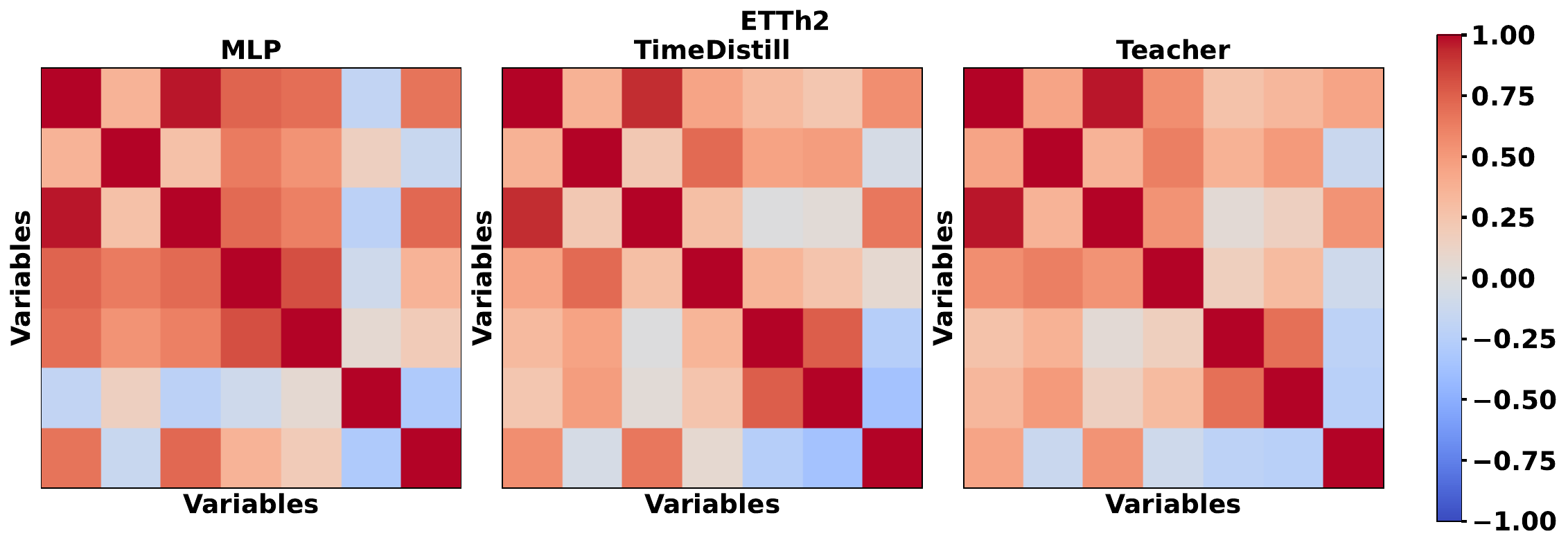}
    \vspace{-1em}
    \caption{Multi-Variate Person Correlation Visualization on ETTh1 and ETTh2 datasets.}
    \label{fig:correlation_heatmaps}
\end{figure}

\textbf{Does \method{} Learn Multi-Variate Correlation via Distillation?}
\label{app:multivariate}
Some teacher models, such as iTransformer and ModerTCN, are designed in a channel-dependent (CD) manner. While this approach can utilize multi-variate correlation effectively, it incurs a higher computational burden. In contrast, advanced methods like PatchTST, TimeMixer, and our method \method{} are employed in a channel-independent (CI) manner, which offers better computational efficiency. This leads to an intriguing question: \textit{Can \method{} enable MLP to learn multi-variate correlation from the teacher without an explicit distillation design?} To explore this, we present a preliminary study using Figure~\ref{fig:correlation_heatmaps}, which illustrates two cases from the ETTh1 and ETTh2 datasets, each containing seven variables. Our observations reveal that \method{}'s multi-variate correlation more closely resembles that of the teacher compared to the MLP model before distillation. This suggests that \method{} facilitates learning underlying structures among multi-variates through distillation even though the student MLP is a CI model. These findings highlight a potential benefit of distillation that merits further exploration. We leave this for further work.

\section{Conclusion and Future work}

We propose \method{}, a cross-architecture \textit{KD} framework enabling lightweight MLP to surpass complex teachers. By distilling multi-scale and multi-period patterns, \method{} efficiently transfers temporal and frequency-domain knowledge to student MLP. Theoretical interpretations and comprehensive experiments confirm its versatility and effectiveness. Future work includes distilling from advanced time series models, e.g. time series foundation models, and incorporating multivariate patterns.





\bibliographystyle{ACM-Reference-Format}
\balance
\bibliography{kdd}

\appendix
\clearpage

\section{Theoretical Understanding from Data Augmentation Perspective}
\label{app:theory}
In this section, we propose a theoretical viewpoint that \method{} unifies the idea of augmenting training data with \emph{blended} data that combines teacher outputs and ground truth.
This concept is closely related to \emph{mixup}~\cite{mixup}, a well-known data augmentation strategy.
\emph{Mixup} creates new training samples by \emph{mixing} two examples' inputs and corresponding targets according to a random interpolation coefficient. 
Analogously, in multi-scale or multi-period distillation, we create augmented samples that blend teacher predictions and ground truth signals in a manner that provides richer supervision signals.  Concretely, in a standard \emph{mixup} procedure, one obtains augmented samples \((\tilde{x}, \tilde{y})\) from two distinct training samples \((x_i, y_i)\) and \((x_j, y_j)\) by drawing a mixing ratio \(\lambda\) and computing:
\[
\tilde{x} = \lambda x_i + (1-\lambda) x_j,
\quad
\tilde{y} = \lambda y_i + (1-\lambda) y_j.
\]
In our distillation setting, instead of mixing two different input samples, we can think of \(\mathbf{\hat{Y}}_t\) (the teacher predictions) and \(\mathbf{Y}\) (the ground truth) as two sources to be mixed. Hence, these augmented targets become:
\[
\mathbf{\tilde{Y}} = \lambda \, \mathbf{Y} \;+\; (1-\lambda)\mathbf{\hat{Y}}_t,
\]
while keeping the input as it is. 
By treating \(\mathbf{\hat{Y}}_t\) as another ``view'' of the target, we are effectively performing a \emph{mixup-style} augmentation in the target space. 
This approach can be especially valuable in time series forecasting, where incorporating teacher signals from multiple horizons or multiple resolutions may help capture nuanced temporal structures. \emph{Next, we present two theorems that formalize multi-scale and multi-period loss as upper bound of matching mixup augmented prediction and student prediction.}

\begin{theorem} 
Let $(x, y)$ denote original input data pairs and $(x, y^t)$ represent corresponding teacher data pairs. Consider a data augmentation function $\mathcal{A}(\cdot)$ applied to $(x, y)$, generating augmented samples $(x', y')$. Define the training loss on these augmented samples as $\mathcal{L}_{aug} = \sum_{(x',y') \in \mathcal{A}(x,y)} |f_s(x') - y'|^2$. Then, the following inequality holds: 
\begin{equation}
   \mathcal{L}_{sup} + \eta \mathcal{L}_{scale} \geq \mathcal{L}_{aug}
\end{equation}
when $\mathcal{A}(\cdot)$ is instantiated as a mixup function~\cite{mixup} that interpolates between the original input data $(x,y)$ and teacher data $(x,y^t)$ with a mixing coefficient $\lambda=\frac{1}{1+\eta}$, i.e. $y' = \lambda y + (1-\lambda) y^t$.
\end{theorem}


\begin{proof}
Let \(\gamma_0 = 1\) and \(\gamma_m = \frac{\eta}{M+1}\) for \(m=1,\dots,M\). Observe that these weights satisfy \(\gamma_0 + \sum_{m=1}^{M} \gamma_m = 1 + \eta\). Using these weights, we have:
\[
\begin{aligned}
&\quad \mathcal{L}_{sup} + \eta \mathcal{L}_{scale} \\
&=\bigl(\hat{Y}_s - Y\bigr)^2
+ \eta\sum_{m=0}^{M}\frac{\bigl(\hat{Y}_s - \hat{Y}_t^{(m)}\bigr)^2}{M+1}\\
&=
\sum_{m=0}^{M}\gamma_m \bigl(\hat{Y}_s - x_m\bigr)^2,
\end{aligned}
\]
where \(x_0 = Y\) and \(x_{m+1} = \hat{Y}_t^{(m)}\) for \(m=0,\dots,M\). By Jensen's inequality~\cite{jensen1906fonctions}, for a convex function \(f\), we have:
\[
\sum_{m=0}^{M} \gamma_m f(x_m) \geq f\Bigl(
\frac{\sum_{m=0}^{M} \gamma_m x_m}{\sum_{m=0}^{M} \gamma_m}
\Bigr).
\]
We use Jensen's inequality, which applies due to the convexity of the squared loss function \(f(x) = (\hat{Y}_s - x)^2\). We obtain:
\[
\sum_{m=0}^{M} \gamma_m \bigl(\hat{Y}_s - x_m\bigr)^2
\geq 
\Bigl(\hat{Y}_s - \frac{\sum_{m=0}^{M} \gamma_m x_m}{\sum_{m=0}^{M} \gamma_m}\Bigr)^2.
\]
Now, compute the weighted sum of \(x_m\) using the weights \(\gamma_m\) and we define $\lambda = \frac{1}{1+\eta}$:
\[
\begin{aligned}
&\frac{\sum_{m=0}^{M} \gamma_m x_m }{\sum_{m=0}^{M} \gamma_m}
= 
\frac{\gamma_0 x_0 + \sum_{m=1}^{M} \gamma_m x_m}{1+\eta}\\
&= \frac{1}{1+\eta}Y + \frac{\eta}{1+\eta} \sum_{m=0}^{M} \frac{\hat{Y}_t^{(m)}}{M+1}\\
&= \lambda Y + (1-\lambda) \sum_{m=0}^{M} \frac{\hat{Y}_t^{(m)}}{M+1}.
\end{aligned}
\]
Substituting this result back, we conclude that:
\[
\begin{aligned}
&\quad \mathcal{L}_{sup} + \eta \mathcal{L}_{scale} \\
&=\bigl(\hat{Y}_s - Y\bigr)^2
+ \eta\sum_{m=0}^{M}\frac{\bigl(\hat{Y}_s - \hat{Y}_t^{(m)}\bigr)^2}{M+1}\\
&\geq \bigl[\hat{Y}_s - [\lambda Y + (1-\lambda) \sum_{m=0}^{M} \frac{\hat{Y}_t^{(m)}}{M+1}]\bigr]^2\\
&= \mathcal{L}_{aug}
\end{aligned}
\]
This completes the proof.
\end{proof}


\begin{theorem} 
Define the training loss on the augmented samples using KL divergence as $\mathcal{L}_{aug} = \sum_{(x',y') \in \mathcal{A}(x,y)} \text{KL}\big(y', \mathcal{X}(f_s(x'))\big)$, where $\mathcal{X}(\cdot) = \text{Softmax}(\text{FFT}(\cdot))$. Then, the following inequality holds: 
\begin{equation}
   \mathcal{L}_{sup} + \eta\mathcal{L}_{period} \geq \mathcal{L}_{aug}
\end{equation}
where $\mathcal{A}(\cdot)$ is instantiated as a mixup function that interpolates between the period distribution of original input data $(x,\mathcal{X}(y))$ and teacher data $(x,\mathcal{X}(y^t))$ with a mixing coefficient $\lambda=\eta$, i.e. $y' =  \mathcal{X}(y) + \lambda \mathcal{X}(y^t)$.
\end{theorem}

\begin{proof} 
We denote $\mathbf{Q}_y = \mathcal{X}(y)$, $\mathbf{Q}_s = \mathcal{X}(f_s(x))$, $\mathbf{Q}_t = \mathcal{X}(y_t)$, $\hat{Y}_s = f_s(x)$ According to Parseval's theorem, we rewrite $\mathcal{L}_{sup}$ as following:
\[
\begin{aligned}
&\quad \mathcal{L}_{sup} + \eta \mathcal{L}_{period}\\
&= \,\bigl(\hat{Y}_s - Y\bigr)^2
\;+\;
\eta\,\mathrm{KL}\bigl(\mathbf{Q}_t, \mathbf{Q}_s\bigr) \\
&= \,\bigl(FFT(\hat{Y}_s) - FFT(Y)\bigr)^2
\;+\;
\eta\,\mathrm{KL}\bigl(\mathbf{Q}_t, \mathbf{Q}_s\bigr) \\
\end{aligned}
\]
When the temperature $\tau$ goes to $\infty$, minimizing KL divergence between two distributions is equivalent to minimizing the MSE error~\cite{kim2021comparing}. Thus, we have
\[
\begin{aligned}
& \quad \,\bigl(FFT(\hat{Y}_s) - FFT(Y)\bigr)^2
\;+\;
\eta\,\mathrm{KL}\bigl(\mathbf{Q}_t, \mathbf{Q}_s\bigr)\\
&= \lim_{\tau \rightarrow \infty}\,\mathrm{KL}\bigl(\mathcal{X}(Y), \mathcal{X}(\hat{Y}_s)\bigr)\;+\;
\eta\,\mathrm{KL}\bigl(\mathbf{Q}_t, \mathbf{Q}_s\bigr) \\
&= \lim_{\tau \rightarrow \infty}\,\mathrm{KL}\bigl(\mathbf{Q}_y, \mathbf{Q}_s\bigr)
\;+\;
\eta\,\mathrm{KL}\bigl(\mathbf{Q}_t, \mathbf{Q}_s\bigr) 
\end{aligned}
\]
The KL divergence of $P$ from $Q$ is defined as 
\[\mathrm{KL}[P||Q]=\sum_{x\in\mathcal{X}}p(x)\cdot\log\frac{p(x)}{q(x)}\]
and the log sum inequality~\cite{thomas2006elements} states that
\[\sum_{i=1}^na_i\log\frac{a_i}{b_i}\geq\left(\sum_{i=1}^na_i\right)\log\frac{\sum_{i=1}^na_i}{\sum_{i=1}^nb_i},\]
where $a_1,\ldots,a_n$ and $b_1,\ldots,b_n$ are non-negative real numbers. For clarity, we omit $\lim_{\tau \rightarrow \infty}$ in the following derivation. Continuing with the main derivation, we have:
\[
\begin{aligned}
&\quad \,\mathrm{KL}\bigl(\mathbf{Q}_y, \mathbf{Q}_s\bigr)
\;+\;
\eta\,\mathrm{KL}\bigl(\mathbf{Q}_t, \mathbf{Q}_s\bigr) \\
&=\sum_{x\in\mathcal{X}}\mathbf{Q}_y(x)\cdot\log\frac{\mathbf{Q}_y(x)}{\mathbf{Q}_s(x)}\\
&\quad\quad\quad\quad\quad\quad+\eta\sum_{x\in\mathcal{X}}\mathbf{Q}_t(x)\cdot\log\frac{\mathbf{Q}_t(x)}{\mathbf{Q}_s(x)} \\
&=\sum_{x\in\mathcal{X}}[ \mathbf{Q}_y(x)\cdot\log\frac{ \mathbf{Q}_y(x)}{\mathbf{Q}_s(x)}\\
&\quad\quad\quad\quad\quad\quad+\eta\mathbf{Q}_t(x)\cdot\log\frac{\eta\mathbf{Q}_t(x)}{\eta\mathbf{Q}_s(x)}] \\
&\ge\sum_{x\in\mathcal{X}}[[ \mathbf{Q}_y(x)
+\eta\mathbf{Q}_t(x)]\\
&\quad\quad\quad\quad\quad\quad\cdot\log\frac{ \mathbf{Q}_y(x)+\eta\mathbf{Q}_t(x)}{ \mathbf{Q}_s(x)+\eta\mathbf{Q}_s(x)}] \\
&=
\mathrm{KL}\Bigl(\,\mathbf{Q}_y + \eta\,\mathbf{Q}_t
\;,\;\mathbf{Q}_s\Bigr)\\
\end{aligned}
\]
To align with the mixup formulation defined in our theorem. We denote $\lambda = \eta$, then we have
\[
\begin{aligned}
&\quad \mathrm{KL}\Bigl(\,\mathbf{Q}_y + \eta\,\mathbf{Q}_t
\;,\;\mathbf{Q}_s\Bigr)\\
&=\mathrm{KL}\Bigl(\,\mathbf{Q}_y + \lambda\,\mathbf{Q}_t
\;,\;\mathbf{Q}_s\Bigr)\\
&=\mathrm{KL}\Bigl(\,\mathcal{X}(Y) + \lambda\,\mathcal{X}(\hat{Y}_t)
\;,\;\mathcal{X}(\hat{Y}_s)\Bigr)\\
&=\mathcal{L}_{aug}\\
\end{aligned}
\]

This completes the proof.
\end{proof}

These two theorems show that the disllation process provides additional augmented samples that smooth the ground truth and the teacher prediction. This is analogous to label smoothing in classification, where labels are adjusted to be less extreme or certain.
This data augmentation perspective brings several notable benefits for time series forecasting:
\textbf{(1) Enhanced Generalization:} By blending teacher predictions (e.g., from a larger or more accurate model, or from multiple horizons/scales) with the ground truth, the student model is exposed to a richer supervision signal. This can mitigate overfitting, especially in scenarios with limited or noisy training data and the teacher predictions might encapsulate trends or patterns not immediately apparent from the ground truth alone, especially in cases of complex or high-dimensional regression tasks. 
\textbf{(2) Explicit Integration of Patterns:} In time series, important temporal structures often manifest at different resolutions or periods.  Incorporating teacher predictions at multiple scales can help the student model learn both short-term fluctuations and long-term trends, which might be overlooked when relying solely on the ground truth. \textbf{(3) Stabilized Training Dynamics:} Blending ground truth with teacher outputs naturally “softens” the targets, making the learning process less sensitive to noise and variance in the dataset.  This smoothness can also lead to gentler gradients, reducing abrupt directional changes during optimization.  Consequently, training is more stable and less prone to exploding or vanishing updates, facilitating better convergence.

\begin{table*}[ht!]
\centering
\caption{Dataset characteristics for time series long-term forecasting task.}
\resizebox{0.97\textwidth}{!}{%
\begin{tabular}{c|c|c|c|c|c|c}
\toprule
 Dataset & Dim & Series Length & Dataset Size     & Frequency & Forecastability* & Information \\ \midrule
  ECL      & 321          & \{96, 192, 336, 720\}  & (18317, 2633, 5261)      & Hourly             & 0.77                       & Electricity          \\ \midrule
  ETTh1            & 7            & \{96, 192, 336, 720\}  & (8545, 2881, 2881)        & 15min              & 0.38                       & Temperature          \\ \midrule
                       ETTh2            & 7            & \{96, 192, 336, 720\}  & (8545, 2881, 2881)        & 15min              & 0.45                       & Temperature          \\ \midrule
 ETTm1            & 7            & \{96, 192, 336, 720\}  & (34465, 11521, 11521)     & 15min              & 0.46                       & Temperature          \\ \midrule
                       ETTm2            & 7            & \{96, 192, 336, 720\}  & (34465, 11521, 11521)     & 15min              & 0.55                       & Temperature          \\ \midrule
                       
                       Solar     & 137          & \{96, 192, 336, 720\}  & (36601, 5161, 10417)     & 10min              & 0.33                       & Electricity
                       \\ \midrule
                       Traffic          & 862          & \{96, 192, 336, 720\}  & (12185, 1757, 3509)      & Hourly             & 0.68                       & Transportation       \\ \midrule
                       Weather          & 21           & \{96, 192, 336, 720\}  & (36792, 5271, 10540)     & 10min              & 0.75                       & Weather                        \\ \bottomrule
\end{tabular}
}
*The forecastability is borrowed from TimeMixer~\cite{timemixer}. A larger value indicates better predictability.
\label{tab:dataset}
\end{table*}

\section{Implementation Details}
\label{app:implementation_details}
\textbf{Datasets Details.}
We conduct experiments to evaluate the performance and efficiency of \method{} on eight widely used benchmarks: Electricity (ECL), the ETT datasets (ETTh1, ETTh2, ETTm1, ETTm2), Solar, Traffic, and Weather, following the pipeline in~\cite{itransformer, timemixer, moderntcn}. Detailed descriptions of these datasets are provided in Table~\ref{tab:dataset}.

\textbf{Metric Details.} We use Mean Square Error (MSE) and Mean Absolute Error (MAE) as our evaluation metrics, following~\cite{itransformer, moderntcn, timemixer, micn, timesnet, autoformer, informer, dlinear}:
\begin{equation}
    \text{MSE} = \frac{1}{S \times C}\sum_{i=1}^S\sum_{j=1}^C(Y_{ij}-\hat{Y}_{ij})^2,
\end{equation}
\begin{equation}
    \text{MAE} = \frac{1}{S \times C}\sum_{i=1}^S\sum_{j=1}^C|Y_{ij}-\hat{Y}_{ij}|.
\end{equation}
Here, $Y \in \mathbb{R}^{S \times C}$ represents the ground truth, and $\hat{Y} \in \mathbb{R}^{S \times C}$ represents the predictions. $S$ denotes the future prediction length, $C$ is the number of channels, and $Y_{ij}$ indicates the value at the $i$-th future time point for the $j$-th channel.

\begin{table}[h]
\centering
\caption{Experiment configuration of \method{}.}
\resizebox{0.45\textwidth}{!}{
\begin{tabular}{c c c c c c}
\hline
\multirow{2}{*}{Dataset} & \multicolumn{2}{c}{Model Hyper-parameter} & \multicolumn{3}{c}{Training Process} \\
 & $D$ & Norm & Epoch & $\alpha$ & $\beta$ \\
\hline
ECL     & 512   & non-stationary & 20 & 0.1 & 0.5 \\
ETTh1   & 512   & non-stationary & 20 & 2   & 2   \\
ETTh2   & 512   & non-stationary & 20 & 2   & 0.5 \\
ETTm1   & 512   & non-stationary & 20 & 2   & 0.1 \\
ETTm2   & 512   & non-stationary & 20 & 2   & 0.5 \\
Solar   & 512   & non-stationary & 20 & 0.1 & 2   \\
Traffic & 1024  & revin          & 10 & 0.1 & 0.1 \\
Weather & 512   & non-stationary & 20 & 0.5 & 2   \\
\hline
\end{tabular}
}
\label{tab:hyperparams}
\end{table}

\textbf{Experiment Details.} All experiments are implemented in PyTorch~\cite{paszke2019pytorch} and conducted on a cluster equipped with NVIDIA A100, V100, and K80 GPUs. The teacher models are trained using their default configurations as reported in their respective papers. For the student MLP model, we employ a combination of a decomposition scheme~\cite{dlinear, autoformer, fedformer} and a two-layer MLP with hidden dimension $D$ as the architecture.  When using \method{} for distillation, the teacher model is frozen, and only the student MLP is trained. The initial learning rate is set to 0.01, and the ADAM optimizer~\cite{kingma2014adam} is used with MSE loss for the training of the student model. We apply early stopping with a patience value of 5 epochs. The batch size is set to 32. The temperature for multi-period distillation is set to $\tau=0.5$. The number of scales is set to $M=3$. We use two types of normalization methods, i.e. non-stationary~\cite{non-stationary} and revin~\cite{revin}. We perform a hyperparameter search for $\alpha$ and $\beta$ within the range \{0.1, 0.5, 1, 2\}. We conduct hyperparameter sensitivity analysis in Appendix~\ref{app:hyperparameter_sensitivity}. For consistency, \method{} defaults to using ModernTCN as the teacher and fixes the look-back window length to 720. Additional experiments with other teacher models and different look-back window lengths are presented in Section~\ref{sec:versatility}. Additional detailed model configuration information is presented in Table~\ref{tab:hyperparams}. We also report the standard deviation of \method{} performance under five runs with different random seeds in Table~\ref{tab:standard_deviation}, which exhibits that the performance of \method{} is stable.

\textbf{Channel Independent MLP.} We denote the model as \( f \). The channel-independent strategy predicts each channel independently, defined as:$\hat{Y} = [f(x^1), \dots, f(x^C)]$, where $\hat{Y} = [\hat{y}^1, \dots, \hat{y}^C]$ denotes the model prediction. In contrast, the channel-dependent strategy considers the combination of all channels simultaneously, defined as: $\hat{Y} = f(x^1, \dots, x^C)$. In this paper, for MLP, we adopt the channel-independent strategy because it eliminates the need to model inter-channel relationships~\cite{han2024capacity}, which reduces model complexity and results in a more lightweight model, which is especially advantageous when dealing with a large number of variables. 

\begin{table*}[ht]
\centering
\caption{Robustness of \method{} performance. The results are obtained from five random seeds, averaged over four prediction lengths.}
\resizebox{0.99\textwidth}{!}{
\begin{tabular}{lcccccccc}
\toprule
\textbf{Dataset} & \textbf{ECL} & \textbf{ETTh1} & \textbf{ETTh2} & \textbf{ETTm1} & \textbf{ETTm2} & \textbf{Solar} & \textbf{Traffic} & \textbf{Weather} \\
\midrule
\textbf{MSE} & 0.157$\pm$0.002 & 0.429$\pm$0.004 & 0.345$\pm$0.002 & 0.348$\pm$0.002 & 0.244$\pm$0.001 & 0.184$\pm$0.000 & 0.387$\pm$0.001 & 0.220$\pm$0.000 \\
\textbf{MAE} & 0.254$\pm$0.003 & 0.441$\pm$0.003 & 0.395$\pm$0.002 & 0.379$\pm$0.001 & 0.311$\pm$0.001 & 0.241$\pm$0.000 & 0.271$\pm$0.001 & 0.269$\pm$0.001 \\
\bottomrule
\end{tabular}
}
\label{tab:standard_deviation}
\end{table*}

\section{Ensuring Fair Comparison}
\label{app:fair_comparison}

To ensure a fair comparison, we have taken the following measures.  
\textbf{First}, we treat normalization (either non-stationary normalization or REVIN) as a hyperparameter and follow the original implementation of each baseline to obtain its optimal performance.  
\textbf{Second}, we reproduce all baselines using the widely adopted GitHub repository “Time-Series-Library” and implement \method{} within the same framework to provide a unified and fair evaluation environment.  
\textbf{Third}, we set \textit{drop\_last=False} so that all methods are evaluated on the exact same test set.  
\textbf{Finally}, the only difference between the baseline MLP and \method{} is the inclusion of the distillation loss, which confirms that the observed performance improvements come solely from the distillation process.

\section{Further Illustration on Motivation}
\label{app:further_motivation}

\textbf{\method{} does not rely on the assumption that the student model (MLP) must be weaker than the teacher model.} While MLPs may sometimes match or even outperform teacher models, our method does not assume that MLPs are generally weaker. Instead, knowledge distillation is effective for two main reasons:

\textit{Empirical complementary strengths:} As outlined in Section~\ref{sec:preliminaries}, our analysis shows that MLPs benefit from distilling complementary knowledge from teacher models. As shown in Figure~\ref{fig:result_cv}, the win ratio is high, which indicates that teacher and student models have different strengths on different subsets of data. Importantly, the teacher model does not need to outperform the MLP; it is sufficient for the teacher to provide additional or diverse information that the MLP can exploit to improve its own performance.

\textit{Theoretical justification:} We provide a theoretical interpretation of \method{} in Section~\ref{sec:theoretical_analysis}, where we view it through a data augmentation perspective. This mixup-based data augmentation interpretation explains the generalization benefits of distillation, independent of the performance gap between teacher and student models.

\section{Implementation Details for Preliminary Study.} 
\label{app:implementation_preliminary}
For the multi-scale pattern, we downsample the predicted time series $\mathbf{\hat{Y}}$ using a 1D convolution layer with a temporal stride of 2 to generate time series at $M$ scales. The downsampling is defined as: 
\begin{equation}
    \mathbf{\hat{Y}}^m = \mathrm{Conv}(\mathbf{\hat{Y}}^{m-1}, \mathrm{stride}=2),
\end{equation}
where \(m \in \{1, \cdots, M\}\) and $\mathbf{\hat{Y}}^0 = \mathbf{\hat{Y}}$. We set the prediction length $S$ to 96 and display only the first 48 time steps due to space constraints, with $M$ set to 3. We compare the MLP's multi-scale prediction performance against two Transformer-based teacher models (i.e., iTransformer, PatchTST) and one CNN-based teacher model (i.e., ModernTCN), using the ground truth as a reference.

For multi-period distillation, we compute the spectrogram of $\mathbf{\hat{Y}}$ via Fast Fourier Transform (FFT), which provides $S/2$ frequencies and their corresponding amplitudes. We remove the DC component in the spectrogram. Each frequency corresponds to a period, calculated as $S / \text{frequency}$, with amplitude indicating the period's significance. Thus, the spectrogram can effectively show the multi-period pattern of a time series. The prediction length $S$ is set to 96. Similar to the multi-scale, we compare the MLP's multi-period prediction performance against iTransformer, PatchTST, and ModernTCN, using the ground truth as a reference.

\begin{table*}[h]
\centering
\caption{Performance promotion obtained by our \textbf{\method{}} framework with different teacher models. $\Delta_{MLP}$ ($\Delta_{Teacher}$) represents the performance promotion between \textbf{MLP+\method{}} and a trained MLP (Teacher). We report the average performance from all prediction lengths.}
\resizebox{\textwidth}{!}{%
\begin{tabular}{cc cc|cc|cc|cc|cc|cc|cc|cc}
\toprule
\multicolumn{2}{c}{\multirow{2}{*}{Teacher Models}} & \multicolumn{2}{c}{iTranformer} & \multicolumn{2}{c}{ModernTCN} & \multicolumn{2}{c}{TimeMixer} & \multicolumn{2}{c}{PatchTST}  & \multicolumn{2}{c}{MICN} & \multicolumn{2}{c}{FEDformer} & \multicolumn{2}{c}{TimesNet} & \multicolumn{2}{c}{Autoformer}  \\ 
& & \multicolumn{2}{c}{(\citeyear{itransformer})} & \multicolumn{2}{c}{(\citeyear{moderntcn})} & \multicolumn{2}{c}{(\citeyear{timemixer})} & \multicolumn{2}{c}{(\citeyear{patchtst})} & \multicolumn{2}{c}{(\citeyear{micn})} & \multicolumn{2}{c}{(\citeyear{fedformer})} & \multicolumn{2}{c}{(\citeyear{timesnet})} &  \multicolumn{2}{c}{(\citeyear{autoformer})} \\
\multicolumn{2}{c}{Metric} & \multicolumn{1}{|c}{MSE} & MAE & \multicolumn{1}{|c}{MSE} & MAE & \multicolumn{1}{|c}{MSE} & MAE & \multicolumn{1}{|c}{MSE} & MAE & \multicolumn{1}{|c}{MSE} & MAE & \multicolumn{1}{|c}{MSE} & MAE & \multicolumn{1}{|c}{MSE} & MAE & \multicolumn{1}{|c}{MSE} & MAE \\
\midrule
\multirow{5}{*}{ECL} & \multicolumn{1}{|c}{Teacher} &\multicolumn{1}{|c}{0.163}	&0.259	&0.167	&0.262	&0.165	&0.259 &0.165	&0.266	&0.181	&0.293	&0.274	&0.376	&0.250	&0.347	&0.238	&0.347 \\ 
&\multicolumn{1}{|c}{MLP} &\multicolumn{1}{|c}{0.173}	&0.276	&0.173		&0.276	&0.173		&0.276	&0.173		&0.276 &0.173		&0.276	&0.173		&0.276	&0.173		&0.276	&0.173		&0.276 \\
&\multicolumn{1}{|c}{\textbf{+\method{}}} &\multicolumn{1}{|c}{\textbf{0.157}}	&\textbf{0.254}	&\textbf{0.157}	&\textbf{0.254}	&\textbf{0.159}	&\textbf{0.256} &\textbf{0.159}	&\textbf{0.256}	&\textbf{0.158}	&\textbf{0.256}	&\textbf{0.157}	&\textbf{0.255}	&\textbf{0.162}	&\textbf{0.261}	&\textbf{0.158}	&\textbf{0.256} \\\cmidrule(lr){2-18} 
&\multicolumn{1}{|c}{$\Delta_{Teacher}$} &\multicolumn{1}{|c}{3.68\%}	&1.93\%	&5.61\%	&3.24\%	&3.31\%	&1.02\% & 3.64\%	&3.76\%	&12.78\%	&12.89\%	&42.70\%	&32.18\%	&35.08\%	&24.82\%	&33.61\%	&26.22\% 
\\
&\multicolumn{1}{|c}{$\Delta_{MLP}$} &\multicolumn{1}{|c}{9.27\%}	&7.93\%	&9.09\%	&8.06\%	&7.94\%	&7.18\%	&8.11\%	&7.21\%	&8.70\%	&7.34\%	&9.27\%	&7.57\%	&6.22\%	&5.39\%	&8.69\%	&7.21\%	 \\
\midrule
\multirow{5}{*}{ETTh1} & \multicolumn{1}{|c}{Teacher} &\multicolumn{1}{|c}{0.468}		&0.476	&0.469	&0.465	&0.459	&0.465 &0.498	&0.490	&0.739	&0.631	&0.527	&0.524	&0.507	&0.500	&0.731	&0.659 \\ 
&\multicolumn{1}{|c}{MLP} &\multicolumn{1}{|c}{0.502}	&0.489	&0.502	&0.489	&0.502	&0.489	&0.502	&0.489	&0.502	&0.489 &0.502	&0.489	&0.502	&0.489	&0.502	&0.489 \\
&\multicolumn{1}{|c}{\textbf{+\method{}}} &\multicolumn{1}{|c}{\textbf{0.428}}		&\textbf{0.445}	&\textbf{0.429}	&\textbf{0.441}	&\textbf{0.457}	&\textbf{0.464} &\textbf{0.443}	&\textbf{0.456}	&\textbf{0.496}	&\textbf{0.480}	&\textbf{0.438}	&\textbf{0.454}	&\textbf{0.441}	&\textbf{0.454}	&\textbf{0.475}	&\textbf{0.480} \\\cmidrule(lr){2-18} 
&\multicolumn{1}{|c}{$\Delta_{Teacher}$} &\multicolumn{1}{|c}{8.55\%}		&6.51\%	&8.42\%	&5.23\%	&0.42\%	&0.25\% &11.04\%	&6.94\%	&32.90\%	&23.97\%	&16.89\%	&13.36\%	&13.09\%	&9.22\%	&35.02\%	&27.16\%
\\
&\multicolumn{1}{|c}{$\Delta_{MLP}$} &\multicolumn{1}{|c}{14.74\%}		&9.00\%	&14.47\%	&9.85\%	&8.92\%	&5.13\% &11.75\%	&6.75\%	&1.20\%	&1.82\%	&12.75\%	&7.16\%	&12.14\%	&7.25\%	&5.38\%	&1.84\%	 \\
\midrule
\multirow{5}{*}{ETTh2} & \multicolumn{1}{|c}{Teacher} &\multicolumn{1}{|c}{0.398}	&0.426	&0.357	&0.403	&0.422	&0.444 &0.444	&0.443	&1.078	&0.736	&0.456	&0.485	&0.419	&0.446	&1.594	&0.940 \\ 
&\multicolumn{1}{|c}{MLP} &\multicolumn{1}{|c}{0.393}		&0.438	&0.393	&0.438	&0.393	&0.438	&0.393	&0.438 &0.393	&0.438	&0.393	&0.438	&0.393	&0.438	&0.393	&0.438 \\
&\multicolumn{1}{|c}{\textbf{+\method{}}} &\multicolumn{1}{|c}{\textbf{0.345}}		&\textbf{0.395}	&\textbf{0.345}	&\textbf{0.395}	&\textbf{0.357}	&\textbf{0.403} &\textbf{0.371}	&\textbf{0.416}	&\textbf{0.372}	&\textbf{0.418}	&\textbf{0.359}	&\textbf{0.408}	&\textbf{0.356}	&\textbf{0.404}	&\textbf{0.368}	&\textbf{0.419} \\\cmidrule(lr){2-18} 
&\multicolumn{1}{|c}{$\Delta_{Teacher}$} &\multicolumn{1}{|c}{13.32\%}	&7.28\%	&3.52\%	&2.03\%	&15.49\%	&9.17\% &16.44\%	&6.09\%	&65.50\%	&43.22\%	&21.27\%	&15.88\%	&14.90\%	&9.33\%	&76.91\%	&55.43\%
\\
&\multicolumn{1}{|c}{$\Delta_{MLP}$} &\multicolumn{1}{|c}{12.21\%}	&9.82\%	&12.27\%	&9.90\%	&9.17\%	&7.91\% &5.60\%	&5.02\%	&5.41\%	&4.60\%	&8.65\%	&6.85\%	&9.31\%	&7.69\%	&6.36\%	&4.34\%	 \\
\midrule
\multirow{5}{*}{ETTm1} & \multicolumn{1}{|c}{Teacher} &\multicolumn{1}{|c}{0.372}	&0.402	&0.390	&0.410	&0.367	&0.388 &0.383	&0.412	&0.439	&0.461	&0.423	&0.451	&0.398	&0.419	&0.570	&0.526 \\ 
&\multicolumn{1}{|c}{MLP} &\multicolumn{1}{|c}{0.391}		&0.413	&0.391	&0.413	&0.391	&0.413	&0.391	&0.413	&0.391	&0.413	&0.391	&0.413	&0.391	&0.413	&0.391	&0.413 \\
&\multicolumn{1}{|c}{\textbf{+\method{}}} &\multicolumn{1}{|c}{\textbf{0.354}}	&\textbf{0.390}	&\textbf{0.348}	&\textbf{0.379}	&\textbf{0.347}	&\textbf{0.383} &\textbf{0.358}	&\textbf{0.397}	&\textbf{0.375}	&\textbf{0.400}	&\textbf{0.354}	&\textbf{0.398}	&\textbf{0.356}	&\textbf{0.391}	&\textbf{0.376}	&\textbf{0.407} \\\cmidrule(lr){2-18} 
&\multicolumn{1}{|c}{$\Delta_{Teacher}$} &\multicolumn{1}{|c}{4.84\%}	&2.99\%	&10.94\%	&7.64\%	&5.34\%	&1.25\% &6.53\%	&3.64\%	&14.66\%	&13.11\%	&16.31\%	&11.75\%	&10.53\%	&6.64\%	&34.04\%	&22.62\%
\\
&\multicolumn{1}{|c}{$\Delta_{MLP}$} &\multicolumn{1}{|c}{9.42\%}		&5.65\%	&11.07\%	&8.43\%	&11.16\%	&7.34\%	&8.39\%	&3.96\%	&4.13\%	&3.17\%	&9.42\%	&3.72\%	&8.79\%	&5.40\%	&3.79\%	&1.54\%	 \\
\midrule
\multirow{5}{*}{ETTm2} & \multicolumn{1}{|c}{Teacher} &\multicolumn{1}{|c}{0.276}	&0.337	&0.267	&0.330	&0.279	&0.339 &0.274	&0.335	&0.348	&0.404	&0.359	&0.401	&0.291	&0.349	&0.420	&0.448 \\ 
&\multicolumn{1}{|c}{MLP} &\multicolumn{1}{|c}{0.300}		&0.373	&0.300	&0.373	&0.300	&0.373	&0.300	&0.373	&0.300	&0.373	&0.300	&0.373	&0.300	&0.373	&0.300	&0.373 \\
&\multicolumn{1}{|c}{\textbf{+\method{}}} &\multicolumn{1}{|c}{\textbf{0.252}}	&\textbf{0.316}	&\textbf{0.244}	&\textbf{0.311}	&\textbf{0.252}	&\textbf{0.317} &\textbf{0.261}	&\textbf{0.321}	&\textbf{0.262}	&\textbf{0.322}	&\textbf{0.262}	&\textbf{0.322}	&\textbf{0.258}	&\textbf{0.320}	&\textbf{0.265}	&\textbf{0.328} \\\cmidrule(lr){2-18} 
&\multicolumn{1}{|c}{$\Delta_{Teacher}$} &\multicolumn{1}{|c}{8.70\%}	&6.23\%	&8.63\%	&5.48\%	&9.67\%	&6.31\% &4.74\%	&4.18\%	&24.64\%	&20.28\%	&27.02\%	&19.70\%	&11.22\%	&8.23\%	&36.90\%	&26.79\%
\\
&\multicolumn{1}{|c}{$\Delta_{MLP}$} &\multicolumn{1}{|c}{16.07\%}		&15.28\%	&18.61\%	&16.50\%	&16.17\%	&14.91\%	&13.08\%	&13.94\%	&12.64\%	&13.75\%	&12.74\%	&13.67\%	&13.91\%	&14.12\%	&11.74\%	&12.06\%	 \\
\midrule
\multirow{5}{*}{ETT (avg)} & \multicolumn{1}{|c}{Teacher} &\multicolumn{1}{|c}{0.379}			&0.410	&0.371	&0.402	&0.382	&0.409	&0.400	&0.420	&0.651	&0.558	&0.441	&0.465	&0.404	&0.428	&0.829	&0.643 \\ 
&\multicolumn{1}{|c}{MLP} &\multicolumn{1}{|c}{0.397}		&0.428	&0.397	&0.428	&0.397	&0.428	&0.397	&0.428	&0.397	&0.428	&0.397	&0.428	&0.397	&0.428	&0.397	&0.428 \\
&\multicolumn{1}{|c}{\textbf{+\method{}}} &\multicolumn{1}{|c}{\textbf{0.345}}	&\textbf{0.387}	&\textbf{0.342}	&\textbf{0.381}	&\textbf{0.353}	&\textbf{0.392}	&\textbf{0.358}	&\textbf{0.398}	&\textbf{0.376}	&\textbf{0.405}	&\textbf{0.353}	&\textbf{0.396}	&\textbf{0.353}	&\textbf{0.392}	&\textbf{0.371}	&\textbf{0.409} \\\cmidrule(lr){2-18} 
&\multicolumn{1}{|c}{$\Delta_{Teacher}$} &\multicolumn{1}{|c}{8.92\%}			&5.79\%	&7.94\%	&5.09\%	&7.46\%	&4.16\%	&10.38\%	&5.36\%	&42.21\%	&27.41\%	&19.94\%	&14.99\%	&12.59\%	&8.42\%	&55.23\%	&36.49\%
\\
&\multicolumn{1}{|c}{$\Delta_{MLP}$} &\multicolumn{1}{|c}{13.06\%}		&9.77\%	&13.87\%	&10.97\%	&10.91\%	&8.50\%	&9.65\%	&7.20\%	&5.13\%	&5.46\%	&10.91\%	&7.67\%	&10.95\%	&8.41\%	&6.44\%	&4.63\%	 \\
\midrule
\multirow{5}{*}{Solar} & \multicolumn{1}{|c}{Teacher} &\multicolumn{1}{|c}{0.214}	&0.270	&0.191	&0.243	&0.288 &0.259	&0.210	&0.257	&0.213	&0.277	&0.300	&0.383	&0.196	&0.262	&1.037	&0.742 \\ 
&\multicolumn{1}{|c}{MLP} &\multicolumn{1}{|c}{0.194}		&0.255	&0.194	&0.255	&0.194	&0.255	&\textbf{0.194}	&\textbf{0.255}	&0.194	&0.255	&0.194	&0.255	&0.194	&0.255	&\textbf{0.194}	&\textbf{0.255} \\
&\multicolumn{1}{|c}{\textbf{+\method{}}} &\multicolumn{1}{|c}{\textbf{0.185}}	&\textbf{0.241}	&\textbf{0.184}	&\textbf{0.241}	&\textbf{0.187}	&\textbf{0.245} &0.204	&0.283	&\textbf{0.186}	&\textbf{0.242}	&\textbf{0.184}	&\textbf{0.241}	&\textbf{0.186}	&\textbf{0.243}	&0.237	&0.329 \\\cmidrule(lr){2-18} 
&\multicolumn{1}{|c}{$\Delta_{Teacher}$} &\multicolumn{1}{|c}{13.55\%}	&10.74\%	&3.60\%	& 0.71\%	&21.41\%	&14.98\% &2.86\%	&-10.12\%	&12.99\%	&12.47\%	&38.67\%	&37.08\%	&5.19\%	&6.97\%	&77.15\%	&55.66\%
\\
&\multicolumn{1}{|c}{$\Delta_{MLP}$} &\multicolumn{1}{|c}{4.80\%}		&5.42\%	&5.14\%	&5.24\%	&3.58\%	&3.86\%	&-4.98\%	&-11.07\%	&4.49\%	&4.83\%	&5.31\%	&5.42\%	&4.47\%	&4.49\%	&-21.97\%	&-29.12\%	 \\
\midrule
\multirow{5}{*}{Traffic} & \multicolumn{1}{|c}{Teacher} &\multicolumn{1}{|c}{\textbf{0.379}}	&0.271	&0.413	&0.284	&\textbf{0.391}	&0.275 &0.402	&0.284	&0.500	&0.316	&0.629	&0.388	&0.693	&0.399	&0.696	&0.427 \\ 
&\multicolumn{1}{|c}{MLP} &\multicolumn{1}{|c}{0.434}		&0.318	&0.434	&0.318	&0.434	&0.318	&0.434	&0.318	&0.434	&0.318	&0.434	&0.318	&0.434	&0.318	&0.434	&0.318 \\
&\multicolumn{1}{|c}{\textbf{+\method{}}} &\multicolumn{1}{|c}{0.389}	&\textbf{0.271}	&\textbf{0.387}	&\textbf{0.271} &0.391	 &\textbf{0.274}	&\textbf{0.390}	&\textbf{0.272}	&\textbf{0.391}	&\textbf{0.274}	&\textbf{0.395}	&\textbf{0.278}	&\textbf{0.400}	&\textbf{0.282}	&\textbf{0.396}	&\textbf{0.280} \\\cmidrule(lr){2-18} 
&\multicolumn{1}{|c}{$\Delta_{Teacher}$} &\multicolumn{1}{|c}{-2.64\%}	&0.00\%	&6.32\%	&4.74\%	&-0.04\%	 &0.38\% &2.99\%	&4.23\%	&21.78\%	&13.38\%	&37.20\%	&28.35\%	&42.35\%	&29.39\%	&43.10\%	&34.43\%
\\
&\multicolumn{1}{|c}{$\Delta_{MLP}$} &\multicolumn{1}{|c}{10.30\%}		&14.84\%	&10.70\%	&14.91\%	&9.76\%	&13.79\%	&10.07\%	&14.53\%	&9.78\%	&13.89\%	&8.92\%	&12.64\%	&7.83\%	&11.54\%	&8.69\%	&12.01\%	 \\
\midrule
\multirow{5}{*}{Weather} & \multicolumn{1}{|c}{Teacher} &\multicolumn{1}{|c}{0.259}	&0.290	&0.238	&0.277	&0.230	&0.271 &0.246	&0.283	&0.240	&0.292	&0.355	&0.398	&0.257	&0.294	&0.471	&0.465\\ 
&\multicolumn{1}{|c}{MLP} &\multicolumn{1}{|c}{0.234}		&0.294	&0.234	&0.294	&0.234	&0.294	&0.234	&0.294	&0.234	&0.294	&0.234	&0.294	&0.234	&0.294	&0.234	&0.294 \\
&\multicolumn{1}{|c}{\textbf{+\method{}}} &\multicolumn{1}{|c}{\textbf{0.220}}	&\textbf{0.270}	&\textbf{0.220}	&\textbf{0.269}	&\textbf{0.219}	&\textbf{0.266} &\textbf{0.220}	&\textbf{0.267}	&\textbf{0.225}	&\textbf{0.272}	&\textbf{0.224}	&\textbf{0.274}	&\textbf{0.222}	&\textbf{0.271}	&\textbf{0.226}	&\textbf{0.278} \\\cmidrule(lr){2-18} 
&\multicolumn{1}{|c}{$\Delta_{Teacher}$} &\multicolumn{1}{|c}{15.06\%}	&6.90\%	&7.37\%	&2.97\%	&4.82\%	&1.77\% &10.57\%	&5.65\%	&6.10\%	&6.61\%	&36.90\%	&31.16\%	&13.38\%	&7.59\%	&52.02\%	&40.22\%
\\
&\multicolumn{1}{|c}{$\Delta_{MLP}$} &\multicolumn{1}{|c}{5.83\%}		&8.02\%	&5.66\%	&8.48\%	&6.16\%	&9.38\%	&5.83\%	&9.04\%	&3.61\%	&7.26\%	&4.12\%	&6.66\%	&4.82\%	&7.57\%	&3.27\%	&5.30\%	 \\
\bottomrule

\end{tabular}%
}
\label{tab:app_different_teacher}
\end{table*}

\section{Different Teachers Analysis}
\label{app:full_different_teacher_results}
In this section, we evaluate the ability of \method{} to distill knowledge from various teacher models (additional teachers compared to Section~\ref{sec:versatility}) into a student MLP. The experiment involves eight teacher models: iTransformer, ModernTCN, TimeMixer, PatchTST, MICN, FEDformer, TimesNet, and Autoformer. The student model is an MLP designed as a channel-independent architecture, making it lightweight and computationally efficient. From Table~\ref{tab:app_different_teacher}, it is evident that iTransformer, ModernTCN, TimeMixer, PatchTST, MICN, FEDformer, TimesNet, and Autoformer improve the performance of MLP on MSE by 10.33\%, 10.88\%, 9.11\%, 7.23\%, 6.24\%, 8.90\%, 8.44\%, and 3.24\% on average across all datasets, respectively. These results demonstrate that \method{} can effectively distill knowledge from all the teacher models to improve student performance. Notably, ModernTCN improves MLP performance by up to 18.61\% on the ETTm2 dataset. 

On Traffic dataset, which involves a large number of channels (861 channels) and is particularly challenging for channel-independent models like the student MLP, \method{} achieves the second-best performance, improving MSE by 10.30\% when the teacher is iTransformer. This is likely due to \method{} can implicitly distill iTransformer's knowledge of modeling multivariate correlations to student MLP. We further explore whether \method{} can implicitly learn multi-variate correlation via distillation in Appendix~\ref{app:multivariate}.

However, when the teacher performs poorly on certain datasets, it can negatively impact the student MLP. For example, MLP achieves much better performance than the teacher Autoformer even without distillation. Autoformer performs exceptionally poorly on the Solar dataset, and after distillation, the MLP’s performance degrades by -29.12\% in terms of MAE. This suggests that poor teacher models may transfer excessive noise to the student MLP, resulting in degraded performance. Furthermore, we also observe that \method{} can effectively learn from various teachers, achieving comparable or even better performance than the teachers themselves. \method{} achieves significant improvements over the teachers themselves, with average gains of 8.13\%, 6.80\%, 7.55\%, 7.35\%, 23.92\%, 29.62\%, 18.22\%, and 48.59\% on MSE across all datasets. 

\section{Different Students Analysis}
\label{app:different_students}
\textbf{Analysis of MLP Architecture.} 
To examine the impact of the student MLP architecture on \method{}, we vary the number of layers and the hidden dimensions of the MLP. From Table~\ref{tab:app_different_student}, we observe that increasing the number of layers and hidden dimensions generally enhances performance. However, an excessive number of layers (e.g., 4 layers) or overly small hidden dimensions (e.g., 64) can lead to performance degradation. Therefore, to balance efficiency and performance, we select 2L-512 as the default MLP configuration.

\textbf{Analysis of Other MLP-Based Student Models.}
From Table~\ref{tab:app_different_student}, we observe that both LightTS and TSMixer outperform the original MLP. Incorporating \method{} further enhances their performance across all datasets, with improvements ranging from a minimum of 2.92\% to a maximum of 43.42\%, highlighting the effectiveness and adaptability of \method{}.

\section{Hyperparameter Sensitivity Analysis}
\label{app:hyperparameter_sensitivity}

\begin{figure*}[htbp]
    \centering
    \begin{tabular}{cccc}
        
        \includegraphics[width=0.23\textwidth]{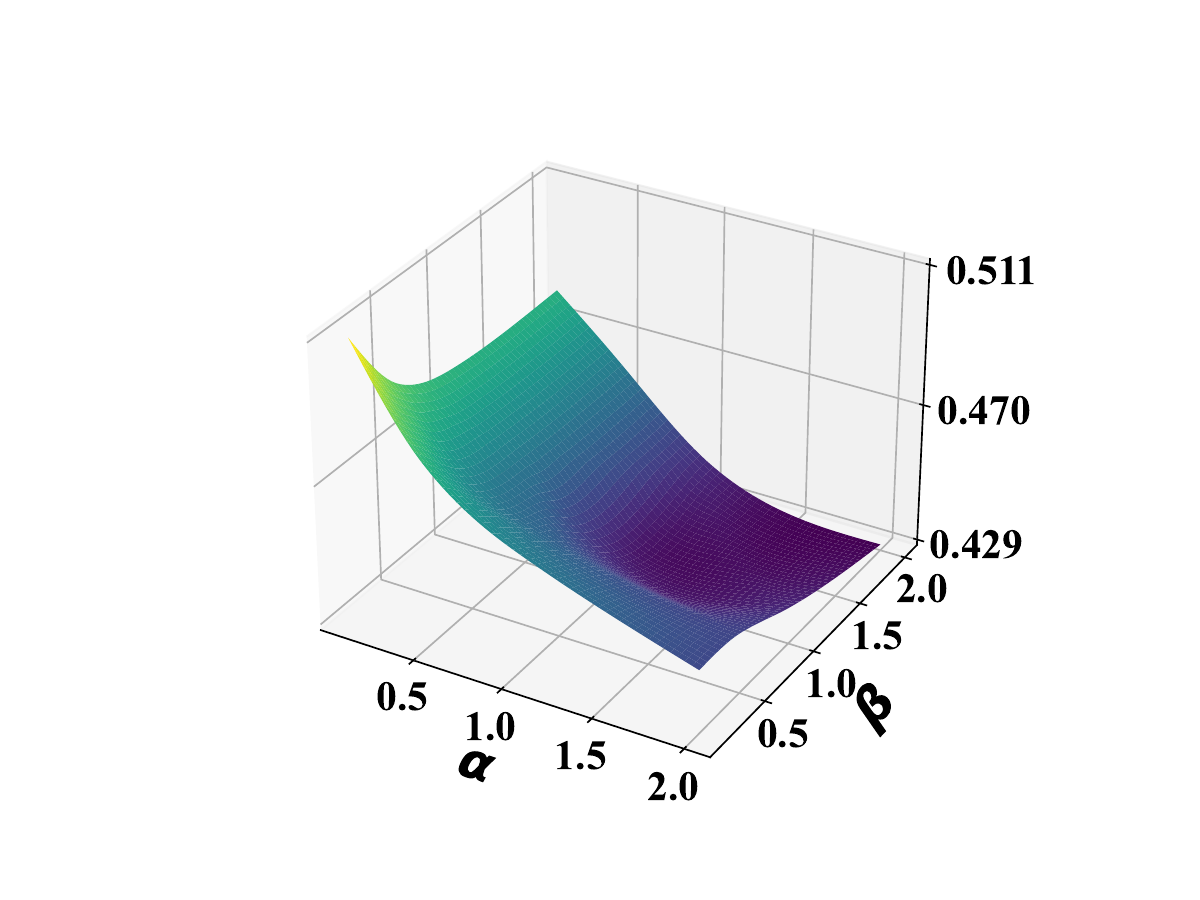} &
        \includegraphics[width=0.23\textwidth]{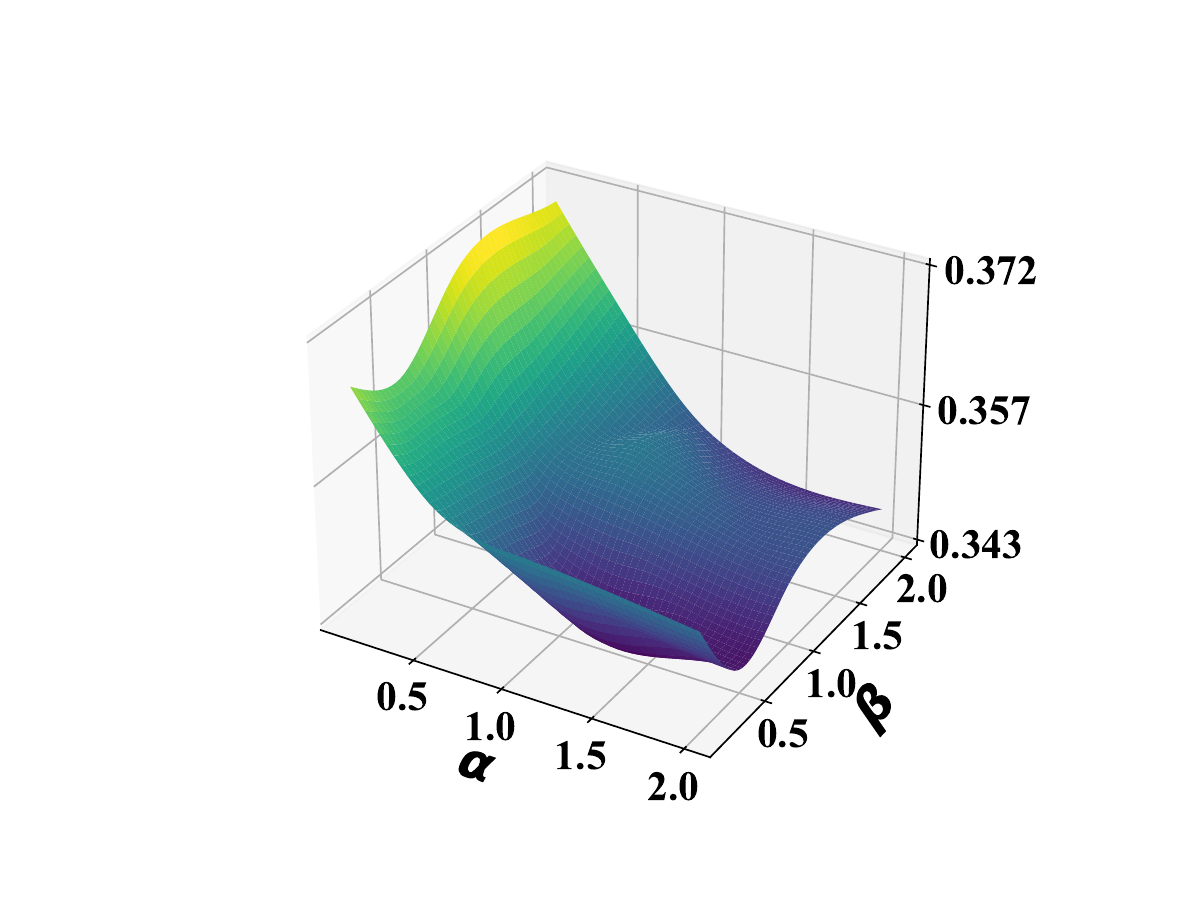} &
        \includegraphics[width=0.23\textwidth]{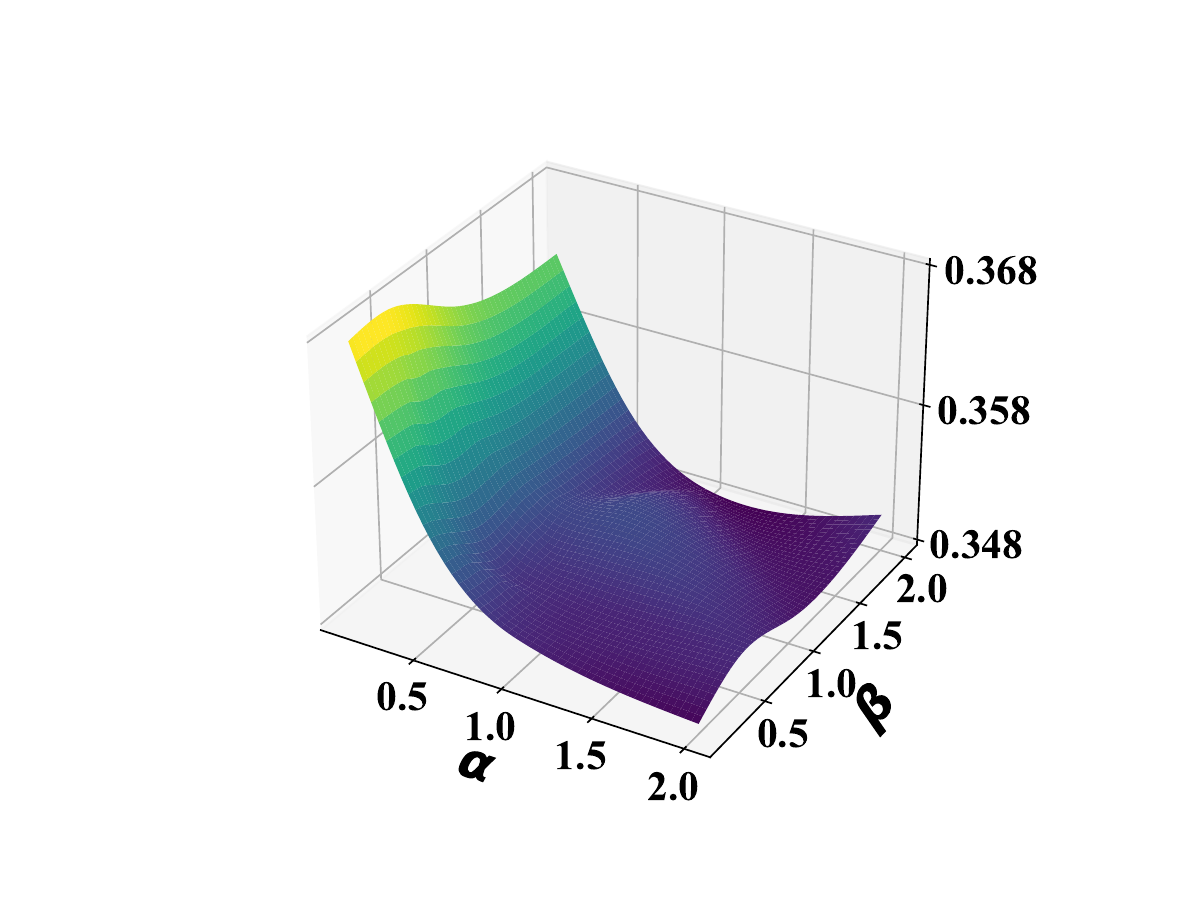} &
        \includegraphics[width=0.23\textwidth]{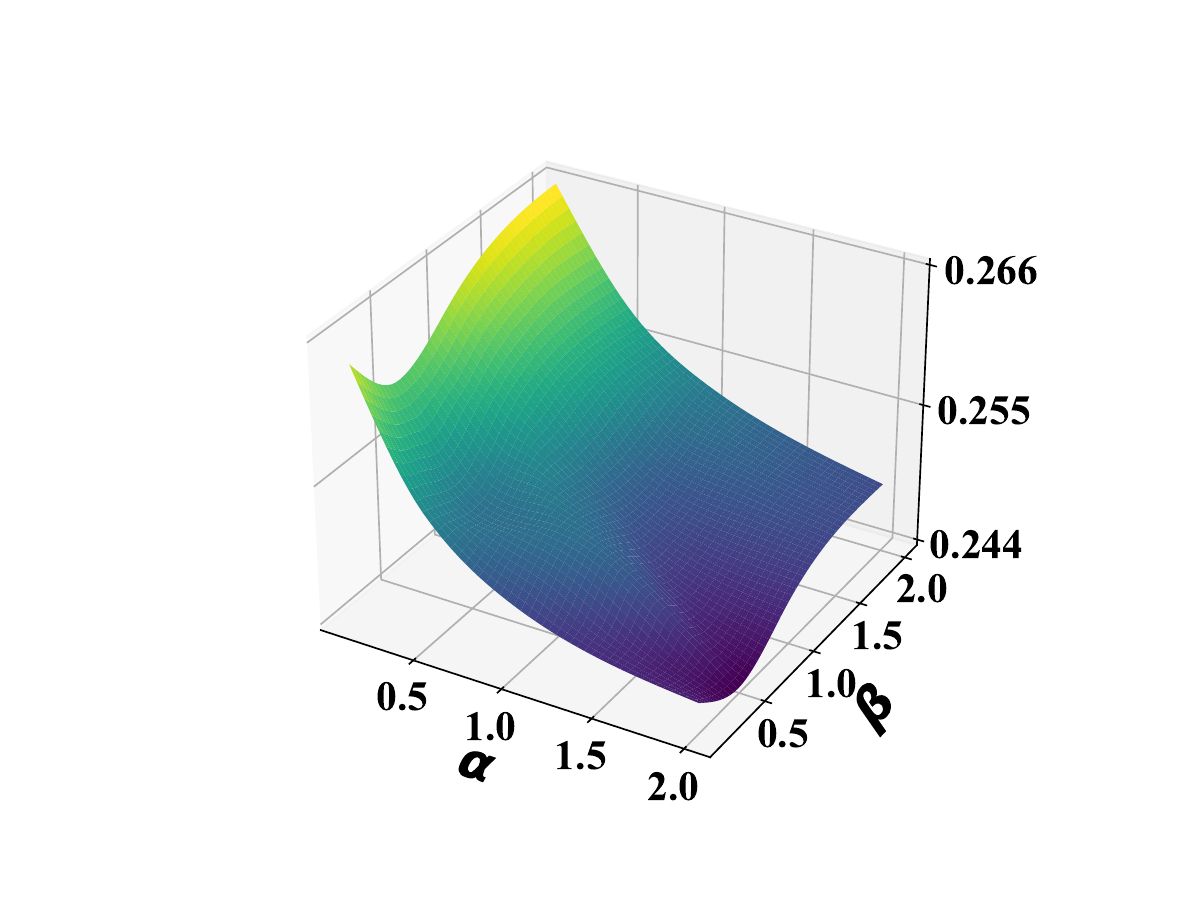} \\
        \textbf{(a) ETTh1} & \textbf{(b) ETTh2} & \textbf{(c) ETTm1} & \textbf{(d) ETTm2} \\
        \includegraphics[width=0.23\textwidth]{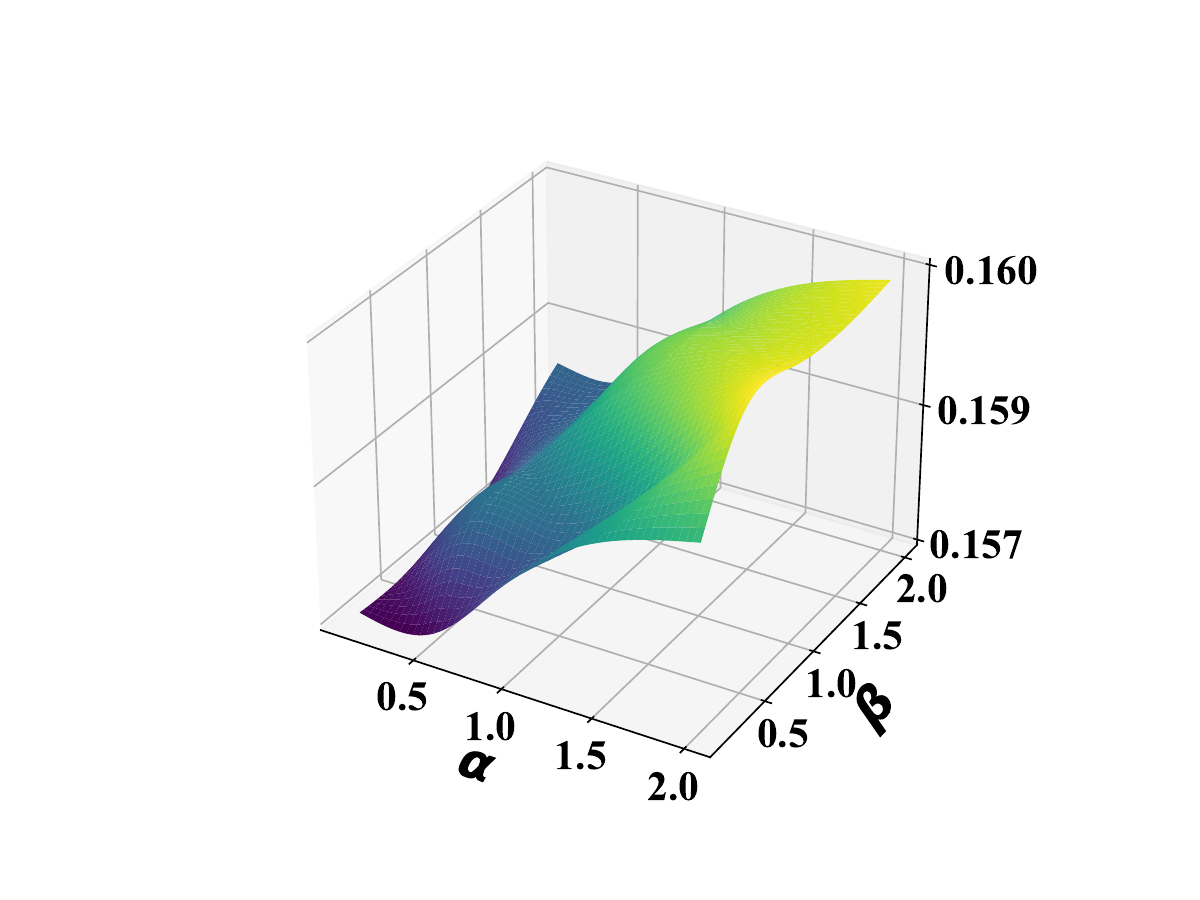} &
        \includegraphics[width=0.23\textwidth]{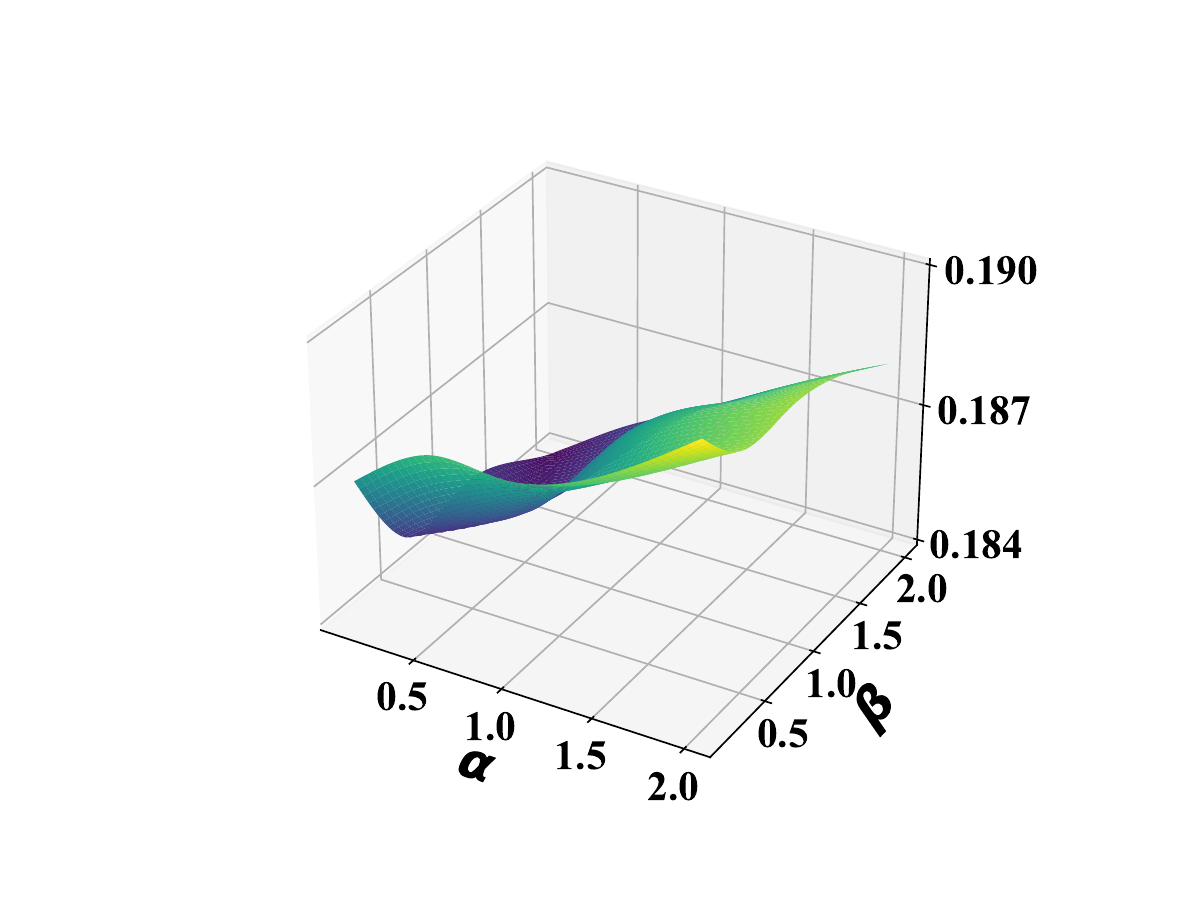} &
        \includegraphics[width=0.23\textwidth]{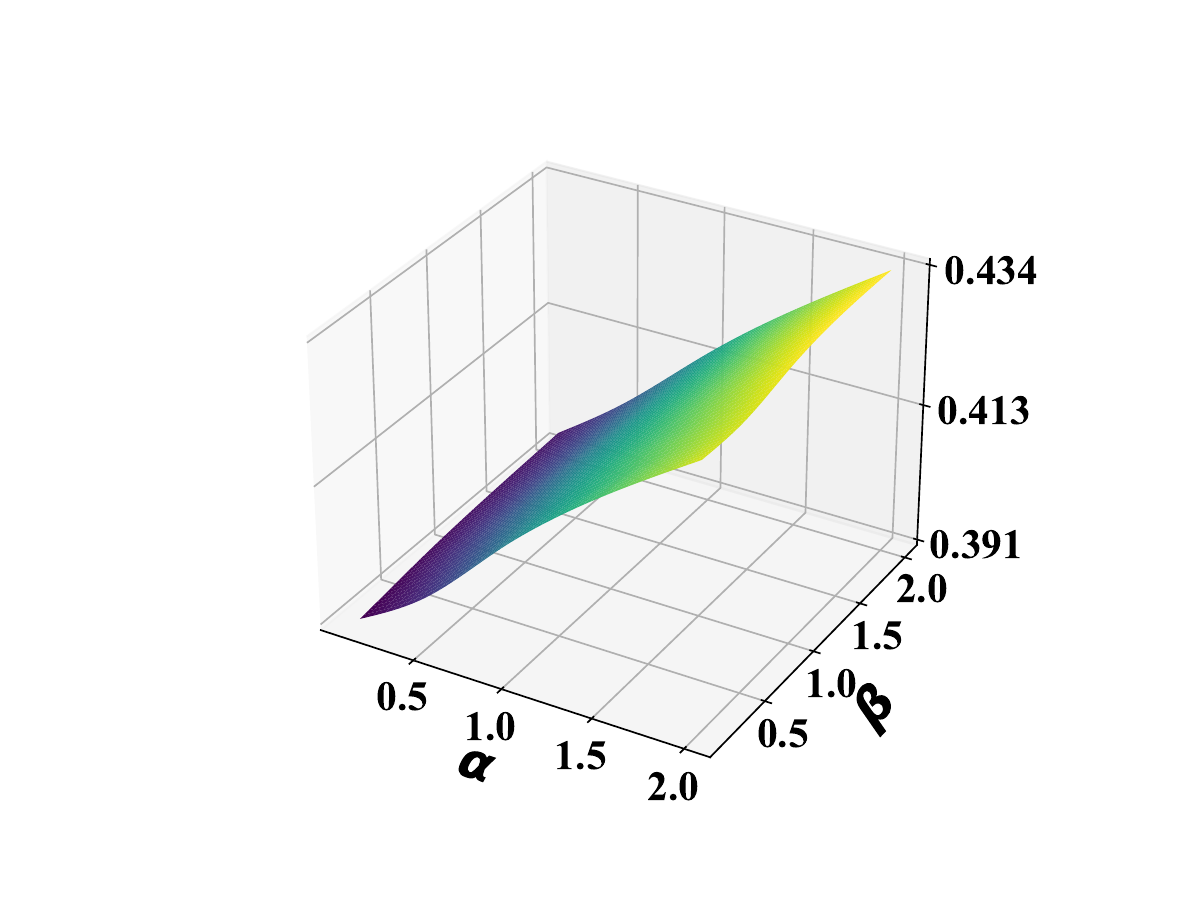} &
        \includegraphics[width=0.23\textwidth]{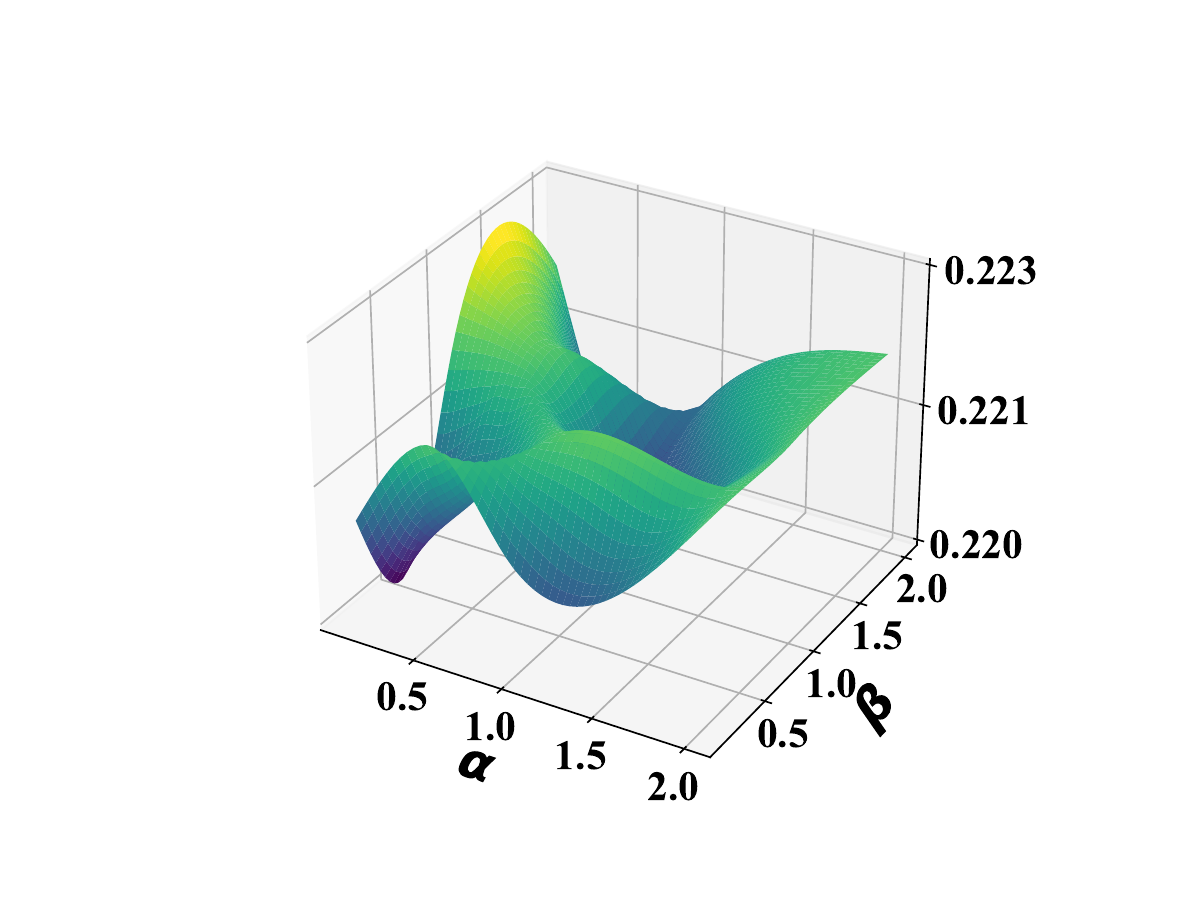} \\
        \textbf{(e) ECL} & \textbf{(f) Solar} & \textbf{(g) Traffic} & \textbf{(h) Weather} \\
    \end{tabular}
    \caption{Sensitivity Analysis Result of Predcition Level $\alpha$ and Feature Level $\beta$ Distillation on Different Datasets.}
    \label{fig:app_sensitivity}
\end{figure*}

\textbf{Sensitivity of Predcition Level $\alpha$ and Feature Level $\beta$ Distillation.}
In this subsection, we analyze the robustness of \method{} by investigating the sensitivity of two key hyperparameters, $\alpha$ and $\beta$, in the final objective function of \method{}. The parameters $\alpha$ and $\beta$ regulate the contributions of the prediction-level and feature-level distillation loss terms, respectively. To assess their effects, we vary both $\alpha$ and $\beta$ over the set $\{0.1, 0.5, 1, 2\}$ while keeping other hyperparameters fixed. We conduct experiments on all 8 datasets. Figure~\ref{fig:app_sensitivity} illustrates that \method{} generally performs better when a smaller $\beta$, such as 0.1 or 0.5, is selected. For $\alpha$, the results show dataset-specific preferences. On the four ETT datasets, a smaller $\alpha$ usually leads to inferior performance, while increasing $\alpha$ improves the results. This is because $\alpha$ controls the contribution of the loss term at the prediction level. As shown in Table~\ref{tab:app_different_teacher}, MLP performs poorly on ETT datasets, leading to a larger gap between MLP and the teacher. Matching predictions directly with a larger $\alpha$ helps reduce this performance gap, resulting in better outcomes on ETT datasets. However, for other datasets such as ECL, Solar, Traffic, and Weather, the performance gap between MLP and the teacher is not as significant, or MLP even outperforms the teacher. In such cases, a large $\alpha$ introduces excessive noise to MLP, negatively impacting the performance of \method{}. Consequently, on these datasets, a smaller $\alpha$ is more beneficial.

\begin{figure*}[htbp]
    \centering
    \begin{tabular}{cccc}
        \includegraphics[width=0.45\textwidth]{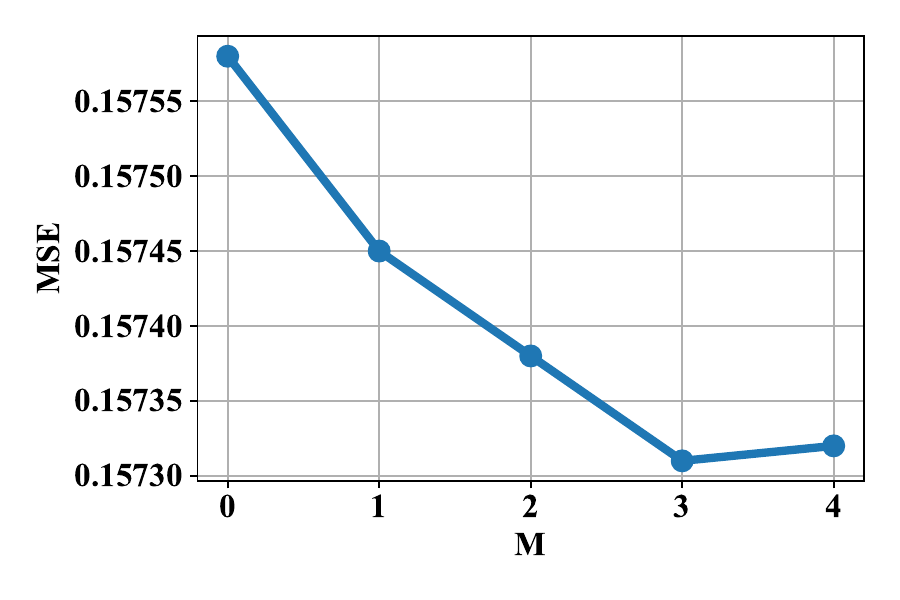} &
        \includegraphics[width=0.45\textwidth]{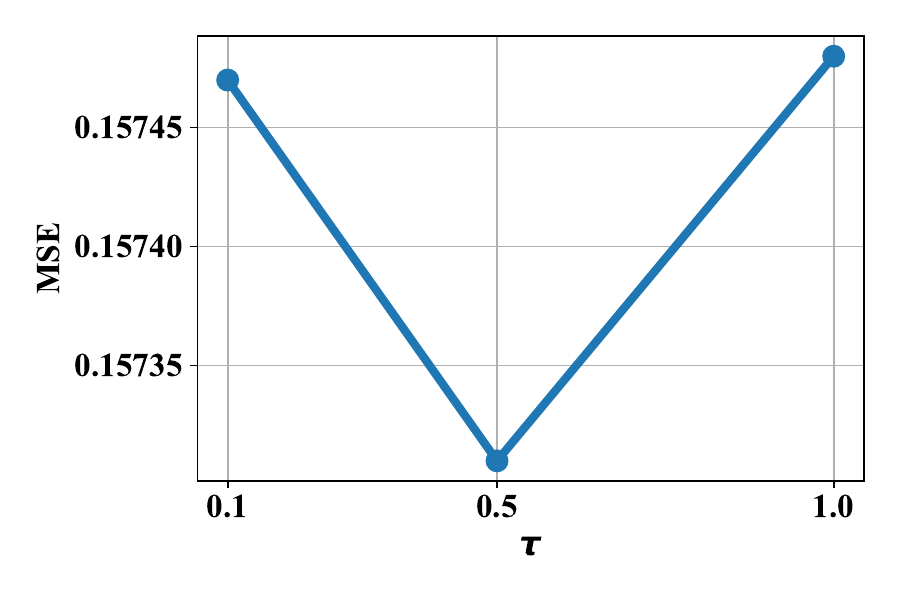}
    \end{tabular}
    \caption{Hyperparameter sensitivity with respect to the number of scales $M$ and temperature $\tau$ on ECL dataset. The results are recorded with the lookback window length T = 720 and averaged across all prediction window lengths $S~\in\{96, 192, 336, 720\}$.}
    \label{tab:app_m_tau}
\end{figure*}

\begin{table*}[htbp]
\centering
\caption{Results of directly match the patterns in ground truth time series on eight datasets. The look-back length is set to be consistent at 720 for fair comparison. The MSE and MAE metrics are averaged from all prediction lengths $S \in \{96, 192, 336, 720\}$.}
\label{tab:app_incremental_experiments}
\small
\resizebox{0.97\textwidth}{!}{
\begin{tabular}{l cc|cc|cc|cc|cc|cc|cc|cc}
\toprule
\multirow{2}{*}{Method} & \multicolumn{2}{c}{ECL} & \multicolumn{2}{c}{ETTh1} & \multicolumn{2}{c}{ETTh2} & \multicolumn{2}{c}{ETTm1} & \multicolumn{2}{c}{ETTm2} & \multicolumn{2}{c}{Solar} & \multicolumn{2}{c}{Traffic} & \multicolumn{2}{c}{Weather} \\
\cmidrule(lr){2-3}\cmidrule(lr){4-5}\cmidrule(lr){6-7}\cmidrule(lr){8-9}\cmidrule(lr){10-11}\cmidrule(lr){12-13}\cmidrule(lr){14-15}\cmidrule(lr){16-17}
 & MSE & MAE & MSE & MAE & MSE & MAE & MSE & MAE & MSE & MAE & MSE & MAE & MSE & MAE & MSE & MAE \\
\midrule
\method{} 
& \textcolor{red}{\textbf{0.157}} & \textcolor{red}{\textbf{0.254}}
& \textcolor{red}{\textbf{0.429}} & \textcolor{red}{\textbf{0.441}}
& \textcolor{red}{\textbf{0.345}} & \textcolor{red}{\textbf{0.395}}
& \textcolor{red}{\textbf{0.348}} & \textcolor{red}{\textbf{0.379}}
& \textcolor{red}{\textbf{0.244}} & \textcolor{red}{\textbf{0.311}}
& \textcolor{red}{\textbf{0.184}} & \textcolor{red}{\textbf{0.241}}
& \textcolor{red}{\textbf{0.391}} & \textcolor{red}{\textbf{0.275}}
& \textcolor{red}{\textbf{0.220}} & \textcolor{red}{\textbf{0.269}} \\
$\mathcal{L}^\mathbf{H} + \mathcal{L}^\mathbf{Y} + \mathcal{L}'_{\text{sup}}$ 
& 0.159 & 0.256 
& 0.436 & 0.449
& 0.348 & 0.397
& 0.349 & 0.380
& 0.247 & 0.314
& 0.185 & 0.243
& 0.393 & 0.277
& 0.222 & 0.270 \\
$\mathcal{L}^\mathbf{H} + \mathcal{L}'_{\text{sup}}$
& 0.159 & 0.257
& 0.478 & 0.479
& 0.383 & 0.427
& 0.374 & 0.404
& 0.271 & 0.332
& 0.185 & 0.243
& 0.394 & 0.278
& 0.223 & 0.275 \\
$\mathcal{L}^\mathbf{Y} + \mathcal{L}'_{\text{sup}}$
& 0.163 & 0.261
& 0.440 & 0.447
& 0.357 & 0.404
& 0.355 & 0.384
& 0.251 & 0.317
& 0.188 & 0.251
& 0.396 & 0.279
& 0.224 & 0.268 \\
$\mathcal{L}'_{\text{sup}}$
& 0.165 & 0.265
& 0.509 & 0.496
& 0.398 & 0.439
& 0.402 & 0.418
& 0.278 & 0.335
& 0.188 & 0.250
& 0.397 & 0.281
& 0.228 & 0.280 \\
\bottomrule
\end{tabular}
}
\end{table*}

\begin{table}[htbp]
\centering
\caption{Static and runtime metrics of \method{}, MLP, and other mainstream models on the ECL dataset. For fair comparison, the look-back length for each model is set to be consistent at 720 and the batch size is set to be 16. The static and runtime metrics are averaged from all prediction lengths $S~\in\{96, 192, 336, 720\}$.}
\label{tab:app_efficiency}
\resizebox{0.49\textwidth}{!}{
\begin{tabular}{c|cccc}
\hline
\multirow{2}{*}{Model} & \multirow{2}{*}{MSE} & \multirow{2}{*}{MAE} & Inference Time & \multirow{2}{*}{Parameters}\\
&&&\textit{(ms/batch)}& \\
\hline
\textbf{\method{}} & \textcolor{red}{\textbf{0.157}} & \textcolor{red}{\textbf{0.254}} & \textcolor{red}{\textbf{1.0956}}  & \textcolor{red}{\textbf{1,084,966}}  \\
\textbf{MLP}         & 0.173 & 0.276 & \textcolor{red}{\textbf{1.0956}}  & \textcolor{red}{\textbf{1,084,966}}  \\
iTransformer~\cite{itransformer} & 0.163 & 0.258 & 3.6377 & 5,276,496  \\
ModernTCN~\cite{moderntcn}   & 0.167 & 0.262 & 6.2233 & 132,075,892  \\
TimeMixer~\cite{timemixer}   & 0.165 & 0.259 & 7.7662 & 5,064,993  \\
PatchTST~\cite{patchtst}    & 0.165 & 0.266 & 2.9802 & 24,949,584 \\
MICN~\cite{micn}        & 0.181 & 0.293 & 2.715 & 60,215,726 \\
Crossformer~\cite{crossformer} & 0.203 & 0.299 & 8.7768 & 2,264,036  \\
Fedformer~\cite{fedformer}   & 0.274 & 0.376 & 32.2642 & 16,827,399 \\
Autoformer~\cite{autoformer}  & 0.238 & 0.347 & 196.5266 & 14,914,398 \\
\hline
\end{tabular}
}
\end{table}

\textbf{Sensitivity Analysis of the Number of Scales $M$ in Multi-Scale Distillation}
$M$ determines the number of scales used for matching in multi-scale distillation. To evaluate the robustness of \method{} with respect to $M$, we vary $M$ from 0 to 5 while keeping other hyperparameters fixed. The left panel of Figure~\ref{tab:app_m_tau} illustrates the forecasting performance (measured by MSE) on the ECL dataset as $M$ increases. When $M = 0$, \method{} does not downsample the time series and directly matches them. Consequently, \method{} relies solely on the noisy guidance signals at the finest scale from the teacher, resulting in suboptimal performance. As $M$ increases, the performance improves, demonstrating that multi-scale signals provide informative hierarchical knowledge to the MLP. To balance performance and computational efficiency, we set $M = 3$.

\textbf{Sensitivity Analysis of the Temperature $\tau$ in Multi-Period Distillation}
The temperature $\tau$ determines the extent to which the frequency distribution is softened. A larger $\tau$ results in a softer distribution. As shown in the right panel of Figure~\ref{tab:app_m_tau}, there is an optimal value of $\tau$ that delivers the best MSE. When $\tau = 1$, the distribution is softened, reducing the significance of high-magnitude frequencies. This process also amplifies the magnitudes of noisy frequencies, which should not be learned by the student MLP, leading to inferior performance. Conversely, when $\tau = 0.1$, only a few high-magnitude frequencies are retained, resulting in a frequency distribution that is less informative. To balance informativeness and noise, we set $\tau = 0.5$.

\section{Directly Match the Patterns in Ground Truth Time Series} 
\label{app:direct_match}
The ground truth usually contains rich information for model learning. Based on our preliminary analysis, the multi-scale and multi-period patterns should also exist in ground truth time series that can benefit MLP. Thus, in this section, we try to directly distill the multi-scale and multi-period patterns in supervised ground truth signals using the following loss: 
\begin{equation} 
    \mathcal{L}'_{sup} = \mathcal{L}_{scale}^{sup} + \mathcal{L}_{period}^{sup}. \label{eq:loss_sup} 
\end{equation}
We perform several incremental experiments: (1) $\mathcal{L}^\mathbf{H} + \mathcal{L}^\mathbf{Y} + \mathcal{L}'_{\text{sup}}$ means that we replace $\mathcal{L}_{sup}$ in Equation~\ref{eq:overall_optimization} with $\mathcal{L}'_{\text{sup}}$ defined in Equation~\ref{eq:loss_sup}. (2) $\mathcal{L}^\mathbf{H} + \mathcal{L}'_{\text{sup}}$ and $\mathcal{L}^\mathbf{Y} + \mathcal{L}'_{\text{sup}}$ denote we further remove multi-scale and multi-period distillation at logit level or at feature level, respectively.
(3) $\mathcal{L}'_{\text{sup}}$ means we only use $\mathcal{L}'_{\text{sup}}$ as the overall training loss $\mathcal{L}$ in Equation~\ref{eq:overall_optimization}, which only uses the ground truth for model learning and does not use the teacher.
From Table~\ref{tab:app_incremental_experiments}, we observe that adding $\mathcal{L}'_{\text{sup}}$ to \method{} results in an average MSE reduction of \textbf{0.8\%}, which slightly degrades performance rather than improving it. Notably, when \method{} is removed and only $\mathcal{L}'_{\text{sup}}$ is used to train the MLP, the performance deteriorates by approximately \textbf{9.5\%} on average. This can be attributed to the noise present in the ground truth, making it harder to fit, whereas the teacher model provides simpler and more learnable knowledge. Furthermore, when multi-scale and multi-period distillation at the feature level (i.e., $\mathcal{L}^\mathbf{Y} + \mathcal{L}'_{\text{sup}}$) is removed, performance declines on most datasets. This indicates that higher-dimensional features encapsulate more valuable knowledge compared to lower-dimensional logits or ground truth, highlighting the importance of feature-level distillation.

\section{Efficiency Analysis}
\label{app:efficiency}
In the main text, we have presented the efficiency analysis in Figure~\ref{fig:result_efficiency}. Here, we provide the quantitative results in Table~\ref{tab:app_efficiency}. Latency-sensitive applications require fast inference, making it crucial for models to minimize inference time. However, most works in the time series domain primarily focus on \textit{less time-sensitive} training time comparisons~\cite{itransformer, sparsetsf}. Additionally, edge-device applications demand models with fewer parameters for deployment. Therefore, in our experiments, we emphasize comparing both inference time and model parameters to evaluate practical efficiency. For a fair comparison, we fix the batch size to 16 during inference and record the average inference time per batch. Notably, \method{} (MLP) demonstrates a significant efficiency advantage over Transformer-based models, such as iTransformer, PatchTST, Crossformer, FEDformer, and Autoformer, as well as CNN-based models like ModernTCN and MICN, and the complex MLP-based model TimeMixer, which employs multiple MLPs. In terms of inference time, \method{} achieves up to \textbf{196$\times$} faster inference than other methods. Additionally, it requires up to \textbf{60$\times$} fewer parameters compared to other methods. Beyond the ECL dataset, we also visualize efficiency analyses for other datasets in Figure~\ref{fig:app_efficiency}. Across eight datasets, \method{} consistently resides in the lower-left corner, achieving the best trade-off between efficiency (inference time) and performance (MSE).

\section{Additional Cases}
\label{app:cases}

In this section, we present additional cases to demonstrate how \method{} effectively bridges the gap between the teacher and the student models. Figure~\ref{fig:vis_multiscale_multiperiod_ETTm1} and Figure~\ref{fig:vis_multiscale_multiperiod_ETTm2} showcase examples from the ETTm1 and ETTm2 datasets, respectively. From these figures, we make the following observations: \textbf{First}, in the temporal domain, the MLP struggles to capture the overall trend of the time series at coarser scales, resulting in a significant performance gap compared to the teacher at the finest scale. By effectively capturing the overall trend at coarser scales, \method{} narrows this gap, enabling the MLP to approximate the teacher's performance more closely. \textbf{Second}, in the frequency domain, the MLP fails to effectively capture the multi-periodic patterns of the time series, leading to inaccuracies in the frequency distribution. In contrast, \method{} accurately learns these multi-periodic patterns from the teacher, which helps the MLP improve its overall MSE performance.

\textbf{These observations highlight the importance of multi-scale and multi-period patterns distillation.} In the temporal domain, learning from coarser scales helps the student MLP capture the overall trend in the teacher's predictions without overfitting to noise present at the finest scale. Meanwhile, learning from the finest scale enables the MLP to refine details, resulting in more precise predictions. Consequently, jointly learning across multi-scale achieves a balance between the low accuracy of coarser scales and the high noise sensitivity of finer scales. In the frequency domain, learning from the frequency distribution enables the MLP to effectively capture overlapping periodicities, further enhancing prediction accuracy.

\begin{table}[t!]
\centering
\caption{Comparison of SpectralEntropy and Variance Ratios for different datasets}
\resizebox{0.49\textwidth}{!}{%
\begin{tabular}{lcccc}
\hline
Dataset & SpectralEntropy & VarRatio1H & VarRatio6H & VarRatio1D \\
\hline
ECL     & 7.77 & 1 & 0.84 & 0.76 \\
ETTh1   & 6.38 & 1 & 0.71 & 0.23 \\
ETTm1   & 6.46 & 1 & 0.70 & 0.23 \\
Traffic & 7.45 & 1 & 0.55 & 0.07 \\
Linear  & 2.36 & 1 & 1    & 1 \\
\hline
\end{tabular}
}
\label{tab:spectral_entropy_varratios}
\end{table}

\section{What if there is no Periodic or Multi-Scale Structure in the Dataset?} 
\label{app:limitation}
TimeDistill may perform suboptimally on time series without strong periodicity or multi-scale patterns. However, such cases are rare, as most real-world time series exhibit these properties~\citeyear{timesnet,autoformer}. To verify this, we analyze four datasets using spectral entropy and variance ratio (1H/6H/1D denote different scales). Spectral entropy measures the diversity of periodic components, while variance ratio captures changes across scales. High entropy and decreasing variance ratios indicate strong periodic and multi-scale structure. As shown in the Table~\ref{tab:spectral_entropy_varratios}, real-world datasets show both, unlike synthetic linear series, supporting TimeDistill's broad applicability.

Moreover, TimeDistill does not depend solely on multi-period distillation. As shown in Table~\ref{tab:ablation} in the paper, even without it, the framework still improves student performance. Additionally, the loss weights are adjustable, allowing adaptation to datasets with weak periodicity.

Furthermore, Theorem~\ref{thm:multiscale} and Theorem~\ref{thm:multiperiod} show that the multi-scale and multi-period loss act as mixup data augmentation between the teacher's and ground truth period distributions. This augmentation remains useful even with weak multi-scale or periodic structure, as long as a distributional gap exists.

\begin{table}[t!]
\centering
\caption{Comparison with other KD methods tailored for time series forecasting. Avg.Imp. denotes the average relative improvement over MLP.}
\resizebox{0.49\textwidth}{!}{
\begin{tabular}{lcccc}
\toprule
\textbf{Dataset} & \textbf{MLP} & \textbf{\method{}} & DE-TSMCL & LightCTS* \\
\midrule
ETTh1   & 0.502 & 0.429 & 0.479 & 0.448 \\
ETTh2   & 0.393 & 0.345 & 0.376 & 0.389 \\
ETTm1   & 0.391 & 0.348 & 0.386 & 0.373 \\
ETTm2   & 0.300 & 0.244 & 0.268 & 0.271 \\
Weather & 0.234 & 0.220 & 0.223 & 0.224 \\
\midrule
\textbf{Avg.Imp.} & - & \textbf{12.41\%} & 5.17\% & 5.99\% \\
\bottomrule
\end{tabular}
}
\label{tab:comparison_with_other_KD}
\end{table}

\begin{table*}[ht!]
\centering
\caption{Performance comparison on PEMS datasets (PEMS03, PEMS04, PEMS07, PEMS08, and average).}
\resizebox{0.5\textwidth}{!}{
\begin{tabular}{llccc}
\toprule
\textbf{Dataset} & \textbf{Metric} & \textbf{MLP} & \textbf{TimeMixer} & \textbf{TimeDistill-TimeMixer} \\
\midrule
\multirow{3}{*}{PEMS03}
& MAE  & 16.33 & 15.73 & \textbf{14.86} \\
& MAPE & 16.78 & 15.89 & \textbf{14.14} \\
& RMSE & 26.19 & 25.85 & \textbf{24.03} \\
\midrule
\multirow{3}{*}{PEMS04}
& MAE  & 22.47 & 20.00 & \textbf{19.96} \\
& MAPE & 15.34 & 13.94 & \textbf{12.26} \\
& RMSE & 35.39 & 32.80 & \textbf{32.65} \\
\midrule
\multirow{3}{*}{PEMS07}
& MAE  & 22.64 & 20.88 & \textbf{20.30} \\
& MAPE & 10.03 & 9.88  & \textbf{8.40} \\
& RMSE & 36.16 & 33.57 & \textbf{33.48} \\
\midrule
\multirow{3}{*}{PEMS08}
& MAE  & 17.22 & \textbf{14.89} & 15.02 \\
& MAPE & 11.07 & \textbf{9.39}  & 9.29 \\
& RMSE & 27.83 & \textbf{24.00} & 24.92 \\
\midrule
\multirow{3}{*}{\textbf{Average}}
& MAE  & 19.67 & 17.88 & \textbf{17.53} \\
& MAPE & 13.31 & 12.28 & \textbf{11.02} \\
& RMSE & 31.39 & 29.06 & \textbf{28.77} \\
\bottomrule
\end{tabular}
}
\label{tab:short_term_time_series}
\end{table*}

\section{Comparison with KD Tailored for Time Series Forecasting}
\label{app:Comparison with KD}

To our knowledge, \method{} is the \textbf{first} KD framework for time series forecasting with a time series-aware design (e.g., multi-scale, multi-period) and theoretical support from a data augmentation view. In contrast, LightCTS*~\cite{lai2024lightcts} and DE-TSMCL~\cite{gao2024distillation} apply general KD without time series-specific components. As shown in Table~\ref{tab:comparison_with_other_KD}, \method{} outperforms both, achieving nearly 2$\times$ the MLP improvement.

\section{Short-term Time Series Forecasting Performance}
\label{app:short_term}
\method{} can be directly adapted to the short-term time series forecasting task. Following TimeMixer~\cite{timemixer}, we provide some preliminary results in Table~\ref{tab:short_term_time_series}. We report the results over four PEMS datasets. \method{} significantly improves MLP’s performance and surpasses its teacher on the short-term forecasting task.

\section{Future Directions}
\label{app:future}
\paragraph{Distilling from advanced time--series models.}
Section~\ref{sec:versatility} shows that \method{} can already learn from a wide spectrum of teachers—ModernTCN, iTransformer, PatchTST, and TimeMixer—yielding MSE reductions of up to \(13.9\%\) for the student and up to \(10.4\%\) over the teachers themselves (Table~\ref{app:full_different_teacher_results}).  
The conclusion, therefore, proposes “\emph{distilling from advanced time-series (foundation) models}” as the next milestone.  
Concrete directions include: (i) treating recently released foundation checkpoints (e.g., TimesFM) as teachers and analysing how their broad temporal representations propagate through the multi-scale/multi-period objectives in Figure~\ref{fig:method}; and (ii) leveraging heterogeneous teacher ensembles (Appendix~\ref{app:full_different_teacher_results}, Table~\ref{tab:app_different_teacher}) to let a single student interpolate knowledge from climate, finance, and health domains.

\paragraph{Incorporating multivariate patterns explicitly.}
Although the current student is channel–independent, Section~\ref{sec:deeper_analysis} reveals that distillation already aligns its inter-variable correlation matrix with that of a channel-dependent teacher (Figure~\ref{fig:correlation_heatmaps}).  
Future work can (a) add a covariance-matching loss to reinforce this implicit transfer, and (b) draw on teachers specialised in cross-channel modelling, such as iTransformer.  
A realistic test-bed is the \textsc{Traffic} dataset, where \method{} still lowers MSE by \(10.3\%\) despite 861 variables (Table~\ref{app:full_different_teacher_results}).

\paragraph{Extending the framework to imputation and classification.}
Knowledge distillation is task-agnostic: Appendix~\ref{app:short_term} confirms that the same student–teacher pair handles short-term forecasting.  
Because the multi-scale and multi-period losses operate on reconstructed values and spectra, replacing the forecasting objective with a masked reconstruction loss would adapt \method{} to missing-value \emph{imputation}.  
Moreover, the mixup-based view in Theorem~\ref{thm:multiscale} naturally generalises the framework to classification by softening class targets.  
Validating these ideas on UCR/UEA (classification) and M4C (imputation) benchmarks—while re-using the qualitative analyses of Figure~\ref{fig:vis_multiscale_multiperiod_ETTh1} to track temporal fidelity—is an immediate next step.

\begin{table*}[ht!]
\centering
\caption{Comparison of methods on different datasets under three settings: Original, Non-stationary, and REVIN.}
\label{tab:normalization}
\resizebox{\textwidth}{!}{
\begin{tabular}{l|cccccccc|cccccccc|cccccccc}
\toprule
& \multicolumn{8}{c|}{None} & \multicolumn{8}{c|}{Non-stationary} & \multicolumn{8}{c}{REVIN} \\
\cmidrule(lr){2-9} \cmidrule(lr){10-17} \cmidrule(lr){18-25}
\multirow{2}{*}{Dataset} 
& \multicolumn{2}{c}{TimeDistill} & \multicolumn{2}{c}{iTransformer} & \multicolumn{2}{c}{ModernTCN} & \multicolumn{2}{c|}{TimeMixer}
& \multicolumn{2}{c}{TimeDistill} & \multicolumn{2}{c}{iTransformer} & \multicolumn{2}{c}{ModernTCN} & \multicolumn{2}{c|}{TimeMixer}
& \multicolumn{2}{c}{TimeDistill} & \multicolumn{2}{c}{iTransformer} & \multicolumn{2}{c}{ModernTCN} & \multicolumn{2}{c}{TimeMixer} \\
& MSE & MAE & MSE & MAE & MSE & MAE & MSE & MAE
& MSE & MAE & MSE & MAE & MSE & MAE & MSE & MAE
& MSE & MAE & MSE & MAE & MSE & MAE & MSE & MAE \\
\midrule
ECL   & 0.160 & 0.259 & 0.166 & 0.267 & 0.166 & 0.268 & 0.174 & 0.270
      & 0.157 & 0.254 & 0.163 & 0.259 & 0.166 & 0.268 & 0.162 & 0.257
      & 0.161 & 0.255 & 0.160 & 0.257 & 0.159 & 0.256 & 0.165 & 0.259 \\
ETTh1 & 0.443 & 0.450 & 0.514 & 0.521 & 0.508 & 0.507 & 0.460 & 0.472
      & 0.429 & 0.441 & 0.468 & 0.476 & 0.508 & 0.507 & 0.440 & 0.451
      & 0.433 & 0.448 & 0.468 & 0.477 & 0.443 & 0.447 & 0.459 & 0.465 \\
ETTh2 & 0.452 & 0.460 & 1.209 & 0.826 & 0.783 & 0.606 & 0.862 & 0.645
      & 0.345 & 0.395 & 0.398 & 0.426 & 0.783 & 0.606 & 0.382 & 0.412
      & 0.366 & 0.407 & 0.406 & 0.431 & 0.347 & 0.396 & 0.422 & 0.444 \\
ETTm1 & 0.362 & 0.392 & 0.426 & 0.453 & 0.535 & 0.508 & 0.374 & 0.397
      & 0.348 & 0.379 & 0.372 & 0.402 & 0.535 & 0.508 & 0.392 & 0.409
      & 0.354 & 0.391 & 0.374 & 0.403 & 0.365 & 0.388 & 0.367 & 0.388 \\
ETTm2 & 0.258 & 0.332 & 0.371 & 0.410 & 1.641 & 0.953 & 0.322 & 0.370
      & 0.244 & 0.311 & 0.276 & 0.337 & 1.641 & 0.953 & 0.281 & 0.332
      & 0.270 & 0.324 & 0.275 & 0.336 & 0.267 & 0.330 & 0.279 & 0.339 \\
Solar & 0.184 & 0.242 & 0.207 & 0.264 & 0.192 & 0.246 & 0.238 & 0.288
      & 0.184 & 0.241 & 0.214 & 0.270 & 0.192 & 0.246 & 0.238 & 0.288
      & 0.193 & 0.248 & 0.203 & 0.273 & 0.191 & 0.259 & 0.238 & 0.288 \\
Traffic & 0.395 & 0.277 & 0.661 & 0.407 & 0.414 & 0.289 & 0.426 & 0.282
      & 0.404 & 0.278 & 0.379 & 0.271 & 0.414 & 0.289 & 0.414 & 0.289
      & 0.387 & 0.271 & 0.385 & 0.274 & 0.413 & 0.284 & 0.391 & 0.275 \\
Weather & 0.236 & 0.299 & 0.301 & 0.345 & 0.244 & 0.310 & 0.235 & 0.287
        & 0.220 & 0.269 & 0.259 & 0.290 & 0.244 & 0.310 & 0.232 & 0.273
        & 0.226 & 0.270 & 0.260 & 0.291 & 0.229 & 0.269 & 0.230 & 0.271 \\
\midrule
Average & 0.311 & 0.339 & 0.482 & 0.437 & 0.560 & 0.461 & 0.386 & 0.376
        & 0.292 & 0.321 & 0.316 & 0.341 & 0.560 & 0.461 & 0.318 & 0.339
        & 0.299 & 0.327 & 0.316 & 0.343 & 0.302 & 0.329 & 0.319 & 0.341 \\
\bottomrule
\end{tabular}
}
\end{table*}

\section{Effect of Normalization on Performance Across Methods}
\label{app:effect_normalization}
Our experiment setup follows standard practice~\cite{itransformer,non-stationary,sparsetsf} to ensure fairness: each model uses its best-performing normalization based on validation. To reduce confounding factors, we also evaluate all models under the same normalization. As shown in Tables~\ref{tab:normalization} (Non-stationary and REVIN), \method{} consistently outperforms prior methods, confirming that its gains come from the distillation method rather than preprocessing differences.

We also study the impact of normalization by removing it from both \method{} and baselines in Table~\ref{tab:normalization} (None). Results show that \method{} consistently outperforms all baselines. This confirms that its gains are not due to preprocessing, but to the method itself, improving interpretability.

\section{Showcases}
For clear comparison, we present test set showcases in Appendix~\ref{app:show_cases}, where \method{} shows better performance.

\clearpage
\begin{figure*}[h!]
    \centering
    \begin{minipage}{0.45\textwidth}
        \centering
        \includegraphics[width=\textwidth]{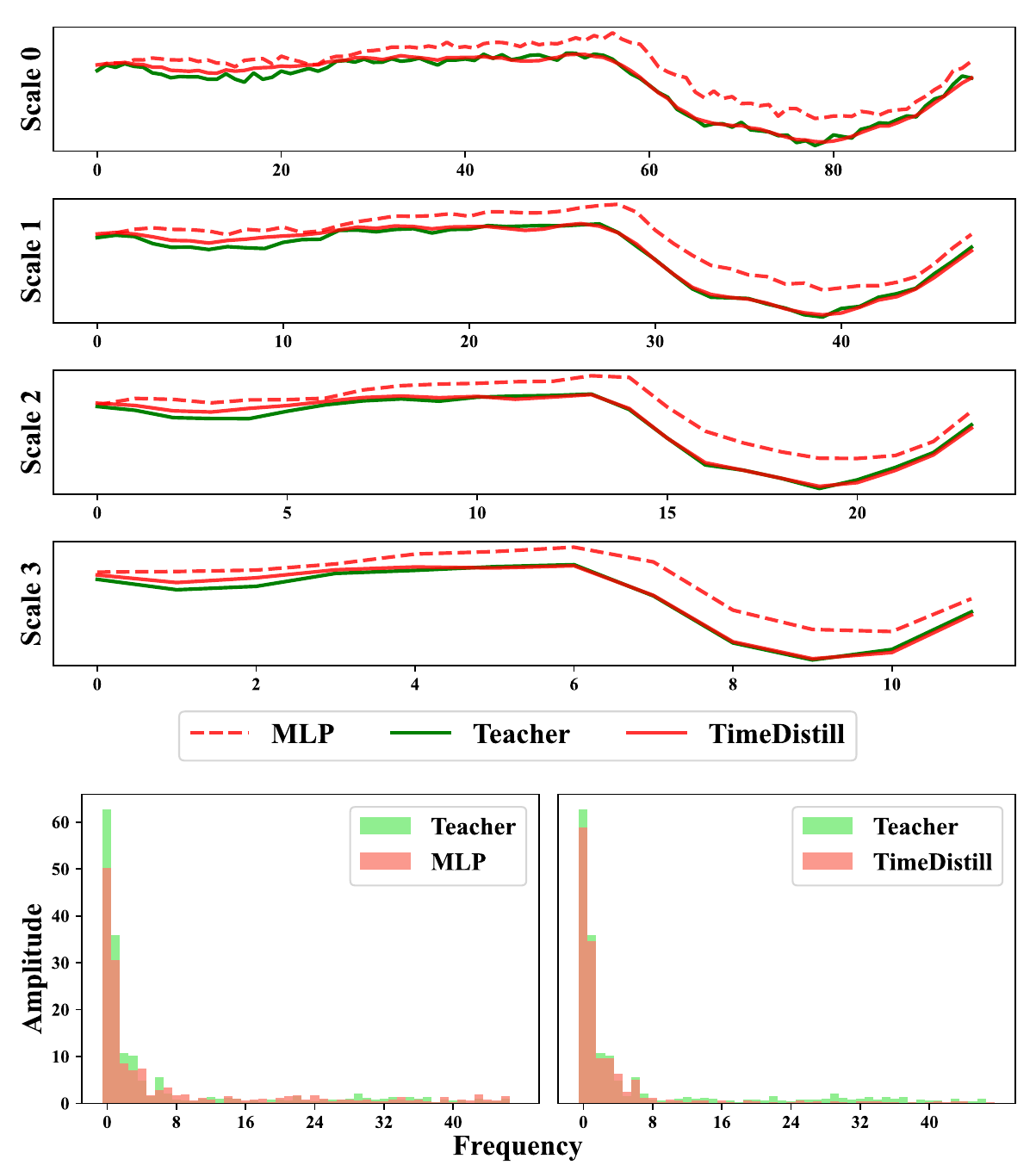}
        \caption{{Comparison of model predictions across temporal scales and corresponding spectrograms before and after distillation on the ETTm1 dataset.} MSE values for MLP, Teacher (ModernTCN~\cite{moderntcn}), and \method{} are 2.16, 0.73, and 0.74.}
        \label{fig:vis_multiscale_multiperiod_ETTm1}
    \end{minipage}
    \hspace{0.05\textwidth}
    \begin{minipage}{0.45\textwidth}
        \centering
        \includegraphics[width=\textwidth]{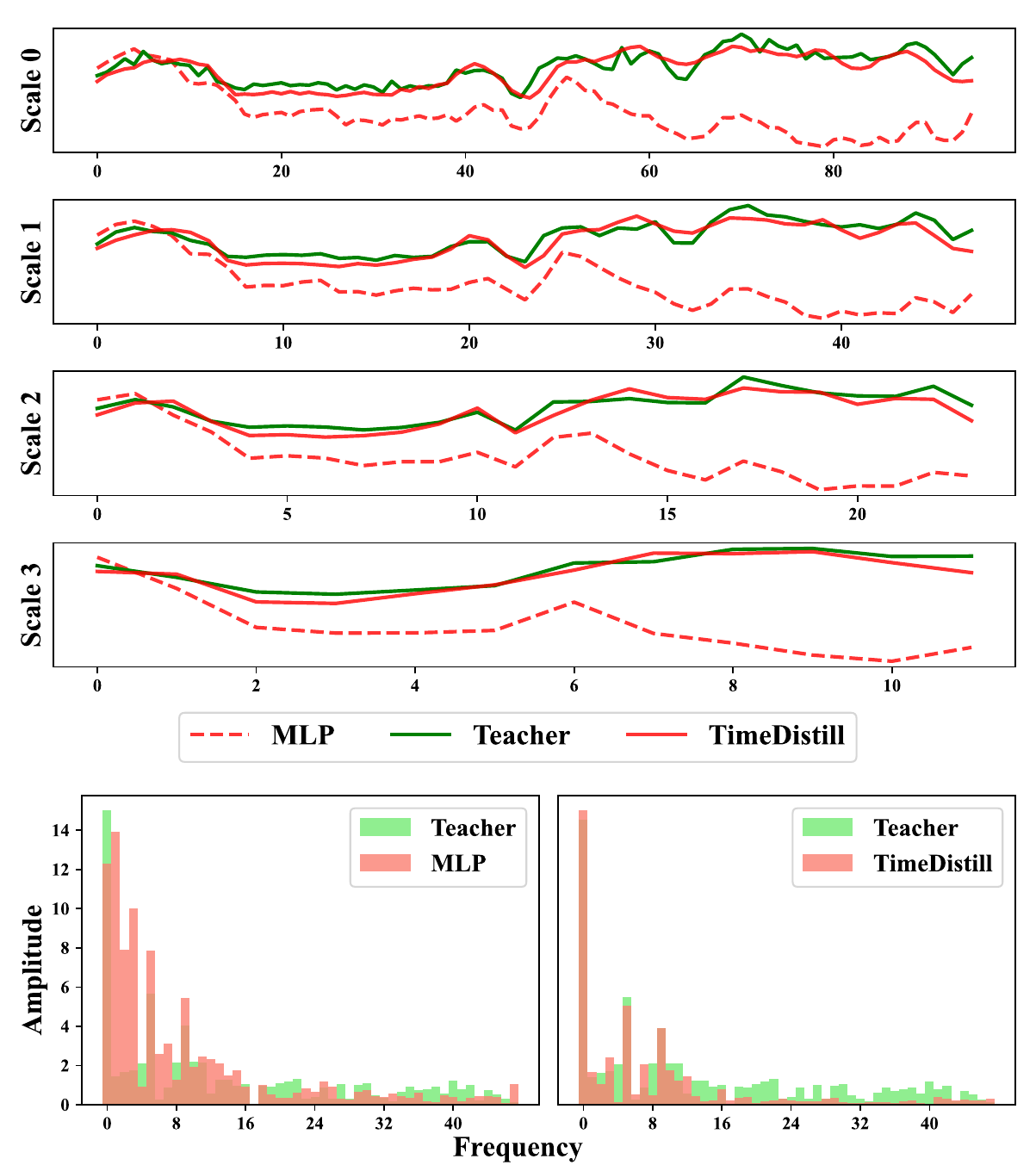}
        \caption{{Comparison of model predictions across temporal scales and corresponding spectrograms before and after distillation on the ETTm2 dataset.} MSE values for MLP, Teacher (ModernTCN~\cite{moderntcn}), and \method{} are 3.35, 1.19, and 1.28.}
        \label{fig:vis_multiscale_multiperiod_ETTm2}
    \end{minipage}
\end{figure*}

\begin{table*}[ht]
\centering
\caption{Performance promotion obtained by our \textbf{\method{}} framework with different student models. 
The notation nL-d represents the structure of an MLP model, where nL specifies the number of layers and d denotes the dimension of the hidden units in each layer. 
For instance, 2L-512 refers to a model with 2 layers and a hidden dimension of 512. 
We report the average performance from all prediction lengths. 
$\Delta_{Student}$ represents the performance promotion between \textbf{Student+\method{}} and original trained student.}
\label{tab:app_different_student}
\resizebox{0.99\textwidth}{!}{%
\begin{tabular}{cc cc|cc|cc|cc|cc|cc|cc}
\toprule
\multicolumn{2}{c}{\multirow{3}{*}{Student Models}} 
& \multicolumn{2}{c}{MLP} 
& \multicolumn{2}{c}{MLP} 
& \multicolumn{2}{c}{MLP} 
& \multicolumn{2}{c}{MLP} 
& \multicolumn{2}{c}{LightTS} 
& \multicolumn{2}{c}{TSMixer}
& \multicolumn{2}{c}{FITS}\\
& & \multicolumn{2}{c}{(2L-512)} & \multicolumn{2}{c}{(3L-512)} & \multicolumn{2}{c}{(4L-512)} & \multicolumn{2}{c}{(2L-1024)} &  & & & \\
& & MSE & MAE & MSE & MAE & MSE & MAE & MSE & MAE & MSE & MAE & MSE & MAE & MSE & MAE\\
\midrule
\multirow{3}{*}{ETTh1} 
& Student
& 0.502 & 0.489
& 0.487 & 0.497
& 0.468 & 0.481
& 0.495 & 0.500
& 0.465 & 0.471
& 0.471 & 0.474 
&0.428	&0.443\\
& \textbf{+\method{}} 
& \textbf{0.428} & \textbf{0.445}
& \textbf{0.442} & \textbf{0.448}
& \textbf{0.443} & \textbf{0.453}
& \textbf{0.442} & \textbf{0.446}
& \textbf{0.436} & \textbf{ 0.445}
& \textbf{0.433} & \textbf{0.446} 
&\textbf{0.411}	&\textbf{0.430}\\
\cmidrule{2-16}
& $\Delta_{Student}$
& 14.74\% & 9.00\%
& 9.11\%  & 9.86\%
& 5.29\%  & 5.70\%
& 10.79\% & 10.83\%
& 6.26\%  & 5.57\%
& 8.02\%  & 5.92\% 
&3.96\%	&2.75\%\\
\midrule
\multirow{3}{*}{ETTh2} 
& Student
& 0.393 & 0.438
& 0.742 & 0.601
& 0.908 & 0.668
& 0.639 & 0.563
& 0.675 & 0.582
& 0.376 & 0.420 
&0.338	&0.388\\
& \textbf{+\method{}} 
& \textbf{0.345} & \textbf{0.397}
& \textbf{0.341} & \textbf{0.390}
& \textbf{0.362} & \textbf{0.406}
& \textbf{0.344} & \textbf{0.390}
& \textbf{0.382} & \textbf{0.423}
& \textbf{0.351} & \textbf{0.399} 
&\textbf{0.333}	&\textbf{0.382}\\
\cmidrule{2-16}
& $\Delta_{Student}$
& 12.16\% & 9.46\%
& 53.97\% & 35.05\%
& 60.12\% & 39.19\%
& 46.20\% & 30.72\%
& 43.42\% & 27.26\%
& 6.83\%  & 5.01\% 
&1.52\%	&1.70\%\\
\midrule
\multirow{3}{*}{ETTm1}
& Student
& 0.391 & 0.413
& 0.378 & 0.406
& 0.392 & 0.414
& 0.377 & 0.402
& 0.376 & 0.399
& 0.370 & 0.395 
& 0.357	& 0.379\\
& \textbf{+\method{}} 
& \textbf{0.354} & \textbf{0.390}
& \textbf{0.346} & \textbf{0.380}
& \textbf{0.353} & \textbf{0.389}
& \textbf{0.346} & \textbf{0.380}
& \textbf{0.358} & \textbf{ 0.382}
& \textbf{0.354} & \textbf{0.383} 
&\textbf{0.358}	&\textbf{0.380}\\
\cmidrule{2-16}
& $\Delta_{Student}$
& 9.42\% & 5.65\%
& 8.54\% & 6.34\%
& 9.90\% & 6.12\%
& 8.40\% & 5.56\%
& 4.83\% & 4.29\%
& 4.13\% & 2.92\% 
& -0.45\% &-0.27\%\\
\midrule
\multirow{3}{*}{ETTm2}
& Student
& 0.300 & 0.373
& 0.359 & 0.408
& 0.346 & 0.393
& 0.307 & 0.374
& 0.283 & 0.346
& 0.295 & 0.347 
&0.253	&0.316\\
& \textbf{+\method{}} 
& \textbf{0.252} & \textbf{0.316}
& \textbf{0.259} & \textbf{0.321}
& \textbf{0.254} & \textbf{0.317}
& \textbf{0.245} & \textbf{0.312}
& \textbf{0.258} & \textbf{0.324}
& \textbf{0.253} & \textbf{0.319} 
&\textbf{0.248}	&\textbf{0.311}\\
\cmidrule{2-16}
& $\Delta_{Student}$
& 16.07\% & 15.28\%
& 27.77\% & 21.31\%
& 26.71\% & 19.41\%
& 20.03\% & 16.55\%
& 8.92\%  & 6.38\%
& 14.19\% & 8.23\% 
&2.20\%	&1.40\%\\
\midrule
\multirow{3}{*}{Weather}
& Student
& 0.234 & 0.294
& 0.232 & 0.289
& 0.237 & 0.291
& 0.221 & 0.278
& 0.235 & 0.293
& 0.239 & 0.278 
&0.241	&0.281\\
& \textbf{+\method{}} 
& \textbf{0.220} & \textbf{0.270}
& \textbf{0.222} & \textbf{0.267}
& \textbf{0.225} & \textbf{0.271}
& \textbf{0.220} & \textbf{0.267}
& \textbf{0.221} & \textbf{0.272}
& \textbf{0.222} & \textbf{0.266} 
&\textbf{0.239}	&\textbf{0.279}\\
\cmidrule{2-16}
& $\Delta_{Student}$
& 5.83\% & 8.02\%
& 4.53\% & 7.53\%
& 5.22\% & 7.04\%
& 0.53\% & 4.10\%
& 5.95\% & 7.17\%
& 7.19\% & 4.09\% 
& 0.90\% &0.74\%\\
\bottomrule
\end{tabular}%
}
\end{table*}

\onecolumn
\clearpage
\begin{figure}[htbp]
    \centering
    \begin{tabular}{cccc}
        \includegraphics[width=0.4\textwidth]{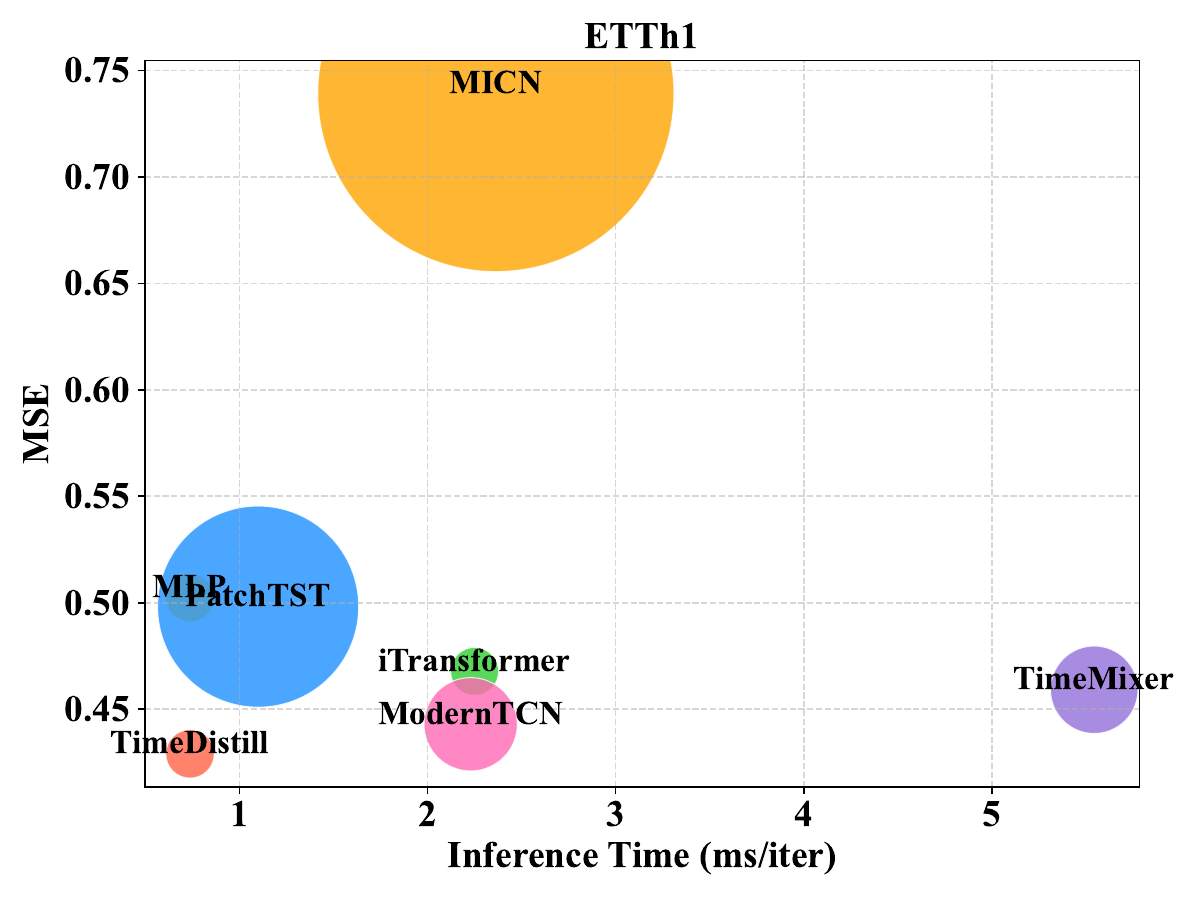} &
        \includegraphics[width=0.4\textwidth]{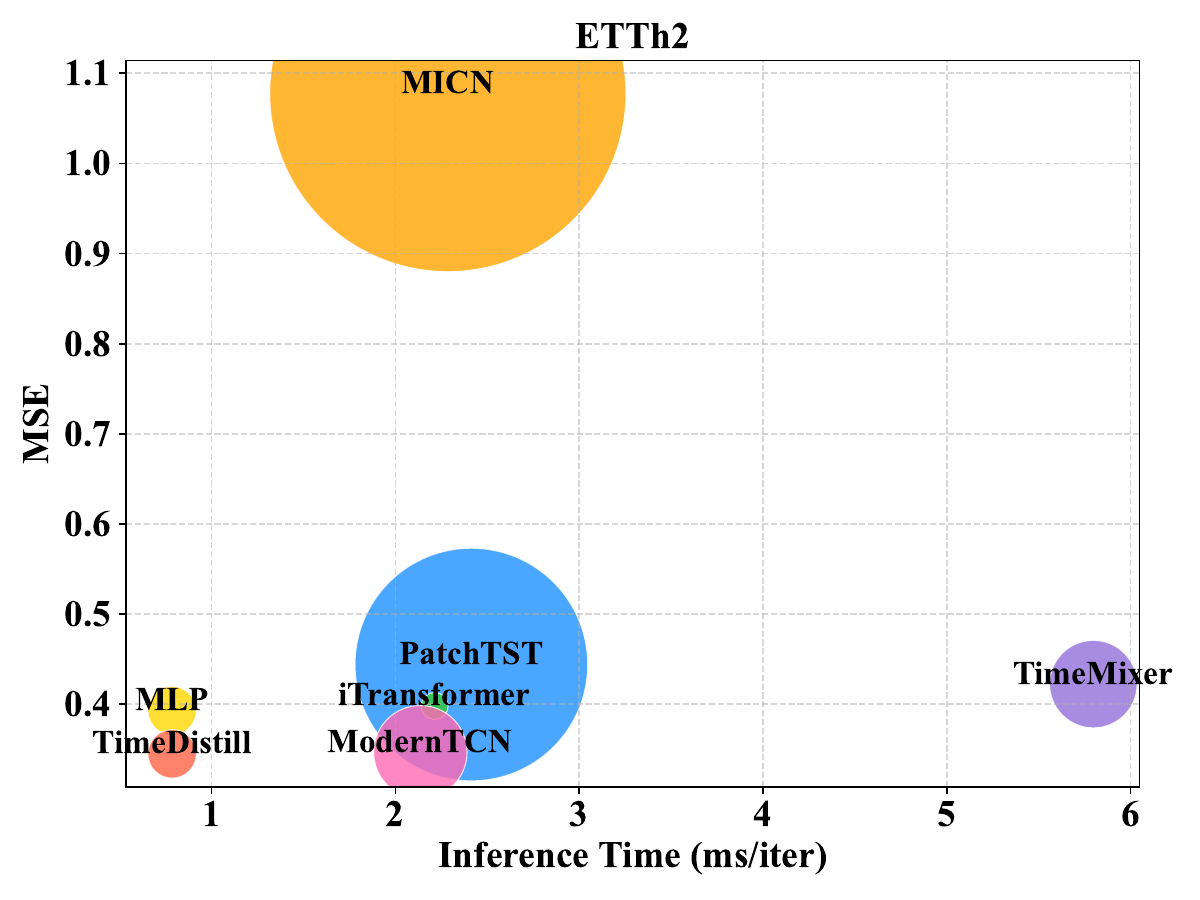}\\
        \includegraphics[width=0.4\textwidth]{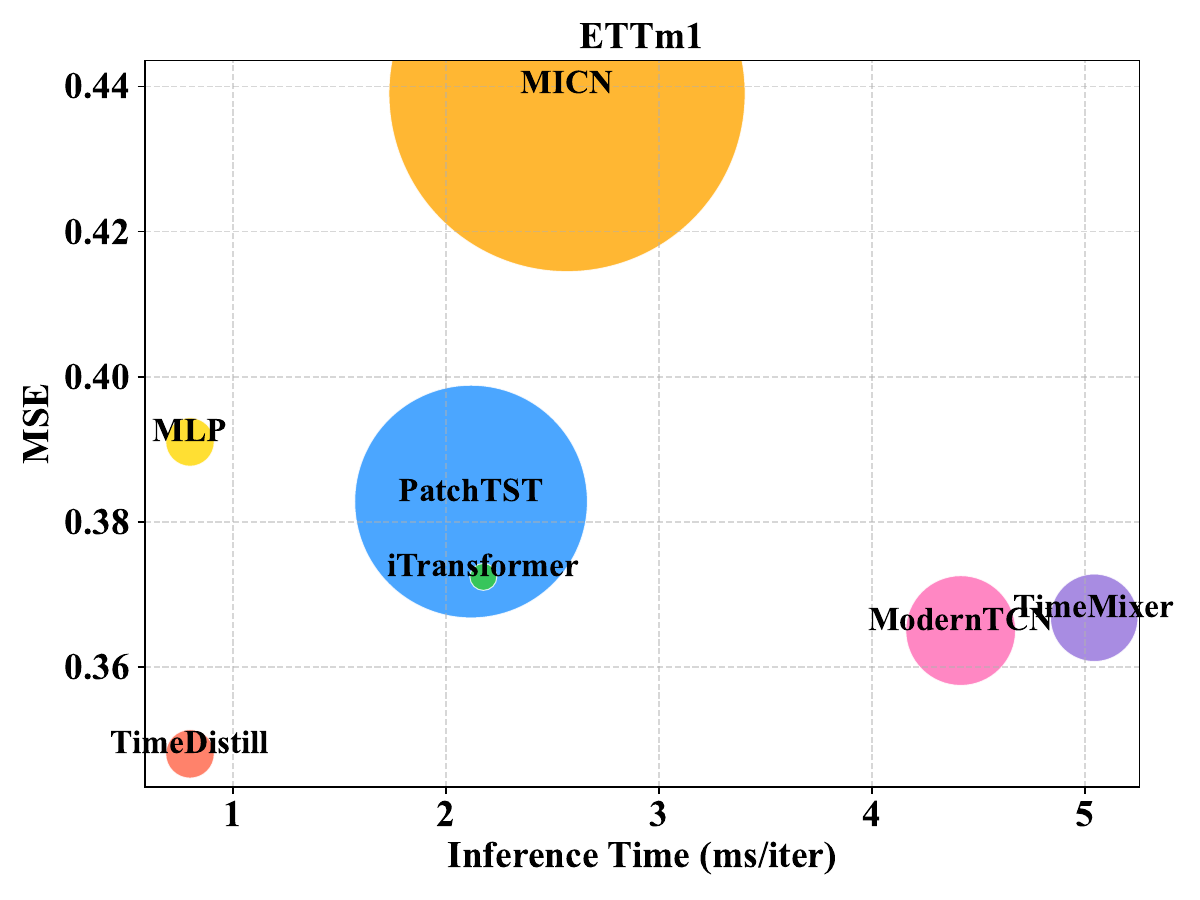}&
        \includegraphics[width=0.4\textwidth]{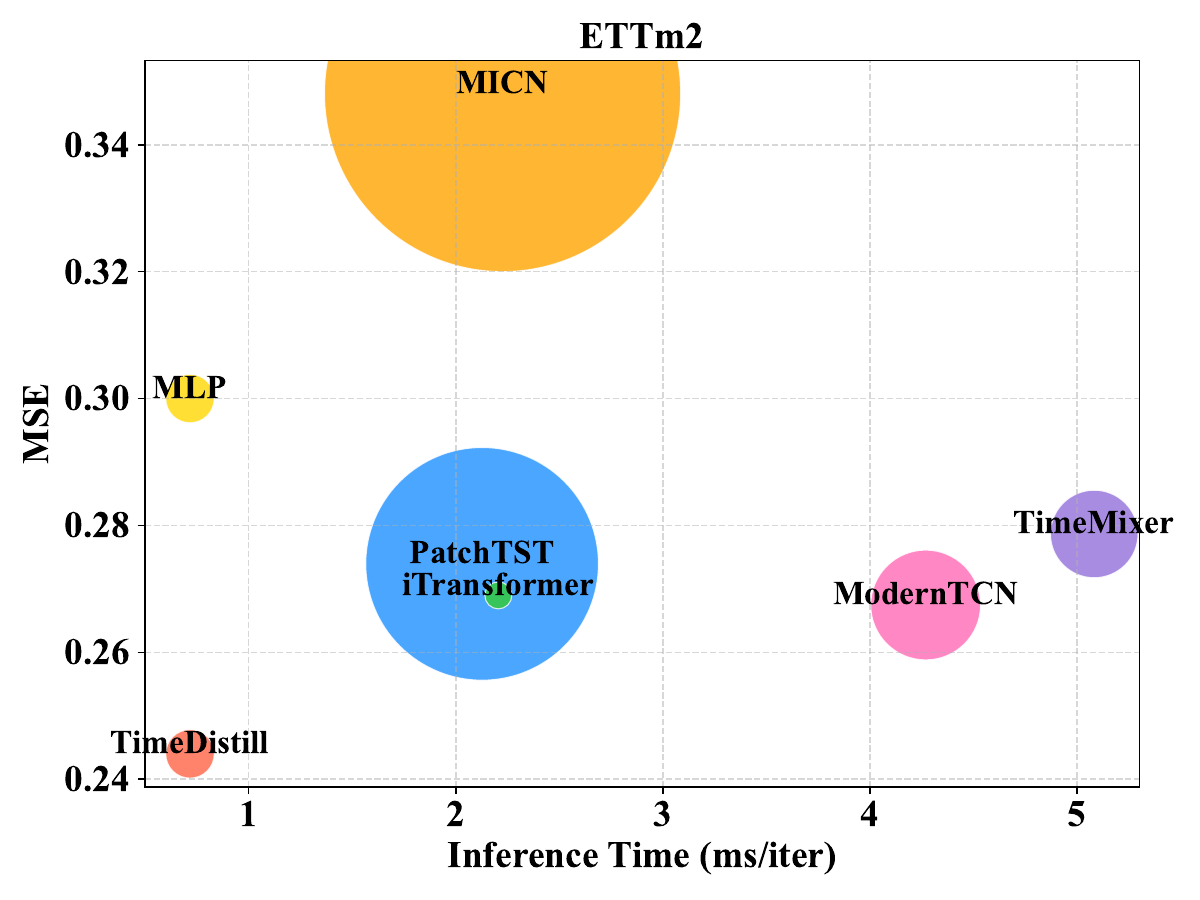} \\
        \includegraphics[width=0.4\textwidth]{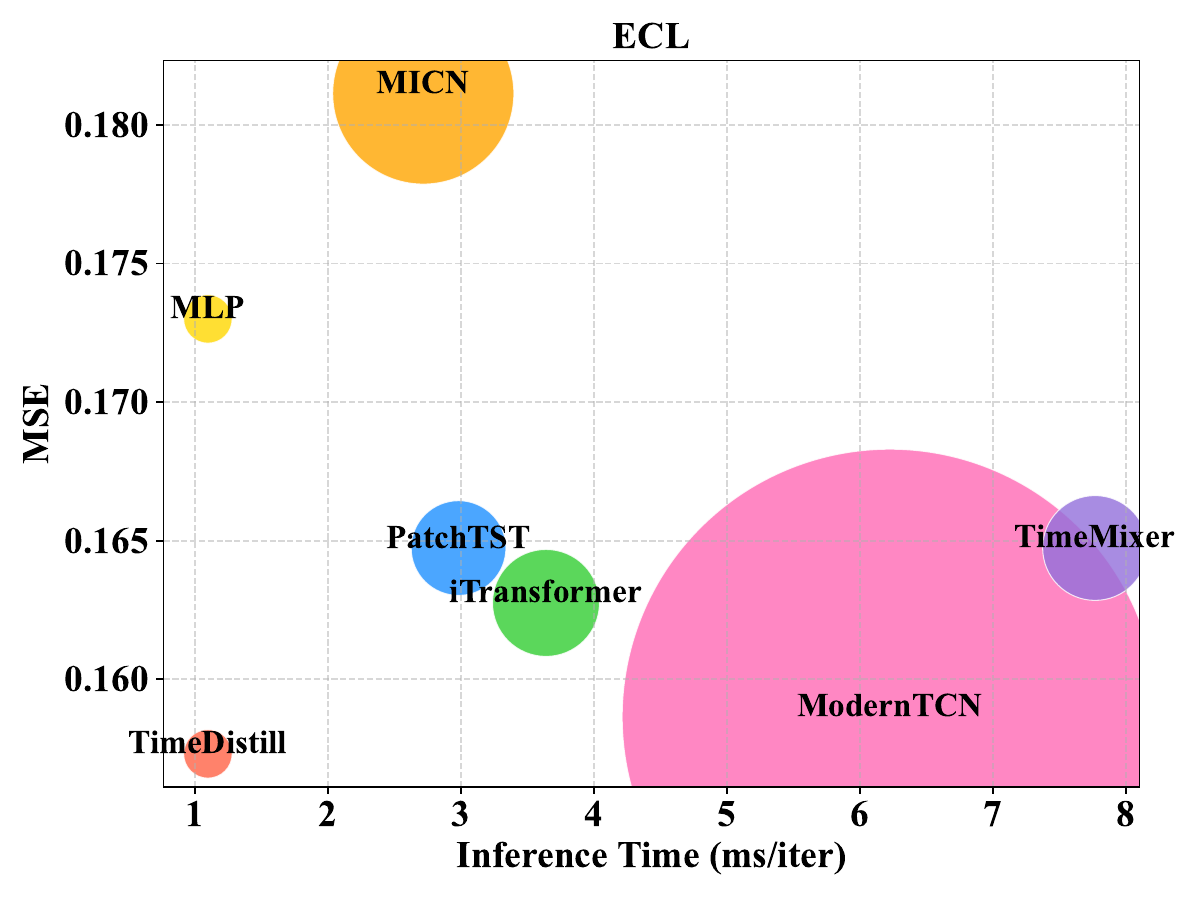}&
        \includegraphics[width=0.4\textwidth]{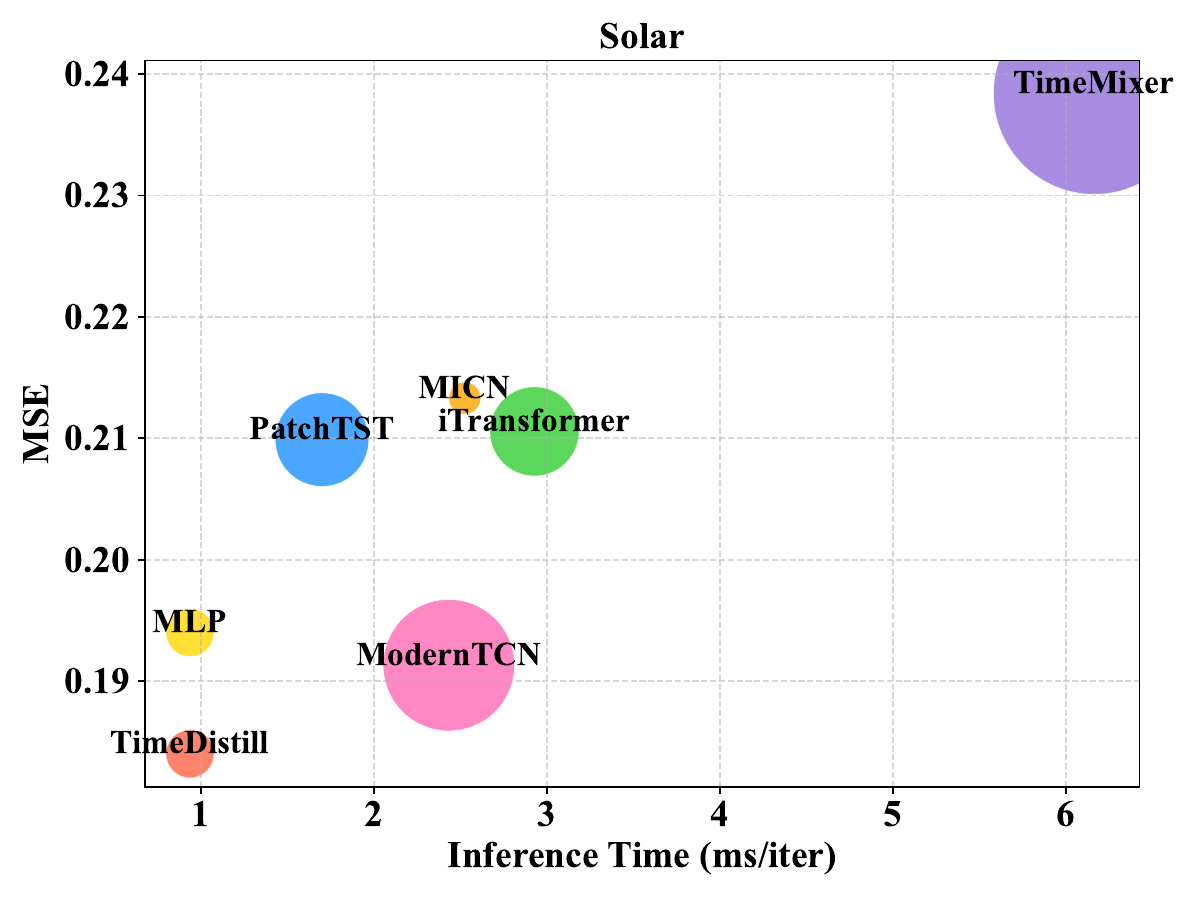} \\
        \includegraphics[width=0.4\textwidth]{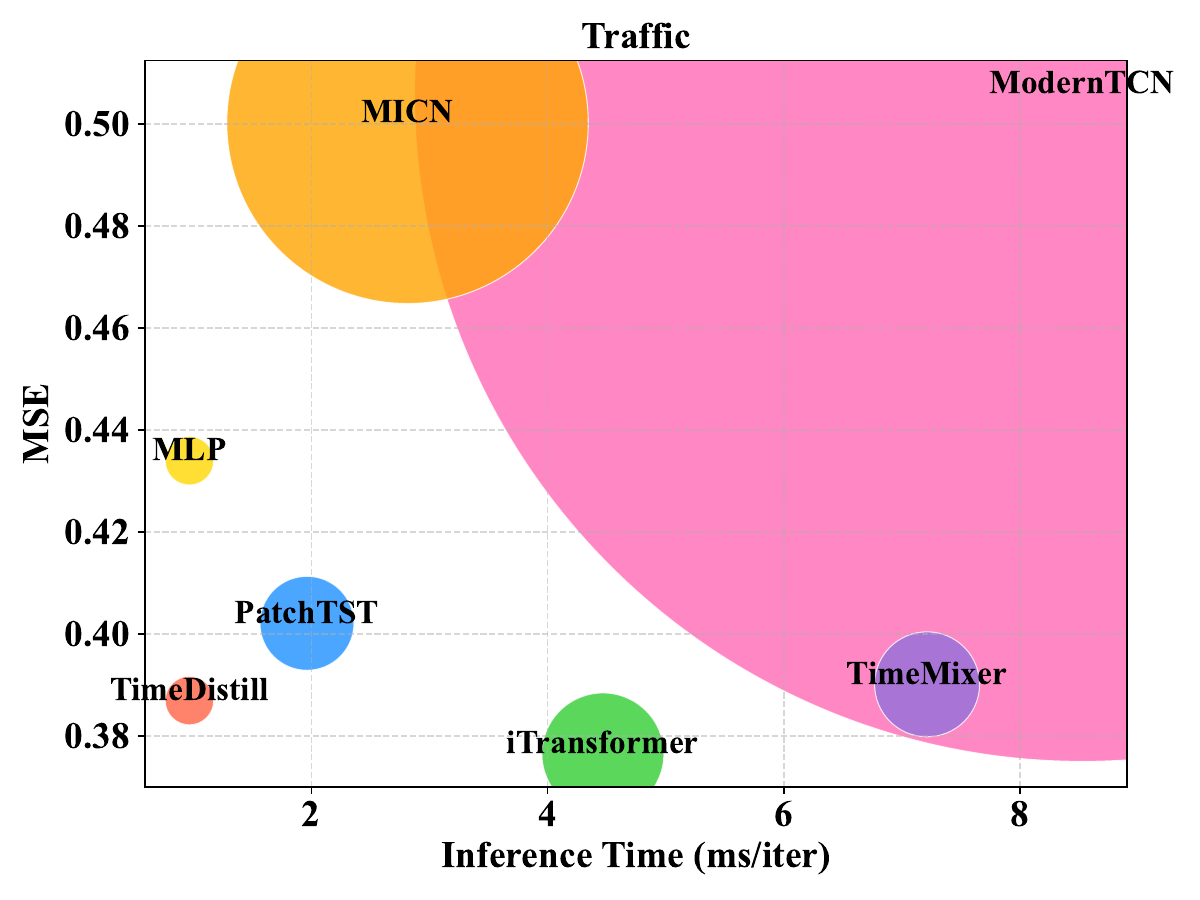}&
        \includegraphics[width=0.4\textwidth]{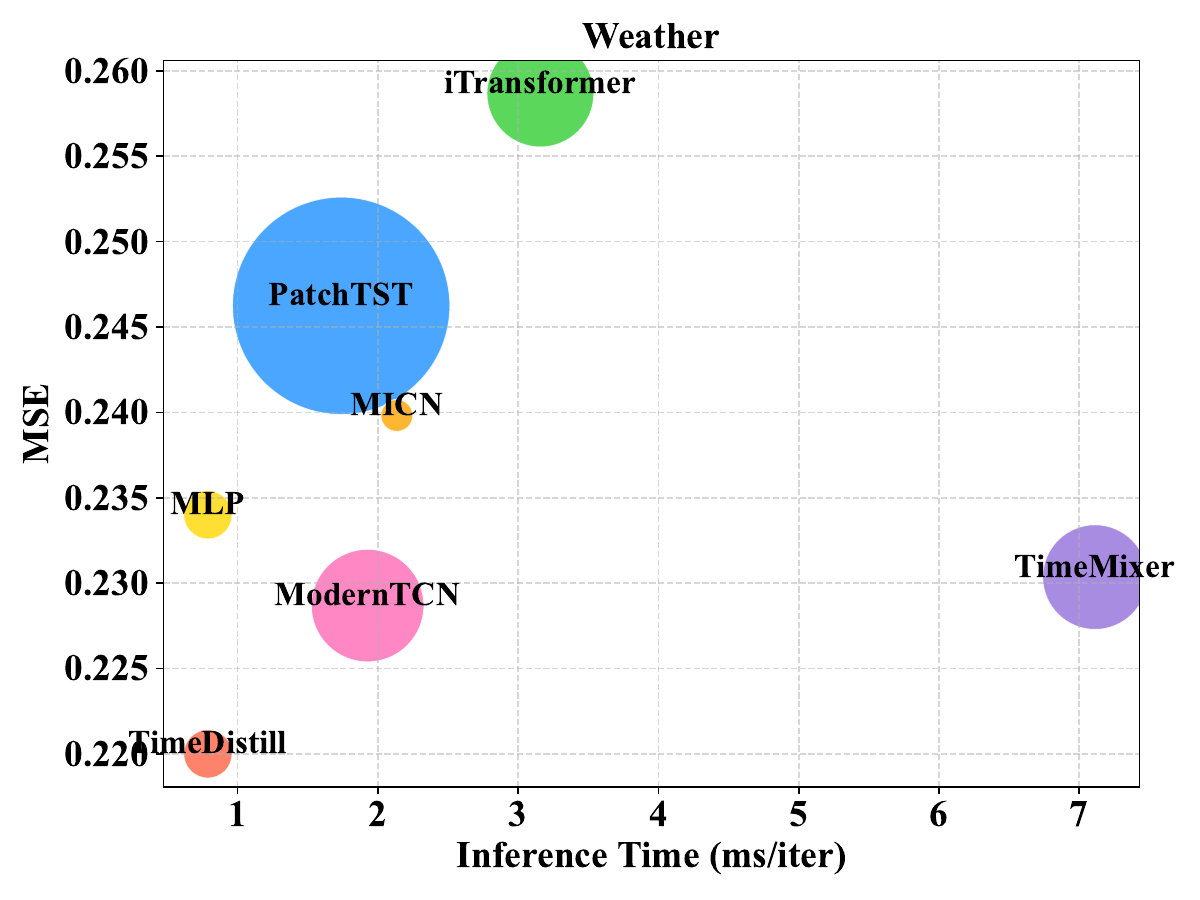} \\
    \end{tabular}
    \caption{Model efficiency comparison averaged across all prediction lengths (96, 192, 336, 720) on eight datasets. The size of each bubble represents the number of parameters in the corresponding method, with larger bubbles indicating more parameters.}
    \label{fig:app_efficiency}
\end{figure}

\clearpage
\section{Full Main Results}
\label{app:full_main_results}

\begin{table*}[h]
\centering
\caption{Long-term time series forecasting results with prediction lengths $S~\in\{96, 192, 336, 720\}$. A lower MSE or MAE indicates a better prediction. For consistency, we maintain a fixed input length of 720 throughout all the experiments. The best performance is highlighted in \textcolor{red}{\textbf{red}}, and the second-best is \textcolor{blue}{\underline{underlined}}.}
\resizebox{\textwidth}{!}{%
\begin{tabular}{c c lr|lr|lr|lr|lr|lr|lr|lr|lr}
\toprule
\multicolumn{2}{c}{Models} & \multicolumn{2}{c}{\textbf{\method{}}} & \multicolumn{2}{c}{iTransformer} & \multicolumn{2}{c}{ModernTCN} & \multicolumn{2}{c}{TimeMixer} & \multicolumn{2}{c}{PatchTST} & \multicolumn{2}{c}{MICN} & \multicolumn{2}{c}{FEDformer} & \multicolumn{2}{c}{TimesNet} & \multicolumn{2}{c}{Autoformer}\\
 & & \multicolumn{2}{c}{(\textbf{Ours})} & \multicolumn{2}{c}{(\citeyear{itransformer})} & \multicolumn{2}{c}{(\citeyear{moderntcn})} & \multicolumn{2}{c}{(\citeyear{timemixer})} &\multicolumn{2}{c}{(\citeyear{patchtst})} & \multicolumn{2}{c}{(\citeyear{micn})} & \multicolumn{2}{c}{(\citeyear{fedformer})} & \multicolumn{2}{c}{(\citeyear{timesnet})} &  \multicolumn{2}{c}{(\citeyear{autoformer})} \\
\cmidrule(lr){3-4} \cmidrule(lr){5-6} \cmidrule(lr){7-8} \cmidrule(lr){9-10} \cmidrule(lr){11-12} \cmidrule(lr){13-14} \cmidrule(lr){15-16} \cmidrule(lr){17-18} \cmidrule(lr){19-20}
 \multicolumn{2}{c}{Metric}& MSE & MAE & MSE & MAE & MSE & MAE & MSE & MAE & MSE & MAE & MSE & MAE & MSE & MAE & MSE & MAE & MSE & MAE \\
\midrule
\multirow{5}{*}{\rotatebox{90}{ECL}} 
& 96  & \textcolor{red}{\textbf{0.128}} & \textcolor{red}{\textbf{0.225}} & 0.134 & 0.230 & 0.140	& 0.239 & \textcolor{blue}{\underline{0.132}} & \textcolor{blue}{\underline{0.227}} & 0.135 & 0.239 & 0.169 & 0.283 & 0.238 & 0.347 & 0.218 & 0.324 & 0.224 & 0.334 \\
& 192 & \textcolor{red}{\textbf{0.145}} & \textcolor{red}{\textbf{0.241}} & 0.153 & \textcolor{blue}{\underline{0.248}} & \textcolor{blue}{\underline{0.153}} & 0.250 & 0.156 & 0.248 & 0.152 & 0.255 & 0.171 & 0.283 & 0.239 & 0.349 & 0.228 & 0.331 & 0.227 & 0.339 \\
& 336 & \textcolor{red}{\textbf{0.161}} & \textcolor{red}{\textbf{0.258}} & 0.171 & 0.268 & \textcolor{blue}{\underline{0.168}} & 0.264 & 0.170 & \textcolor{blue}{\underline{0.264}} & 0.168 & 0.270 & 0.184 & 0.296 & 0.255 & 0.361 & 0.240 & 0.344 & 0.235 & 0.344 \\
& 720 & \textcolor{red}{\textbf{0.195}} & \textcolor{blue}{\underline{0.291}} & 0.196 & \textcolor{red}{\textbf{0.290}} & 0.206 & 0.296 & 0.201 & 0.295 & 0.203 & 0.301 & 0.201 & 0.312 & 0.363 & 0.447 & 0.314 & 0.390 & 0.267 & 0.371 \\
\cmidrule(lr){2-20}
& Avg & \textcolor{red}{\textbf{0.157}} & \textcolor{red}{\textbf{0.254}} & \textcolor{blue}{\underline{0.163}} & \textcolor{blue}{\underline{0.259}} & 0.167 & 0.262 & 0.165 & 0.259 & 0.165 & 0.266 & 0.181 & 0.293 & 0.274 & 0.376 & 0.250 & 0.347 & 0.238 & 0.347 \\
\midrule

\multirow{5}{*}{\rotatebox{90}{ETTh1}} 
& 96  & \textcolor{red}{\textbf{0.373}} & \textcolor{red}{\textbf{0.401}} & 0.392 & 0.423 & \textcolor{blue}{\underline{0.389}} & \textcolor{blue}{\underline{0.412}} & 0.389 & 0.417 & 0.428 & 0.438 & 0.746 & 0.642 & 0.486 & 0.502 & 0.451 & 0.461 & 0.590 & 0.564 \\
& 192 & \textcolor{red}{\textbf{0.411}} & \textcolor{red}{\textbf{0.426}} & \textcolor{blue}{\underline{0.428}} & \textcolor{blue}{\underline{0.448}} & 0.446 & 0.452 & 0.443 & 0.460 & 0.476 & 0.476 & 0.523 & 0.502 & 0.483 & 0.501 & 0.456 & 0.469 & 0.586 & 0.578 \\
& 336 & \textcolor{red}{\textbf{0.439}} & \textcolor{red}{\textbf{0.444}} & \textcolor{blue}{\underline{0.461}} & 0.473 & 0.482 & \textcolor{blue}{\underline{0.469}} & 0.521 & 0.502 & 0.543 & 0.518 & 0.750 & 0.647 & 0.494 & 0.501 & 0.494 & 0.493 & 0.810 & 0.700 \\
& 720 & \textcolor{blue}{\underline{0.495}} & \textcolor{blue}{\underline{0.493}} & 0.590 & 0.562 & 0.557 & 0.528 & \textcolor{red}{\textbf{0.483}} & \textcolor{red}{\textbf{0.481}} & 0.546 & 0.526 & 0.939 & 0.735 & 0.644 & 0.594 & 0.629 & 0.575 & 0.941 & 0.797 \\ 
\cmidrule(lr){2-20}
& Avg & \textcolor{red}{\textbf{0.429}} & \textcolor{red}{\textbf{0.441}} & 0.468 & 0.476 & 0.469 & 0.465 & \textcolor{blue}{\underline{0.459}} & \textcolor{blue}{\underline{0.465}} & 0.498 & 0.490 & 0.739 & 0.631 & 0.527 & 0.524 & 0.507 & 0.500 & 0.731 & 0.660 \\
\midrule

\multirow{5}{*}{\rotatebox{90}{ETTh2}} 
& 96  & \textcolor{red}{\textbf{0.273}} & \textcolor{red}{\textbf{0.336}} & 0.303 & 0.364 & \textcolor{blue}{\underline{0.288}} & \textcolor{blue}{\underline{0.350}} & 0.323 & 0.384 & 0.338 & 0.386 & 0.395 & 0.427 & 0.410 & 0.457 & 0.415 & 0.446 & 1.173 & 0.824 \\
& 192 & \textcolor{red}{\textbf{0.334}} & \textcolor{red}{\textbf{0.381}} & 0.410 & 0.423 & \textcolor{blue}{\underline{0.341}} & \textcolor{blue}{\underline{0.387}} & 0.506 & 0.492 & 0.405 & 0.423 & 0.572 & 0.544 & 0.420 & 0.462 & 0.406 & 0.435 & 1.523 & 0.908 \\
& 336 & \textcolor{red}{\textbf{0.363}} & \textcolor{red}{\textbf{0.415}} & 0.440 & 0.450 & \textcolor{blue}{\underline{0.376}} & \textcolor{blue}{\underline{0.419}} & 0.391 & 0.427 & 0.488 & 0.469 & 1.387 & 0.896 & 0.443 & 0.481 & 0.411 & 0.443 & 2.492 & 1.217 \\
& 720 & \textcolor{red}{\textbf{0.408}} & \textcolor{red}{\textbf{0.446}} & 0.439 & 0.469 & \textcolor{blue}{\underline{0.424}} & \textcolor{blue}{\underline{0.456}} & 0.469 & 0.473 & 0.545 & 0.494 & 1.956 & 1.077 & 0.551 & 0.539 & 0.442 & 0.460 & 1.190 & 0.815 \\
\cmidrule(lr){2-20}
& Avg & \textcolor{red}{\textbf{0.345}} & \textcolor{red}{\textbf{0.395}} & 0.398 & 0.426 & \textcolor{blue}{\underline{0.357}} & \textcolor{blue}{\underline{0.403}} & 0.422 & 0.444 & 0.444 & 0.443 & 1.078 & 0.736 & 0.456 & 0.485 & 0.419 & 0.446 & 1.594 & 0.941 \\
\midrule

\multirow{5}{*}{\rotatebox{90}{ETTm1}}
& 96  & \textcolor{red}{\textbf{0.285}} & \textcolor{red}{\textbf{0.344}} & 0.319 & 0.367 & 0.325 & 0.369 & 0.309 & \textcolor{blue}{\underline{0.357}} & \textcolor{blue}{\underline{0.308}} & 0.368 & 0.356 & 0.404 & 0.363 & 0.422 & 0.333 & 0.374 & 0.475 & 0.485 \\
& 192 & \textcolor{red}{\textbf{0.331}} & \textcolor{red}{\textbf{0.368}} & 0.347 & 0.388 & 0.372 & 0.397 & \textcolor{blue}{\underline{0.339}} & \textcolor{blue}{\underline{0.373}} & 0.363 & 0.395 & 0.428 & 0.454 & 0.401 & 0.444 & 0.367 & 0.398 & 0.504 & 0.495 \\
& 336 & \textcolor{red}{\textbf{0.359}} & \textcolor{red}{\textbf{0.386}} & \textcolor{blue}{\underline{0.387}} & 0.413 & 0.408 & 0.424 & 0.390 & \textcolor{blue}{\underline{0.399}} & 0.414 & 0.430 & 0.465 & 0.483 & 0.422 & 0.447 & 0.417 & 0.429 & 0.670 & 0.559 \\
& 720 & \textcolor{red}{\textbf{0.415}} & \textcolor{red}{\textbf{0.416}} & 0.437 & 0.439 & 0.456 & 0.450 & \textcolor{blue}{\underline{0.429}} & \textcolor{blue}{\underline{0.423}} & 0.446 & 0.452 & 0.507 & 0.502 & 0.505 & 0.492 & 0.477 & 0.474 & 0.635 & 0.567 \\
\cmidrule(lr){2-20}
& Avg & \textcolor{red}{\textbf{0.348}} & \textcolor{red}{\textbf{0.379}} & 0.372 & 0.402 & 0.390 & 0.410 & \textcolor{blue}{\underline{0.367}} & \textcolor{blue}{\underline{0.388}} & 0.383 & 0.412 & 0.439 & 0.461 & 0.423 & 0.451 & 0.398 & 0.419 & 0.571 & 0.527 \\
\midrule

\multirow{5}{*}{\rotatebox{90}{ETTm2}}
& 96  & \textcolor{red}{\textbf{0.163}} & \textcolor{red}{\textbf{0.255}} & 0.180 & 0.273 & \textcolor{blue}{\underline{0.180}} & \textcolor{blue}{\underline{0.269}} & 0.197 & 0.292 & 0.181 & 0.273 & 0.215 & 0.311 & 0.298 & 0.362 & 0.202 & 0.289 & 0.309 & 0.374 \\
& 192 & \textcolor{red}{\textbf{0.220}} & \textcolor{red}{\textbf{0.294}} & 0.243 & 0.316 & 0.240 & 0.310 & 0.240 & \textcolor{blue}{\underline{0.307}} & \textcolor{blue}{\underline{0.230}} & 0.308 & 0.232 & 0.317 & 0.322 & 0.375 & 0.260 & 0.332 & 0.452 & 0.466 \\
& 336 & \textcolor{red}{\textbf{0.269}} & \textcolor{red}{\textbf{0.328}} & 0.299 & 0.352 & 0.288 & 0.344 & \textcolor{blue}{\underline{0.286}} & \textcolor{blue}{\underline{0.340}} & 0.306 & 0.355 & 0.378 & 0.436 & 0.374 & 0.413 & 0.321 & 0.370 & 0.373 & 0.416 \\
& 720 & \textcolor{red}{\textbf{0.326}} & \textcolor{red}{\textbf{0.369}} & 0.382 & 0.405 & \textcolor{blue}{\underline{0.363}} & \textcolor{blue}{\underline{0.395}} & 0.392 & 0.416 & 0.379 & 0.404 & 0.567 & 0.551 & 0.441 & 0.455 & 0.381 & 0.406 & 0.550 & 0.540 \\ 
\cmidrule(lr){2-20}
& Avg & \textcolor{red}{\textbf{0.244}} & \textcolor{red}{\textbf{0.311}} & 0.276 & 0.337 & \textcolor{blue}{\underline{0.267}} & \textcolor{blue}{\underline{0.330}} & 0.279 & 0.339 & 0.274 & 0.335 & 0.348 & 0.404 & 0.359 & 0.401 & 0.291 & 0.349 & 0.421 & 0.449 \\
\midrule

\multirow{5}{*}{\rotatebox{90}{Solar}}
& 96  & \textcolor{red}{\textbf{0.166}} & \textcolor{red}{\textbf{0.229}} & 0.189 & 0.255 & \textcolor{blue}{\underline{0.171}} & \textcolor{blue}{\underline{0.249}} & 0.283 & 0.382 & 0.178 & 0.232 & 0.193 & 0.259 & 0.275 & 0.367 & 0.174 & 0.251 & 0.880 & 0.654 \\
& 192 & \textcolor{red}{\textbf{0.181}} & \textcolor{red}{\textbf{0.239}} & 0.237 & 0.265 & \textcolor{blue}{\underline{0.188}} & \textcolor{blue}{\underline{0.236}} & 0.216 & 0.257 & 0.201 & 0.256 & 0.213 & 0.275 & 0.279 & 0.365 & 0.189 & 0.258 & 1.024 & 0.742 \\
& 336 & \textcolor{red}{\textbf{0.191}} & \textcolor{red}{\textbf{0.246}} & 0.214 & 0.280 & \textcolor{blue}{\underline{0.192}} & \textcolor{blue}{\underline{0.237}}  & 0.232 & 0.267 & 0.230 & 0.266 & 0.209 & 0.278 & 0.306 & 0.384 & 0.206 & 0.271 & 1.258 & 0.858 \\
& 720 & \textcolor{red}{\textbf{0.199}} & \textcolor{red}{\textbf{0.252}} & 0.217 & 0.280 & \textcolor{blue}{\underline{0.214}} & 0.251 & 0.223 & \textcolor{blue}{\underline{0.246}} & 0.231 & 0.275 & 0.238 & 0.297 & 0.341 & 0.416 & 0.214 & 0.267 & 0.987 & 0.718 \\ 
\cmidrule(lr){2-20}
& Avg & \textcolor{red}{\textbf{0.184}} & \textcolor{red}{\textbf{0.241}} & 0.214 & 0.270 & \textcolor{blue}{\underline{0.191}} & \textcolor{blue}{\underline{0.243}} & 0.238 & 0.288 & 0.210 & 0.257 & 0.213 & 0.277 & 0.300 & 0.383 & 0.196 & 0.262 & 1.037 & 0.743 \\
\midrule

\multirow{5}{*}{\rotatebox{90}{Traffic}}
& 96  & \textcolor{blue}{\underline{0.358}} & \textcolor{blue}{\underline{0.256}} &\textcolor{red}{\textbf{0.345}} & \textcolor{red}{\textbf{0.254}} & 0.392 & 0.276 & 0.367 & 0.272 & 0.374 & 0.272 & 0.485 & 0.313 & 0.608 & 0.388 & 0.677 & 0.391 & 0.691 & 0.422 \\
& 192 & \textcolor{blue}{\underline{0.374}} & \textcolor{red}{\textbf{0.264}} & \textcolor{red}{\textbf{0.366}} & 0.266 & 0.400 & 0.277 & 0.377 & \textcolor{blue}{\underline{0.265}} & 0.389 & 0.278 & 0.484 & 0.310 & 0.616 & 0.376 & 0.682 & 0.392 & 0.710 & 0.432 \\
& 336 & \textcolor{blue}{\underline{0.389}} & \textcolor{red}{\textbf{0.271}} & \textcolor{red}{\textbf{0.382}} & 0.272 & 0.412 & 0.283 & 0.389 & \textcolor{blue}{\underline{0.272}} & 0.403 & 0.285 & 0.493 & 0.311 & 0.639 & 0.392 & 0.695 & 0.401 & 0.700 & 0.433 \\
& 720 & \textcolor{blue}{\underline{0.428}} & \textcolor{red}{\textbf{0.292}} & \textcolor{red}{\textbf{0.422}} & \textcolor{blue}{\underline{0.292}} & 0.450 & 0.302 & 0.433 & 0.293 & 0.442 & 0.303 & 0.539 & 0.331 & 0.648 & 0.397 & 0.720 & 0.410 & 0.687 & 0.422 \\
\cmidrule(lr){2-20}
& Avg & \textcolor{blue}{\underline{0.387}} & \textcolor{red}{\textbf{0.271}} & \textcolor{red}{\textbf{0.379}} & \textcolor{blue}{\underline{0.271}} & 0.413 & 0.284 & \textcolor{blue}{\underline{0.391}} & 0.275 & 0.402 & 0.284 & 0.500 & 0.316 & 0.628 & 0.388 & 0.693 & 0.399 & 0.697 & 0.427 \\
\midrule

\multirow{5}{*}{\rotatebox{90}{Weather}}
& 96  & \textcolor{red}{\textbf{0.145}} & \textcolor{red}{\textbf{0.204}} & 0.178 & 0.229 & 0.152 & 0.208 & \textcolor{blue}{\underline{0.147}} & \textcolor{blue}{\underline{0.202}} & 0.158 & 0.214 & 0.171 & 0.231 & 0.333 & 0.395 & 0.174 & 0.234 & 0.372 & 0.419 \\
& 192 & \textcolor{red}{\textbf{0.188}} & \textcolor{blue}{\underline{0.247}} & 0.219 & 0.263 & \textcolor{blue}{\underline{0.198}} & 0.250 & 0.208 & 0.258 & 0.202 & 0.251 & 0.213 & 0.270 & 0.327 & 0.375 & 0.217 & 0.267 & 0.403 & 0.435 \\
& 336 & \textcolor{red}{\textbf{0.238}} & \textcolor{red}{\textbf{0.286}} & 0.280 & 0.307 & 0.257 & 0.294 & \textcolor{blue}{\underline{0.251}} & \textcolor{blue}{\underline{0.289}} & 0.281 & 0.312 & 0.259 & 0.310 & 0.360 & 0.399 & 0.276 & 0.311 & 0.702 & 0.579 \\
& 720 & \textcolor{red}{\textbf{0.310}} & \textcolor{red}{\textbf{0.338}} & 0.358 & 0.361 & 0.345 & 0.355 & \textcolor{blue}{\underline{0.316}} & \textcolor{blue}{\underline{0.334}} & 0.345 & 0.354 & 0.317 & 0.355 & 0.399 & 0.423 & 0.360 & 0.362 & 0.411 & 0.429 \\
\cmidrule(lr){2-20}
& Avg & \textcolor{red}{\textbf{0.220}} & \textcolor{red}{\textbf{0.269}} & 0.259 & 0.290 & 0.238 & 0.277 & \textcolor{blue}{\underline{0.230}} & \textcolor{blue}{\underline{0.271}} & 0.246 & 0.283 & 0.240 & 0.292 & 0.355 & 0.398 & 0.257 & 0.294 & 0.472 & 0.466 \\
\bottomrule
\end{tabular}%
}
\label{tab:full_results}
\end{table*}

\clearpage
\section{Full Different Look-back Window Results}
\label{app:full_diferent_lookback_window_results}

\begin{table*}[h]
\vspace{-1em}
\centering
\caption{Results with look-back length $S=96$.}
\vspace{-1em}
\resizebox{0.68\textwidth}{!}{
\begin{tabular}{l|cc||cc|cc||cc|cc}
\toprule
\multirow{2}{*}{Dataset} & \multicolumn{2}{c||}{MLP} & \multicolumn{2}{c|}{\tiny \method{}  (iTrans)} & \multicolumn{2}{c||}{iTransformer} & \multicolumn{2}{c|}{\tiny \method{}(Modern)} & \multicolumn{2}{c}{ModernTCN} \\ 
& MSE & MAE & MSE & MAE & MSE & MAE & MSE & MAE & MSE & MAE \\
\midrule
ECL     & 0.211 & 0.302 & 0.188 & 0.277 & \textbf{0.180} & \textbf{0.270} & \textbf{0.188} & \textbf{0.277} & 0.197 & 0.282 \\
ETTh1   & 0.499 & 0.481 & \textbf{0.444} & \textbf{0.439} & 0.453 & 0.448 & \textbf{0.443} & \textbf{0.433} & 0.446 & 0.433 \\
ETTh2   & 0.634 & 0.557 & \textbf{0.360} & \textbf{0.392} & 0.383 & 0.407 & \textbf{0.365} & \textbf{0.396} & 0.385 & 0.406 \\
ETTm1   & 0.400 & 0.412 & \textbf{0.376} & \textbf{0.395} & 0.407 & 0.411 &\textbf{ 0.376} & \textbf{0.395} & 0.386 & 0.400 \\
ETTm2   & 0.434 & 0.442 & \textbf{0.275} & \textbf{0.319} & 0.291 & 0.334 & \textbf{0.274} & \textbf{0.319} & 0.279 & 0.322 \\
Solar   & 0.263 & 0.321 & \textbf{0.234} & 0.289 & 0.236 & \textbf{0.262} & \textbf{0.235} & 0.293 & 0.255 & \textbf{0.278} \\
Traffic & 0.583 & 0.379 & 0.488 & 0.313 & \textbf{0.421} & \textbf{0.282} & \textbf{0.499} & \textbf{0.316} & 0.645 & 0.395 \\
Weather & 0.246 & 0.300 & \textbf{0.245} & 0.281 & 0.260 & \textbf{0.280} & \textbf{0.248} & 0.283 & 0.253 & \textbf{0.280} \\
\midrule
Avg     & 0.409 & 0.399 & \textbf{0.326} & 0.338 & 0.329 & \textbf{0.337} & \textbf{0.328} & \textbf{0.339} & 0.356 & 0.350 \\
\bottomrule
\end{tabular}
}
\vspace{-1em}
\end{table*}

\begin{table*}[h!]
\vspace{-1em}
\centering
\caption{Results with look-back length $S=192$.}
\vspace{-1em}
\resizebox{0.68\textwidth}{!}{
\begin{tabular}{l|cc||cc|cc||cc|cc}
\toprule
\multirow{2}{*}{Dataset} & \multicolumn{2}{c||}{MLP} & \multicolumn{2}{c|}{\tiny \method{}  (iTrans)} & \multicolumn{2}{c||}{iTransformer} & \multicolumn{2}{c|}{\tiny \method{}(Modern)} & \multicolumn{2}{c}{ModernTCN} \\ 
& MSE & MAE & MSE & MAE & MSE & MAE & MSE & MAE & MSE & MAE \\
\midrule
ECL     & 0.184 & 0.284 & 0.166 & 0.259 & \textbf{0.164} & \textbf{0.258} & \textbf{0.166} & \textbf{0.259} & 0.171 & 0.263 \\
ETTh1   & 0.491 & 0.484 & \textbf{0.447} & \textbf{0.443} & 0.454 & 0.453 & 0.432 & 0.430 & \textbf{0.429} & \textbf{0.427} \\
ETTh2   & 0.567 & 0.522 & \textbf{0.350} & \textbf{0.391} & 0.384 & 0.409 & \textbf{0.352} & \textbf{0.392} & 0.372 & 0.404 \\
ETTm1   & 0.369 & 0.392 & \textbf{0.346} & \textbf{0.380} & 0.373 & 0.396 & \textbf{0.343} & \textbf{0.376} & 0.365 & 0.386 \\
ETTm2   & 0.395 & 0.424 & \textbf{0.264} & \textbf{0.315} & 0.286 & 0.335 &\textbf{ 0.260} & \textbf{0.314} & 0.265 & 0.318 \\
Solar   & 0.213 & 0.273 & \textbf{0.199} & \textbf{0.258} & 0.224 & 0.260 & \textbf{0.199} & 0.257 & 0.217 & \textbf{0.253} \\
Traffic & 0.478 & 0.338 & 0.422 & 0.284 & \textbf{0.395} & \textbf{0.274} & \textbf{0.426} & \textbf{0.286} & 0.485 & 0.325 \\
Weather & 0.233 & 0.288 & \textbf{0.230} & \textbf{0.272} & 0.246 & 0.275 & \textbf{0.231} & \textbf{0.273} & 0.240 & 0.273 \\
\midrule
Avg     & 0.366 & 0.376 & \textbf{0.303} & \textbf{0.325} & 0.316 & 0.333 & \textbf{0.301} & \textbf{0.323} & 0.318 & 0.331 \\
\bottomrule
\end{tabular}
}
\vspace{-1em}
\end{table*}

\begin{table*}[h!]
\vspace{-1em}
\centering
\caption{Results with look-back length $S=336$.}
\vspace{-1em}
\resizebox{0.68\textwidth}{!}{
\begin{tabular}{l|cc||cc|cc||cc|cc}
\toprule
\multirow{2}{*}{Dataset} & \multicolumn{2}{c||}{MLP} & \multicolumn{2}{c|}{\tiny \method{}  (iTrans)} & \multicolumn{2}{c||}{iTransformer} & \multicolumn{2}{c|}{\tiny \method{}(Modern)} & \multicolumn{2}{c}{ModernTCN} \\ 
& MSE & MAE & MSE & MAE & MSE & MAE & MSE & MAE & MSE & MAE \\
\midrule
ECL     & 0.177 & 0.279 &\textbf{ 0.161} & \textbf{0.256} & 0.164 & 0.258 & \textbf{0.161} & \textbf{0.255} & 0.166 & 0.260 \\
ETTh1   & 0.481 & 0.487 & \textbf{0.429} & \textbf{0.439} & 0.458 & 0.460 & \textbf{0.417} & \textbf{0.427} & 0.418 & 0.427 \\
ETTh2   & 0.555 & 0.521 & \textbf{0.345} & \textbf{0.390} & 0.390 & 0.416 & \textbf{0.345} & \textbf{0.392} & 0.351 & 0.397 \\
ETTm1   & 0.364 & 0.390 & \textbf{0.343} & \textbf{0.380} & 0.368 & 0.395 & \textbf{0.340} & \textbf{0.374} & 0.363 & 0.385 \\
ETTm2   & 0.387 & 0.416 & \textbf{0.254} & \textbf{0.313} & 0.272 & 0.329 & \textbf{0.252} & \textbf{0.311} & 0.266 & 0.321 \\
Solar   & 0.203 & 0.260 & \textbf{0.194} & \textbf{0.248} & 0.235 & 0.270 & \textbf{0.192} & \textbf{0.248} & 0.212 & 0.250 \\
Traffic & 0.450 & 0.326 & 0.402 & 0.277 & \textbf{0.386} & \textbf{0.273} & \textbf{0.407} & \textbf{0.278} & 0.444 & 0.305 \\
Weather & 0.226 & 0.282 & \textbf{0.222} & \textbf{0.264} & 0.239 & 0.273 & \textbf{0.223} & \textbf{0.267} & 0.232 & 0.269 \\
\midrule
Avg     & 0.355 & 0.370 & \textbf{0.294} & \textbf{0.321} & 0.314 & 0.334 & \textbf{0.292} & \textbf{0.319} & 0.307 & 0.327 \\
\bottomrule
\end{tabular}
}
\vspace{-1em}
\end{table*}

\begin{table*}[h!]
\centering
\vspace{-1em}
\caption{Results with look-back length $S=720$.}
\vspace{-1em}
\resizebox{0.68\textwidth}{!}{
\begin{tabular}{l|cc||cc|cc||cc|cc}
\toprule
\multirow{2}{*}{Dataset} & \multicolumn{2}{c||}{MLP} & \multicolumn{2}{c|}{\tiny \method{}  (iTrans)} & \multicolumn{2}{c||}{iTransformer} & \multicolumn{2}{c|}{\tiny \method{}(Modern)} & \multicolumn{2}{c}{ModernTCN} \\ 
& MSE & MAE & MSE & MAE & MSE & MAE & MSE & MAE & MSE & MAE \\
\midrule
ECL     & 0.173 & 0.276 & \textbf{0.157} & \textbf{0.254} & 0.163 & 0.259 & \textbf{0.157} & \textbf{0.254} & 0.167& 0.262\\
ETTh1   & 0.502 & 0.489 & \textbf{0.428} & \textbf{0.445} & 0.468 & 0.476 & \textbf{0.429} &\textbf{ 0.441} & 0.469& 0.465\\
ETTh2   & 0.393 & 0.438 & \textbf{0.345} & \textbf{0.397} & 0.398 & 0.426 & \textbf{0.345} & \textbf{0.395} & 0.357& 0.403\\
ETTm1   & 0.391 & 0.413 & \textbf{0.354} & \textbf{0.390} & 0.372 & 0.402 & \textbf{0.348} & \textbf{0.379} & 0.390& 0.410\\
ETTm2   & 0.300 & 0.373 & \textbf{0.252} & \textbf{0.316} & 0.276 & 0.337 & \textbf{0.244} & \textbf{0.311} & 0.267& 0.330\\
Solar   & 0.194 & 0.255 & \textbf{0.185} & \textbf{0.241} & 0.214 & 0.270 & \textbf{0.184} & \textbf{0.241} & 0.191& 0.243\\
Traffic & 0.434 & 0.318 & 0.389 & \textbf{0.271} & \textbf{0.379} & 0.271 & \textbf{0.387}	&\textbf{ 0.271} & 0.413& 0.284\\
Weather & 0.234 & 0.294 & \textbf{0.220} & \textbf{0.270} & 0.259 & 0.290 & \textbf{0.220} & \textbf{0.269} & 0.238& 0.277\\
\midrule
Avg     & 0.328 & 0.357 & \textbf{0.291} & \textbf{0.323} & 0.316 & 0.341 & \textbf{0.289}	& \textbf{0.320} & 0.312& 0.334\\
\bottomrule
\end{tabular}
}
\vspace{-1em}
\end{table*}

\paragraph{\textbf{Discussion on although ModernTCN and iTransformer differ in performance, their distilled student models perform very consistently.}} There is an interesting observation from the above four tables that although ModernTCN and iTransformer have different performances, the two student models distilled from them as teacher models perform highly consistently. The differences are clear on Traffic, ETTh1/2, Solar, and Weather, but less so on ECL and ETTm1/m2, where students distilled from ModernTCN and iTransformer perform similarly. We attribute this to the limited capacity of the student MLP, leading to performance saturation. Table~\ref{tab:mlp_mse_mae} supports this: as student size increases, performance differences on ETTm1 become more apparent, confirming our hypothesis.

\begin{table}[ht]
\centering
\caption{Performance comparison (MSE and MAE) for different MLP sizes on ModernTCN and iTransformer.}
\resizebox{0.99\textwidth}{!}{%
\begin{tabular}{lcccccc}
\hline
Method & MLP(2L512)MSE & MLP(2L512)MAE & MLP(2L1024)MSE & MLP(2L1024)MAE & MLP(2L2048)MSE & MLP(2L2048)MAE \\
\hline
ModernTCN      & 0.376 & 0.395 & 0.374 & 0.393 & 0.371 & 0.382 \\
iTransformer   & 0.376 & 0.395 & 0.376 & 0.394 & 0.374 & 0.389 \\
\hline
\end{tabular}
}
\label{tab:mlp_mse_mae}
\end{table}

\section{Full Ablation Results}
\label{app:full_ablation_results}

\begin{table}[htbp]
\centering
\caption{Ablation study on different datasets (Teacher: \textbf{iTransformer}~\cite{itransformer}).}
\label{app:tab_ablation_itran}
\vspace{-1em}
\small
\resizebox{0.99\textwidth}{!}{
\begin{tabular}{l cc|cc|cc|cc|cc|cc|cc|cc}
\toprule
\multirow{2}{*}{Method} & \multicolumn{2}{c}{ECL} & \multicolumn{2}{c}{ETTh1} & \multicolumn{2}{c}{ETTh2} & \multicolumn{2}{c}{ETTm1} & \multicolumn{2}{c}{ETTm2} & \multicolumn{2}{c}{Solar} & \multicolumn{2}{c}{Traffic} & \multicolumn{2}{c}{Weather}\\
\cmidrule(lr){2-3}\cmidrule(lr){4-5}\cmidrule(lr){6-7}\cmidrule(lr){8-9}\cmidrule(lr){10-11}\cmidrule(lr){12-13}\cmidrule(lr){14-15}\cmidrule(lr){16-17}
 & MSE & MAE & MSE & MAE & MSE & MAE & MSE & MAE & MSE & MAE & MSE & MAE & MSE & MAE & MSE & MAE\\
\midrule
iTransformer        & 0.163 & 0.259 & 0.468 & 0.476 & 0.398 & 0.426 & 0.372 & 0.402 & 0.276 & 0.337 & 0.214 & 0.270 & \textbf{0.379} & 0.271 & 0.259 & 0.290 \\
MLP                & 0.173 & 0.276 & 0.502 & 0.489 & 0.393 & 0.438 & 0.391 & 0.413 & 0.300 & 0.373 & 0.194 & 0.255 & 0.434 & 0.318 & 0.234 & 0.294 \\
\midrule
\midrule
\method{}        & \textbf{0.157} & \textbf{0.254} & \textbf{0.428} & \textbf{0.445} & \textbf{0.345} & \textbf{0.397} & \textbf{0.354} & \textbf{0.390} & \textbf{0.252} & \textbf{0.316} & \textbf{0.185} & \textbf{0.241} & 0.389 & \textbf{0.271} & \textbf{0.220} & \textbf{0.270} \\
w/o prediction level & 0.157 & 0.254 & 0.480 & 0.472 & 0.365 & 0.413 & 0.372 & 0.403 & 0.261 & 0.321 & 0.186 & 0.242 & 0.392 & 0.274 & 0.221 & 0.271 \\
w/o feature level    & 0.163 & 0.260 & 0.441 & 0.452 & 0.365 & 0.410 & 0.373 & 0.398 & 0.258 & 0.320 & 0.186 & 0.246 & 0.393 & 0.277 & 0.225 & 0.277 \\
w/o multi-scale      & 0.163 & 0.261 & 0.483 & 0.480 & 0.375 & 0.423 & 0.394 & 0.409 & 0.268 & 0.327 & 0.187 & 0.248 & 0.393 & 0.277 & 0.223 & 0.277 \\
w/o multi-period     & 0.159 & 0.255 & 0.507 & 0.487 & 0.376 & 0.422 & 0.381 & 0.399 & 0.268 & 0.323 & 0.195 & 0.256 & 0.392 & 0.273 & 0.222 & 0.270 \\
w/o sup              & 0.161 & 0.258 & 0.425 & 0.446 & 0.353 & 0.396 & 0.368 & 0.396 & 0.257 & 0.320 & 0.205 & 0.269 & 0.394 & 0.277 & 0.223 & 0.270 \\
\hline
\end{tabular}
}
\end{table}

\begin{table}[htbp]
\centering
\vspace{-1em}
\caption{Ablation study on different datasets (Teacher: \textbf{ModernTCN}~\cite{moderntcn}).}
\label{app:tab_ablation_modern}
\vspace{-1em}
\small
\resizebox{0.99\textwidth}{!}{
\begin{tabular}{l cc|cc|cc|cc|cc|cc|cc|cc}
\toprule
\multirow{2}{*}{Method} & \multicolumn{2}{c}{ECL} & \multicolumn{2}{c}{ETTh1} & \multicolumn{2}{c}{ETTh2} & \multicolumn{2}{c}{ETTm1} & \multicolumn{2}{c}{ETTm2} & \multicolumn{2}{c}{Solar} & \multicolumn{2}{c}{Traffic} & \multicolumn{2}{c}{Weather}\\
\cmidrule(lr){2-3}\cmidrule(lr){4-5}\cmidrule(lr){6-7}\cmidrule(lr){8-9}\cmidrule(lr){10-11}\cmidrule(lr){12-13}\cmidrule(lr){14-15}\cmidrule(lr){16-17}
 & MSE & MAE & MSE & MAE & MSE & MAE & MSE & MAE & MSE & MAE & MSE & MAE & MSE & MAE & MSE & MAE \\
\midrule
ModernTCN & 0.167& 0.262& 0.469& 0.465& 0.357& 0.403& 0.390& 0.410& 0.267 & 0.330 & 0.191& 0.243& 0.413& 0.284& 0.238& 0.277\\
MLP & 0.173 & 0.276 & 0.502 & 0.489 & 0.393 & 0.438 & 0.391 & 0.413 & 0.300 & 0.373 & 0.194 & 0.255 & 0.434 & 0.318 & 0.234 & 0.294 \\
\midrule
\midrule
\method{} & \textbf{0.157} & \textbf{0.254} & \textbf{0.429} & \textbf{0.441} & \textbf{0.345} & \textbf{0.395} & \textbf{0.348} & \textbf{0.379} & \textbf{0.244} & \textbf{0.311} & \textbf{0.184} & \textbf{0.241} & \textbf{0.391} & \textbf{0.275} & \textbf{0.220} & \textbf{0.269} \\
w/o prediction level & 0.157 & 0.254 & 0.490 & 0.480 & 0.370 & 0.419 & 0.370 & 0.401 & 0.263 & 0.325 & 0.184 & 0.241 & 0.392 & 0.275 & 0.221 & 0.273 \\
w/o feature level & 0.161 & 0.258 & 0.442 & 0.447 & 0.354 & 0.401 & 0.353 & 0.382 & 0.248 & 0.315 & 0.188 & 0.250 & 0.393 & 0.277 & 0.224 & 0.271 \\
w/o multi-scale & 0.162 & 0.260 & 0.480 & 0.476 & 0.380 & 0.430 & 0.379 & 0.402 & 0.267 & 0.327 & 0.187 & 0.249 & 0.393 & 0.278 & 0.224 & 0.277 \\
w/o multi-period & 0.157 & 0.254 & 0.430 & 0.442 & 0.346 & 0.395 & 0.348 & 0.379 & 0.245 & 0.311 & 0.184 & 0.241 & 0.391 & 0.274 & 0.221 & 0.267 \\
w/o sup & 0.165 & 0.261 & 0.423 & 0.438 & 0.345 & 0.394 & 0.356 & 0.381 & 0.251 & 0.317 & 0.192 & 0.260 & 0.506 & 0.351 & 0.225 & 0.269 \\
\bottomrule
\end{tabular}
}
\end{table}

\begin{table*}[h!]
\centering
\vspace{-1em}
\caption{Ablation study measured by MSE on different components of \method{}. Teacher is \textbf{iTransformer};}
\vspace{-1em}
\label{tab:side_by_side_ablation1}
\centering
\resizebox{0.6\textwidth}{!}{
\begin{tabular}{l c c c c c}
\toprule
Method & ECL & ETT(avg) & Solar & Traffic & Weather \\
\midrule
\textbf{iTransformer}        & 0.163 & 0.379 & 0.214 & \textcolor{red}{\textbf{0.379}} & 0.259 \\
MLP                & 0.173 & 0.396 & 0.194 & 0.434 & 0.234 \\
\midrule
\method{}           & \textcolor{red}{\textbf{0.157}} & \textcolor{red}{\textbf{0.345}} & \textcolor{red}{\textbf{0.185}} & \textcolor{blue}{\underline{0.389}} & \textcolor{red}{\textbf{0.220}} \\
\textit{w/o} multi-scale      & 0.163 & 0.380 & 0.187 & 0.393 & 0.223 \\
\textit{w/o} multi-period     & 0.159 & 0.383 & 0.195 & 0.392 & 0.222 \\
\textit{w/o} prediction level & \textcolor{blue}{\underline{0.157}} & 0.370 & 0.186 & 0.392 & \textcolor{blue}{\underline{0.221}} \\
\textit{w/o} feature level    & 0.163 & 0.359 & \textcolor{blue}{\underline{0.186}} & 0.393 & 0.225 \\
\textit{w/o} sup              & 0.161 & \textcolor{blue}{\underline{0.351}} & 0.205 & 0.394 & 0.223 \\
\bottomrule
\end{tabular}
}
\end{table*}

\begin{table*}[h!]
\centering
\vspace{-1em}
\caption{Ablation study measured by MSE on different components of \method{}. Teacher is \textbf{ModernTCN}.}
\vspace{-1em}
\label{tab:side_by_side_ablation2}
\centering
\resizebox{0.6\textwidth}{!}{
\begin{tabular}{l c c c c c}
\toprule
Method & ECL & ETT(avg) & Solar & Traffic & Weather \\
\midrule
Teacher            & 0.167 & 0.371 & 0.191 & 0.413 & 0.238 \\
MLP                 & 0.173 & 0.397 & 0.194 & 0.434 & 0.234 \\
\midrule
\method{}         & \textcolor{red}{\textbf{0.157}} & \textcolor{red}{\textbf{0.342}} & \textcolor{red}{\textbf{0.184}} & \textcolor{red}{\textbf{0.387}} & \textcolor{red}{\textbf{0.220}} \\
w/o prediction level & \textcolor{blue}{\underline{0.157}} & 0.373 & \textcolor{blue}{\underline{0.184}} & 0.392 & \textcolor{blue}{\underline{0.221}} \\
w/o feature level    & 0.161 & 0.349 & 0.188 & 0.393 & 0.224 \\
w/o multi-scale      & 0.162 & 0.377 & 0.187 & 0.393 & 0.224 \\
w/o multi-period     & 0.157 & \textcolor{blue}{\underline{0.342}} &0.184 & \textcolor{blue}{\underline{0.391}} & 0.221 \\
w/o sup              & 0.165 & 0.344 & 0.192 & 0.506 & 0.225 \\
\bottomrule
\end{tabular}
}
\end{table*}

\paragraph{\textbf{Discussion on removing multi-period loss can still yield acceptable results.}}
We observe that removing multi-period distillation causes a small drop when teacher is ModernTCN, but a large drop with iTransformer (Table~\ref{app:tab_ablation_itran} and Table~\ref{tab:side_by_side_ablation1}).

Theorem~\ref{thm:multiperiod} shows that multi-period loss optimizes a mixup distribution $\mathbf{Q}_y + \lambda\,\mathbf{Q}_t$. When $\mathbf{Q}_t$ is already close to 
$\mathbf{Q}_y$, the mixup offers limited gain. We confirm this on ETTh1 by computing FFT-based period distributions and KL divergence. As shown in Table~\ref{tab:comparison_of_kl}, ModernTCN's $\mathbf{Q}_t$ is closer to $\mathbf{Q}_y$ than iTransformer's, supporting the hypothesis.

\begin{table}[h]
\centering
\caption{Comparison of KL values}
\label{tab:comparison_of_kl}
\begin{tabular}{lcc}
\hline
 & iTransformer & ModernTCN \\
\hline
KL & 0.4606 & 0.3843 \\
\hline
\end{tabular}
\end{table}

\section{Show Cases}
\label{app:show_cases}

\begin{figure*}[h!]
    \centering
    \includegraphics[width=0.99\textwidth]{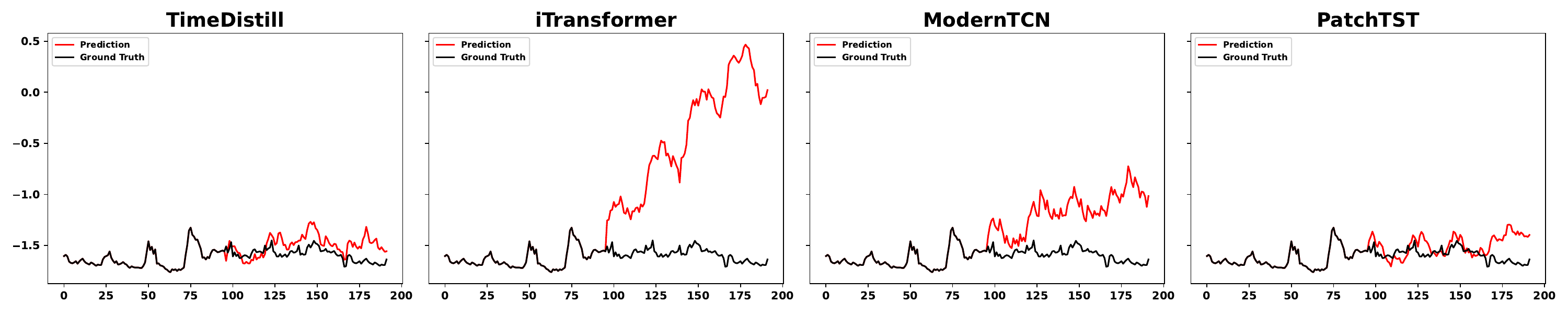}
    \caption{Prediction cases from ECL by different models under the input-720-predict-96 settings. \textcolor{black}
    {\textbf{Black}} lines are the ground truths and \textcolor{red}{\textbf{Red}} lines are the model predictions. Due to space constraints, we only retained the last 96 time steps of input for plotting.}
    \label{fig:ECL_case}
\end{figure*}

\begin{figure*}[h!]
    \centering
    \includegraphics[width=0.99\textwidth]{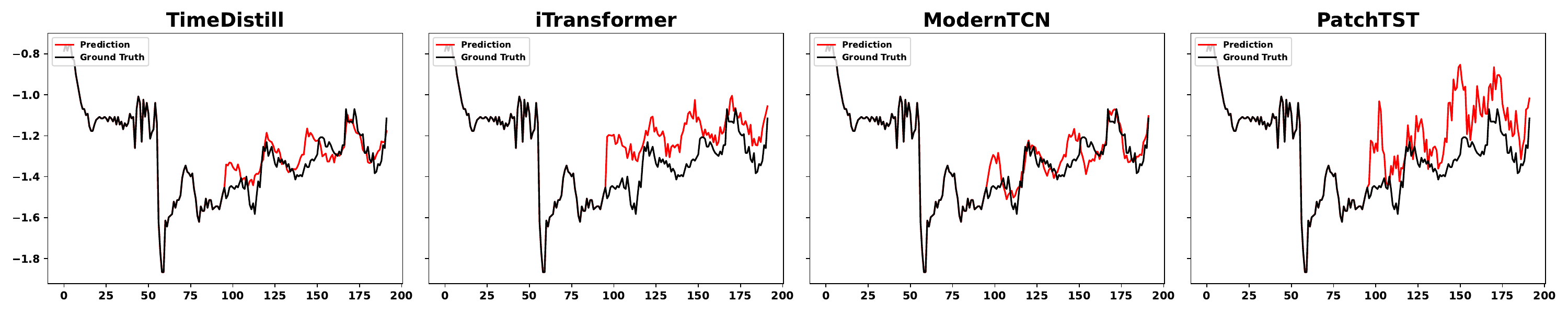}
    \caption{Prediction cases from ETTh1 by different models under the input-720-predict-96 settings.}
    \label{fig:ECL_case}
\end{figure*}

\begin{figure*}[h!]
    \centering
    \includegraphics[width=0.99\textwidth]{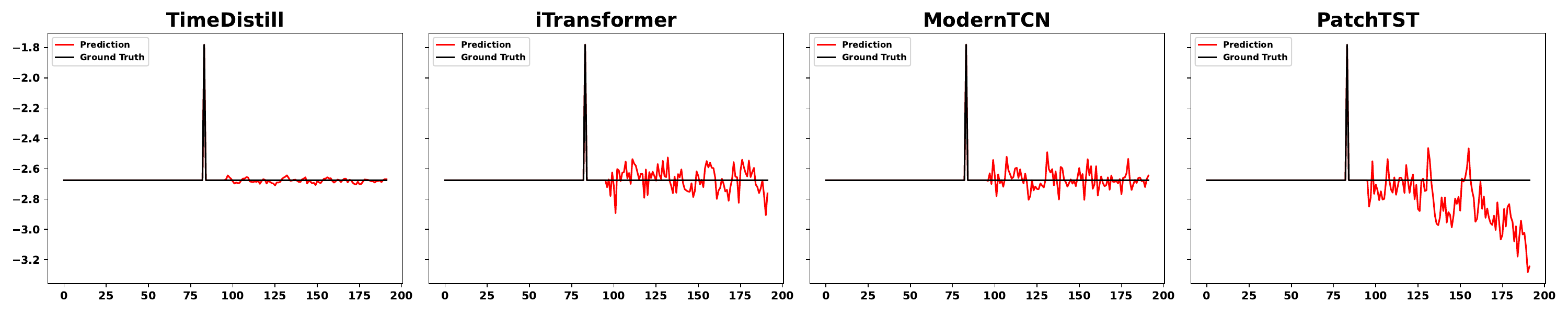}
    \caption{Prediction cases from ETTh2 by different models under the input-720-predict-96 settings.}
    \label{fig:ECL_case}
\end{figure*}

\begin{figure*}[h!]
    \centering
    \includegraphics[width=0.99\textwidth]{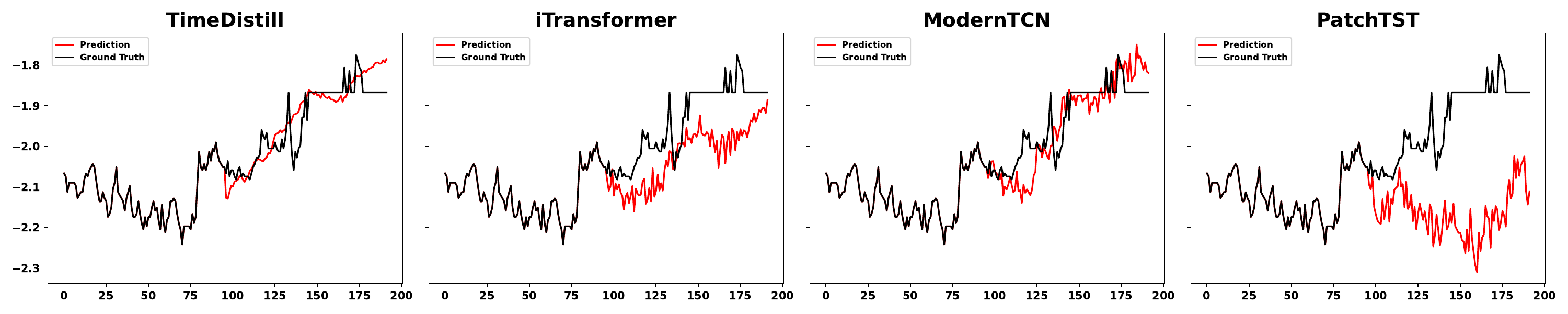}
    \caption{Prediction cases from ETTm1 by different models under the input-720-predict-96 settings.}
    \label{fig:ECL_case}
\end{figure*}

\begin{figure*}[h!]
    \centering
    \includegraphics[width=0.99\textwidth]{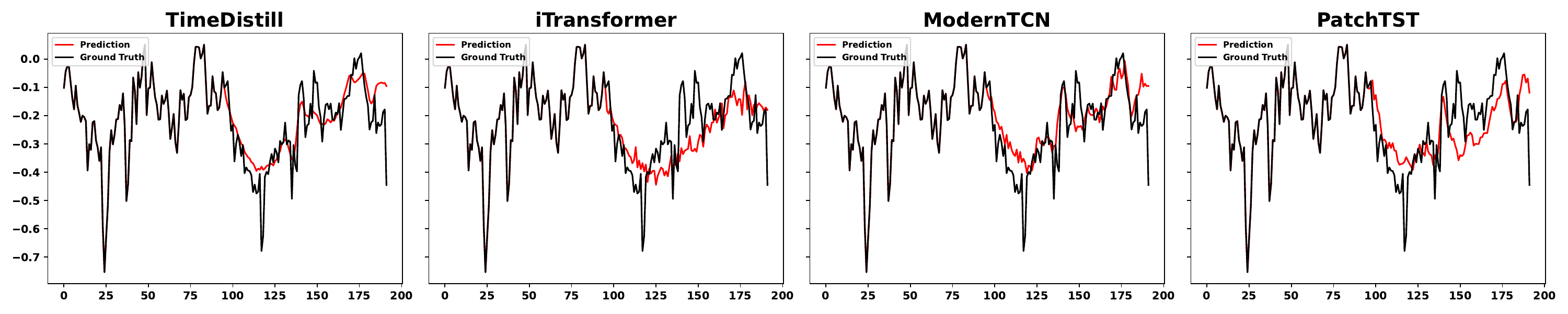}
    \caption{Prediction cases from ETTm2 by different models under the input-720-predict-96 settings.}
    \label{fig:ECL_case}
\end{figure*}

\begin{figure*}[h!]
    \centering
    \includegraphics[width=0.99\textwidth]{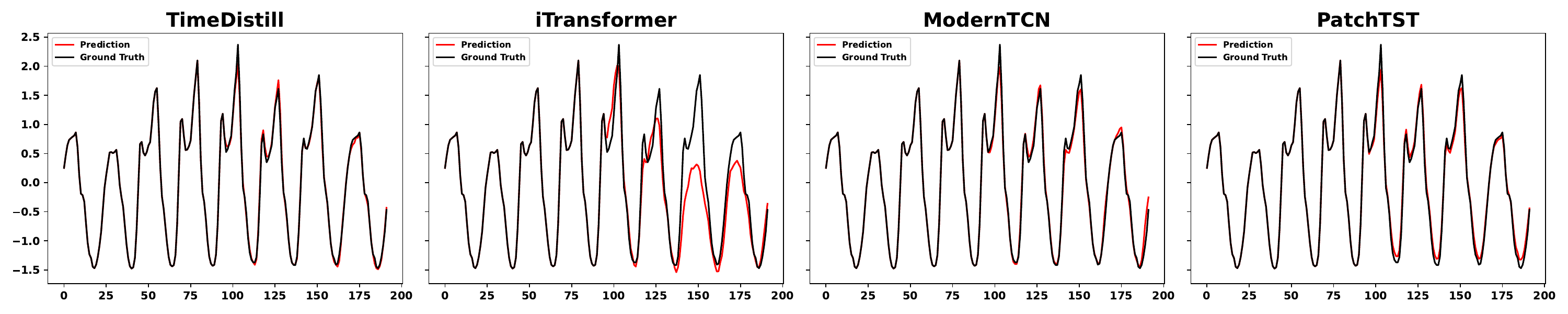}
    \caption{Prediction cases from Traffic by different models under the input-720-predict-96 settings.}
    \label{fig:ECL_case}
\end{figure*}

\begin{figure*}[h!]
    \centering
    \includegraphics[width=0.99\textwidth]{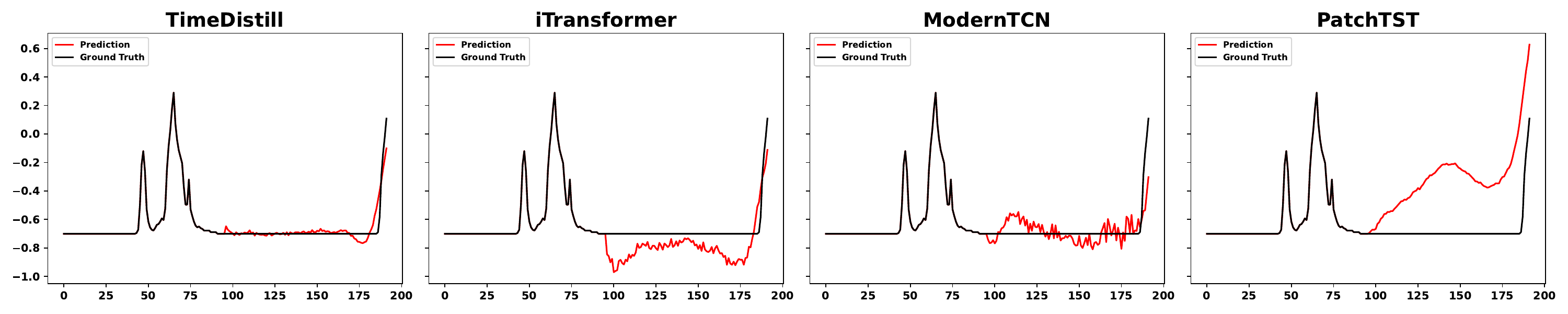}
    \caption{Prediction cases from Solar by different models under the input-720-predict-96 settings.}
    \label{fig:ECL_case}
\end{figure*}

\begin{figure*}[h!]
    \centering
    \includegraphics[width=0.99\textwidth]{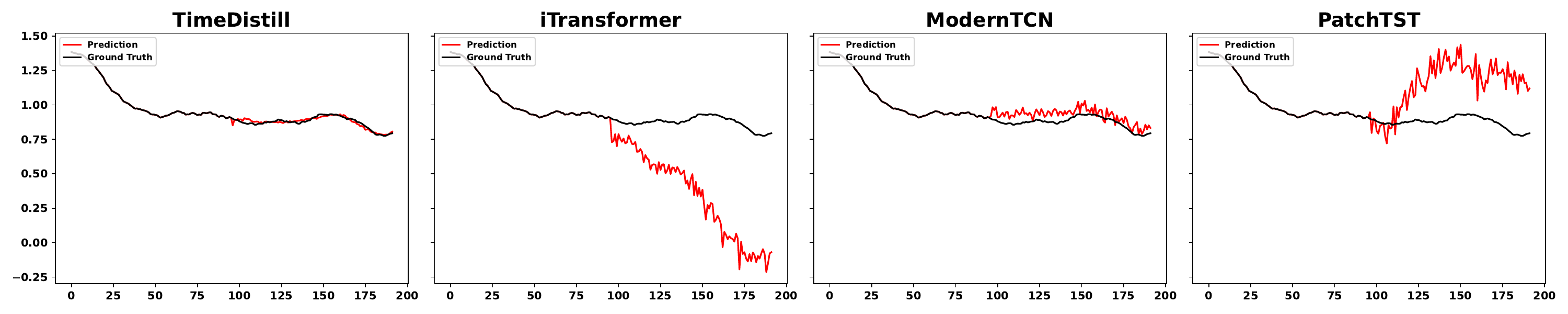}
    \caption{Prediction cases from Weather by different models under the input-720-predict-96 settings.}
    \label{fig:ECL_case}
\end{figure*}

\end{document}